\def\RevisionApplyChanges{1}
\expandafter\def\csname RevisionIncludeM24\endcsname{1}
\expandafter\def\csname RevisionIncludeM28\endcsname{1}

\documentclass[letterpaper]{article}
\usepackage[preprint]{aaai2027}
\usepackage[hyphens]{url}
\usepackage{graphicx}
\usepackage{natbib}
\usepackage{caption}
\usepackage{algorithm}
\usepackage{algorithmic}
\usepackage{booktabs}
\usepackage{amsmath,amssymb,amsfonts,amsthm}
\usepackage{xcolor}
\usepackage{eucal}
\usepackage{xspace}
\usepackage{tabularx}
\usepackage{subcaption}
\usepackage{todonotes}
\newtheorem{lemma}{Lemma}
\newtheorem{proposition}{Proposition}
\theoremstyle{definition}

\newtheorem{theorem}{Theorem}
\newtheorem{remark}{Remark}

\newtheorem{definition}{Definition}

\newenvironment{threeparttable}{}{}
\newenvironment{tablenotes}
  {\begin{list}{}{\setlength{\leftmargin}{1.2em}}}
  {\end{list}}

\newcommand{\tnote}[1]{\textsuperscript{#1}}
\providecommand{\multirow}[3]{#3}

\def\X{\mathcal{X}}

\def\R{\mathbb{R}}

\def\diag{\operatorname{diag}}
\DeclareMathOperator*{\argmax}{arg\,max}
\DeclareMathOperator*{\argmin}{arg\,min}

\providecommand{\mathscr}[1]{\mathcal{#1}}

\newcommand{\brian}[1]{}

\newcommand{\revdel}[1]{}

\newif\ifrevisionmarked
\ifdefined\RevisionMarked
  \revisionmarkedtrue
\else
  \revisionmarkedfalse
\fi
\newif\ifrevisionapplychanges
\ifdefined\RevisionApplyChanges
  \revisionapplychangestrue
\else
  \revisionapplychangesfalse
\fi

\definecolor{RevisionRemoveRed}{RGB}{185,28,28}
\definecolor{RevisionAddGreen}{RGB}{0,118,72}
\definecolor{RevisionMarkerBlue}{RGB}{29,78,216}

\ifrevisionmarked
  \renewcommand{\todo}[1]{%
    \par\smallskip
    \noindent\textit{TODO: #1}%
    \par\smallskip
  }
\fi

\newcommand{\IssueRef}[1]{%
  \ifrevisionmarked
    \textsuperscript{\textcolor{RevisionMarkerBlue}{\texttt{[#1]}}}%
  \fi
}

\newcommand{\RevisionRemove}[2]{%
  \ifrevisionmarked
    {\textcolor{RevisionRemoveRed}{#2}}\IssueRef{#1}%
  \else
    \ifrevisionapplychanges
    \else
      #2%
    \fi
  \fi
}

\newcommand{\RevisionAdd}[2]{%
  \ifrevisionmarked
    {\textcolor{RevisionAddGreen}{#2}}\IssueRef{#1}%
  \else
    \ifrevisionapplychanges
      #2%
    \fi
  \fi
}

\newcommand{\RevisionReplace}[3]{%
  \ifrevisionmarked
    {\textcolor{RevisionRemoveRed}{#2}}%
    {\textcolor{RevisionAddGreen}{#3}}\IssueRef{#1}%
  \else
    \ifrevisionapplychanges
      #3%
    \else
      #2%
    \fi
  \fi
}

\newcommand{\RevisionComment}[2]{%
  \ifrevisionmarked
    \par\smallskip
    \noindent{\footnotesize\color{RevisionMarkerBlue}%
    \textbf{Revision note \texttt{[#1]}.} #2}%
    \par\smallskip
  \fi
}

\newcommand{\RevisionTieredReplace}[4]{%
  \ifrevisionmarked
    {\textcolor{RevisionRemoveRed}{#2}}%
    {\textcolor{RevisionAddGreen}{#4}}\IssueRef{#1}%
  \else
    \ifrevisionapplychanges
      \ifcsname RevisionInclude#1\endcsname
        #4%
      \else
        #3%
      \fi
    \else
      #2%
    \fi
  \fi
}

\newcommand{\RevisionOptionalRemove}[2]{%
  \ifrevisionmarked
    {\textcolor{RevisionRemoveRed}{#2}}\IssueRef{#1}%
  \else
    \ifrevisionapplychanges
      \ifcsname RevisionInclude#1\endcsname
      \else
        #2%
      \fi
    \else
      #2%
    \fi
  \fi
}

\newcommand{\ManuscriptNLSVDRepresentativePoints}{1{,}000}
\newcommand{\ManuscriptNLSVDNumTrials}{100}
\newcommand{\ManuscriptPerturbationNLSVDPoints}{1{,}000}
\newcommand{\ManuscriptPerturbationEvalPointsMnistFashion}{500}
\newcommand{\ManuscriptPerturbationEvalPointsCifar}{100}
\newcommand{\ManuscriptBlackBoxMnistTimeSec}{2.61}

\newcommand{\ManuscriptBlackBoxFashionTimeSec}{2.35}

\newcommand{\ManuscriptBlackBoxCifarTimeSec}{0.36}

\newcommand{\ManuscriptCWBigQueryRatioMnist}{6.6}
\newcommand{\ManuscriptCWBigQueryRatioFashion}{6.6}
\newcommand{\ManuscriptCWBigQueryRatioCifar}{6.6}
\newcommand{\ManuscriptCWMainMaxIterations}{1{,}500}
\newcommand{\ManuscriptCWBigMaxIterations}{10{,}000}
\newcommand{\ManuscriptNLSVDNetBatchSize}{64}
\newcommand{\ManuscriptNLSVDNetEncodingDims}{1{,}000}
\newcommand{\ManuscriptNLSVDNetResLayers}{20}
\newcommand{\ManuscriptNLSVDNetResChannels}{10}
\newcommand{\ManuscriptNLSVDNetLearningRate}{10^{-3}}
\newcommand{\ManuscriptNLSVDNetRegularizationCutoff}{4.0}
\newcommand{\ManuscriptNLSVDNetLabelOnValue}{10.0}
\newcommand{\ManuscriptNLSVDNetLabelOffValue}{0.1}
\newcommand{\ManuscriptNLSVDNetMnistEpochs}{5}
\newcommand{\ManuscriptNLSVDNetFashionEpochs}{10}
\newcommand{\ManuscriptNLSVDNetMnistEncodingScale}{2.5}
\newcommand{\ManuscriptNLSVDNetFashionEncodingScale}{3.0}

\AtBeginDocument{\thispagestyle{plain}}

\title{A Nonlinear Singular Value Theory for Neural Networks}
\author{
Brian Brown\equalcontrib\corresponding\textsuperscript{\rm 1},
Mauricio Mu\~noz\textsuperscript{\rm 2},
Robert Bridges\equalcontrib\textsuperscript{\rm 2},\\
David Grimsman\textsuperscript{\rm 1},
Sean Warnick\textsuperscript{\rm 1}
}
\affiliations{
\textsuperscript{\rm 1}Department of Computer Science, Brigham Young University, Provo, USA\\
\textsuperscript{\rm 2}AI Sweden, Gothenburg, Sweden\\
bcb56@byu.edu
}

\begin{document}
\maketitle

\begin{abstract}
\RevisionReplace{B01}{%
Building on the abstract Nonlinear Singular Value Decomposition (NLSVD) theory of~\citet{brown2025nlsvd}, we prove that most modern neural architectures admit a nonlinear SVD representation in which they are left-invertible before a final linear layer, with no change in input-output behavior. Furthermore, the left-invertible nonlinear portion of the input-output behavior can be made to be \emph{norm preserving}, meaning that perturbations in the left-invertible ``embedding'' (the activations prior to the final linear layer in this representation) correspond proportionally to changes in the input, i.e., distance in feature space can be calibrated directly to distance in input space.
We provide a data-driven algorithm for estimating this representation from trained models and propose a model architecture that naturally facilitates the decomposition.
We also develop the theory necessary for future applications to areas such as model bias and invertibility.%
}{%
Recently \citet{brown2025nlsvd} established a singular value decomposition (SVD) for maps (especially nonlinear)  satisfying certain norm conditions.
We prove that most modern neural architectures admit this nonlinear SVD (NLSVD) representation---with no change in input--output behavior---and enumerate the classes covered.
In this factorization the network is a left-invertible nonlinear map followed by a final linear layer.
Moreover, the left-invertible factor is norm-preserving, so distances in the embedding (activations before the final linear layer) calibrate directly to distances in input space.
We introduce a flexible architecture that yields an explicit decomposition at training time, a data-driven algorithm for estimating the representation from trained models, and the mathematical foundations for nonlinear analogues of row and null spaces in neural networks.
Empirical case studies illustrate uses of the theory for latent-space pullback (visualization and data generation), bias detection, membership-inference robustness under training, and black-box membership inference.
Altogether, these foundations support new approaches to core problems in neural-network analysis.%
}
\RevisionComment{B01}{Bobby rewrite of the abstract for coauthor review. Prior M04 attack wording is already baked into the red baseline (``We also develop\ldots'').}
\RevisionComment{B01}{Revise abstract last to reflect exactly what is in the paper.}
\end{abstract}

\section{Introduction}
\RevisionReplace{B02}{%
Neural networks are nonlinear, but many methods for interpreting and steering them still search for linear directions in an internal representation. An arbitrary representation does not tell us which directions affect the output or how distances are distorted when translating between internal representations and inputs. We study factorizations of the form \(f=U\Sigma v\), where \(v\) is injective and norm-preserving, \(U\) sets the output directions, and \(\Sigma\) assigns their gains. We call this factorization a nonlinear singular value decomposition (NLSVD). The factorization recasts the nonlinearity of broad classes of neural networks as a norm-preserving representation followed by an explicit linear map, so the same classical geometry can support methods for analysis, interpretation, and control. We use that geometry across several applications without claiming specialized state-of-the-art results in each. Rather, we seek to demonstrate how the representation can be used to approach many problems.%
}{%
Neural networks are nonlinear, yet much of modern analysis still hunts for linear structure in a latent representation---directions and subspaces that support interpretation, control, robustness, and related tasks.
Here we prove that most common network classes admit a recently developed nonlinear singular value decomposition (NLSVD), yielding a factorization \(f=U\Sigma v\) in which \(v\) is injective and norm-preserving, \(U\) sets the output directions, and \(\Sigma\) assigns their gains, so classical geometric reasoning can systematically inspire new methods for analysis, interpretation, and control.
We exhibit several uses of the decomposition, not seeking state-of-the-art results but to demonstrate how the representation can approach many problems.%
}
\RevisionComment{B02}{Bobby rewrite of the introduction lead-in. Supersedes the M24 opening paragraph.}

We begin by recalling the Singular Value Decomposition (SVD) from linear algebra.
Any $m\times n$ matrix $A$ can be factored into $A = U\Sigma V^\top$ with unitaries $U$ $(m\times m)$, $V (n\times n)$, and a rectangular diagonal matrix $\Sigma \in \mathbb{R}^{m \times n}$ containing non-negative singular values $\sigma_i$ in descending order.
The SVD provides a quantification of $A$'s \emph{anisotropy}---its directionally--dependent scaling.
Geometrically, the SVD characterizes how a linear transformation maps the unit sphere in the domain to an ellipsoid in the range space.
The columns of $U$ define the principal axes of this ellipsoid, the singular values $\sigma_i$ determine the stretching factors along those axes, and the columns of $V$ represent the pre-images of these axes in the input space.
Under this framework, the familiar operator norm $\|A\|_2 = \sup_{x \neq 0} {\|Ax\|_2}/{\|x\|_2} = \sigma_1$ gives the largest axis of the ellipsoid---the maximum stretching $A$ can produce.
Furthermore, the SVD provides a complete characterization of the fundamental subspaces: the columns of $V$ corresponding to non-zero singular values that span the row space $\mathcal{R}(A^T) \triangleq \{A^\top y : y \in \mathbb{R}^m\}$,  while those corresponding to zero singular values span the null space $\mathcal{N}(A) \triangleq \{x \in \mathbb{R}^n : Ax = 0\}$.
This geometric intuition and row-/null-space decomposition give valuable insights for analysis, e.g., the construction of the Moore-Penrose pseudoinverse $A^\dagger = V \Sigma^\dagger U^\top$, allowing for a principled approximate inversion of the map even in over-/under-determined ($m\neq n$) or rank-deficient settings.

Recently, the Nonlinear SVD (NLSVD) of \citet{brown2025nlsvd} extended this result to a broad class of non-linear functions, namely those with finite 2-induced norm\RevisionRemove{B03}{ (Definition \ref{def:2-2norm})}.
\RevisionComment{B03}{Drop the forward pointer; the definition follows immediately.}
\begin{definition}[$2\to2$ norm or 2-induced norm]\label{def:2-2norm}
For $f:\mathbb{R}^n\to\mathbb{R}^m$, define \mbox{$\|f\|_{2\to2}\triangleq \sup_{x\neq0}\frac{\|f(x)\|_2}{\|x\|_2}$.}
\end{definition}
This is a generalization of the operator norm $\|A\|_2$ and provides the same geometric interpretation:
$f(x)$ lies in an ellipsoid with its major axis stretched by at most $\|f\|_{2\to2}.$
\begin{theorem}[Nonlinear SVD \cite{brown2025nlsvd}]\label{thm:nlsvd}
Let $f:\mathbb R^{n}\to\mathbb R^{m}$ satisfy $f(0)=0$, let $\ell\triangleq m+n$, and assume $\|f\|_{2\to2}<\infty$.
Then there exists an $m\times m$ unitary $U$, a rectangular diagonal matrix $\Sigma=[\diag(\sigma_1,\dots,\sigma_m)\mid \mathbf{0}_{m\times n}]\in\mathbb R^{m\times \ell}$ with $\sigma_i>0$, and an injective map $v:\mathbb R^n\to\mathbb R^{\ell}$ satisfying $\|v(x)\|_2=\|x\|_2$, such that $f=U\Sigma v$.
\end{theorem}
\RevisionRemove{B04}{Theorem~\ref{thm:nlsvd} appears in \cite{brown2025nlsvd}. }
We provide a simplified proof in \RevisionReplace{B05}{Appendix}{Supplement}~\ref{pf:thm-nlsvd} that illuminates its constructive nature, leading directly to our Algorithm~\ref{alg:nlsvd-construction}.
\RevisionComment{B05}{Proof lives in the separate supplement PDF; keep \texttt{\textbackslash ref\{pf:thm-nlsvd\}} (resolve via xr later). Algorithm is in the main paper.}
\RevisionReplace{B07}{See Figures}{See Supplementary Figures}~\ref{fig:conceptual_figure_combined}
\&~\ref{fig:conceptual_figure_lifting}.
\RevisionComment{B07}{Conceptual figures live in the supplement; prefix avoids main/supplement figure-number collision.}
\RevisionComment{B09}{Move conceptual figures from the supplement into the main intro if the seven-page budget permits.}
\begin{table}[t]
\centering
\caption{{NLSVD coverage for modern neural networks.}
After anchoring at a reference input, the following architectures have finite
\(2\)-induced gain and therefore admit an NLSVD (satisfy Theorem \ref{thm:nlsvd}) under the stated
restrictions.
Unbounded attention inputs, unbounded carried state, and arbitrarily long
iteration are not covered. 
The formal result appears as
Theorem~\ref{thm:architecture}; the proof is in the supplement.}
\label{tab:nlsvd-nn-coverage}
\footnotesize
\begin{tabularx}{\linewidth}{@{}>{\raggedright\arraybackslash}p{0.38\linewidth}>{\raggedright\arraybackslash}X@{}}
\toprule
\textbf{Architecture} & \textbf{Required Restriction} \\
\midrule
MLPs, CNNs, ResNets, and U-Nets
  & Finite composition on a compact input domain \\
RNNs and recurrent networks
  & Finite horizon and compact hidden-state domain \\
Fixed-length transformers
  & Compact token domain and fixed attention mask \\
Encoder--decoder transformers
  & Bounded source and target lengths \\
Stateful transformers
  & Finite-dimensional compact state and finite update horizon \\
Diffusion U-Net denoisers
  & Fixed time and conditioning, with deterministic evaluation \\
\bottomrule
\end{tabularx}
\end{table}

\RevisionComment{B10}{Modular coverage table extracted from the M26 boxed minipage (\texttt{20-architecture-brian.tex}); place early (first page) if space allows.}

\subsubsection{Contributions}
We prove applicability of the Nonlinear SVD (NLSVD) to neural networks, introduce this perspective to the machine-learning community, and demonstrate its usefulness for analyzing neural networks.
Specific contributions include:
\textbf{Neural-network coverage:} conditions establishing NLSVD representations for common feedforward, convolutional, recurrent, and transformer architectures (summarized in Table \ref{tab:nlsvd-nn-coverage}). \textbf{Nonlinear geometric tools:} generalized row and null directions, left inversion, and an NLSVD-structured neural-network architecture. \textbf{Black-box construction:} an algorithm for estimating an NLSVD from function evaluations, together with empirical measurements of its approximation quality. \textbf{Controlled empirical studies:} calibrated latent traversal and pullback, inference-time bias detection and membership inference, and training-time robustness to membership inference attacks experiments exhibit NLSVD-specific application and interpretation.
\RevisionComment{B06}{Accepted M24: long contribution bullet list removed; compact contributions paragraph inlined here (was \texttt{10-contributions-brian.tex}).}
Overall, we argue that, analogous to the original SVD's contributions to linear analysis, NLSVD can inspire  geometric intuition and algorithmic improvements in various machine learning tasks.

\RevisionRemove{B12}{%
\subsection{Known Limitations}
Theorem~\ref{thm:nlsvd} relaxes the properties of $V^\top$ from the linear SVD—which requires $V^\top$ to be \emph{bijective} and \emph{inner-product preserving}—to a lifting function $v(x)$ that is only required to be \emph{injective} and \emph{norm-preserving}.
Several other properties of the classical linear SVD are relinquished to accommodate the broader class of functions with a bounded 2-induced norm.
Notably, the largest singular value $\sigma_1$ in the proposed construction is not necessarily a tight upper bound on the $\|f\|_{2\to2}$ norm, as seen in the constraint in \eqref{eq:sigma_constraint}.
This contrasts with the classical SVD of a linear operator, where the largest singular value is exactly the operator's induced norm.
In particular, this bound may be loose by a multiplicative factor as large as $\sqrt{m}$, since $\|f(x)\|_2^2 = \sum f_i(x)^2 \leq m \cdot \max_i f_i(x)^2$.
This same \(\sqrt m\) scale appears in the coordinate construction.
Writing \(\alpha_i=\|f_i\|_{2\to2}\), for any \(\epsilon>0\) the simultaneous
choice \(\sigma_i=(\sqrt m+\epsilon)\alpha_i\) for nonzero \(\alpha_i\) (and
any positive \(\sigma_i\) when \(\alpha_i=0\)) gives
$\sum_i \frac{\alpha_i^2}{\sigma_i^2}
\le \frac{m}{(\sqrt m+\epsilon)^2}<1$,
so the construction makes coordinate gains tight only up to an arbitrarily
small excess over the worst-case \(\sqrt m\) factor.  This scale is unavoidable
for proportional coordinate bounds in the aligned worst case: if all coordinate
gains are attained together and \(\sigma_i=c\alpha_i\), then the constraint
\eqref{eq:sum_norms} requires \(m/c^2<1\), hence \(c>\sqrt m\).
Finally, minimizing the elements of $\Sigma$ requires non-convex optimization over the input space if $f$ is non-convex, which presents a challenge for practical implementations.
Because \citet{brown2025nlsvd} prove the NLSVD theorem constructively, Algorithm \ref{alg:nlsvd-construction} follows the same coordinate-gain recipe for black-box estimation.
The resulting construction yields limitations on $U$ and $v$, which we describe and exhibit via examples in
Sections~\ref{sec:construction} and Appendix~\ref{sec:construction-limitation}.%
}
\RevisionComment{B12}{Known Limitations moved from the intro to the end of Section~\ref{sec:construction}.}

\subsection{NLSVD Related Work}\label{sec:related-works}
\RevisionReplace{B13}{%
The acronym ``GSVD'' already refers to several linear matrix-factorization frameworks, most commonly simultaneous decompositions of matrix pairs sharing a column dimension \cite{vanLoan1976GeneralizingSVD,paigeSaunders1981TowardsGSVD}, along with variants based on modified inner products, restricted factorizations, and generalized matrix pencils \cite{deMoorZha1991TreeGeneralizationsOSVD,deMoorGolub1991generalizations, deMoorZha1991TreeGeneralizationsOSVD}. We therefore use nonlinear singular value decomposition (NLSVD) for the map-level factorization studied here: nonlinear maps $f:\mathbb R^n\to\mathbb R^m$ admitting factorizations of the form $f(x)=U\Sigma v(x)$ with $v$ a norm-preserving injective lift. Related directions in machine learning include invertible and flow-based architectures \cite{kingma2018glow,behrmann2018}, Koopman-inspired lifting methods, saliency and attribution techniques \cite{simonyan2013saliency,smilkov2017smoothgrad}, neural tangent kernel and linearization analyses \cite{jacot2018neuralTangentKernel,lee2019wideLinearModels}, and spectral/Lipschitz approaches to robustness and sensitivity \cite{miyato2018spectralNormGAN,virmaux2018lipschitzRegularity}. The NLSVD perspective is complementary to these methods. Rather than giving a local linearization, training-limit description, or scalar gain bound, it provides a global anisotropic factorization exposing coordinate gains, generalized row/null geometry, and constructive pseudo-inverse structure for trained finite networks. Expanded comparisons and discussion of adjacent literature are deferred to Appendix~\ref{sec:related-works-appendix}.%
}{%
``Generalized SVD'' already refers to several linear matrix-factorization frameworks, usually decompositions of matrix pairs sharing a column dimension \cite{vanLoan1976GeneralizingSVD,paigeSaunders1981TowardsGSVD}, along with variants based on modified inner products, restricted factorizations, and generalized matrix pencils \cite{deMoorZha1991TreeGeneralizationsOSVD,deMoorGolub1991generalizations, deMoorZha1991TreeGeneralizationsOSVD}. We use NLSVD for non-linear maps $f:\mathbb R^n\to\mathbb R^m$ admitting factorizations of the form $f(x)=U\Sigma v(x)$ with $v$ a norm-preserving injective lift.
Related areas include invertible and flow-based architectures \cite{kingma2018glow,behrmann2018}, Koopman-inspired lifting methods, saliency and attribution techniques \cite{simonyan2013saliency,smilkov2017smoothgrad}, neural tangent kernel and linearization analyses \cite{jacot2018neuralTangentKernel,lee2019wideLinearModels}, and spectral/Lipschitz approaches to robustness and sensitivity \cite{miyato2018spectralNormGAN,virmaux2018lipschitzRegularity}. The NLSVD perspective is complementary to these methods. Rather than giving a local linearization, training-limit description, or scalar gain bound, it provides a global anisotropic factorization exposing coordinate gains, generalized row/null geometry, and constructive pseudo-inverse structure for trained finite networks. Expanded comparisons and discussion of adjacent literature are deferred to Supplement~\ref{sec:related-works-appendix}.%
}
\RevisionComment{B13}{Bobby polish of the opening related-work paragraph; Appendix$\to$Supplement for the deferred discussion.}

\RevisionReplace{B14}{%
Representation learning seeks mappings from complex inputs into task-optimized feature spaces by training models to extract abstract semantic variables \cite{bengio2013}. NLSVD similarly isolates an explicit, task-aligned latent representation $v(x)$ before the final layer, but mathematically factorizes a trained model rather than learning new features.
Manifold learning seeks low-dimensional structure in observed data \cite{meilua2024}. NLSVD similarly embeds inputs isometrically into $\operatorname{Im}(v)$ and uses singular values to identify a functionally relevant subspace. However, NLSVD decomposes the function rather than modeling the data distribution.%
}{%
Representation learning seeks mappings from complex inputs into task-optimized feature spaces by training models to extract abstract semantic variables \cite{bengio2013}. NLSVD similarly isolates an explicit, task-aligned latent representation $v(x)$ before the final layer, but factorizes a trained model rather than learning new features.
Manifold learning seeks low-dimensional structure in observed data \cite{meilua2024}. NLSVD similarly embeds inputs isometrically into $\operatorname{Im}(v)$ and uses singular values to identify a functionally relevant subspace. However, NLSVD decomposes the function rather than modeling the data distribution.%
}
\RevisionComment{B14}{Drop ``mathematically'' before factorizes.}

Related work also studies concepts represented as linear directions or subspaces, their measurement with linear probes, and their manipulation using steering vectors
\ifrevisionmarked
\cite{mikolov2013,alain2017}
\textcolor{RevisionRemoveRed}{[\texttt{turner2024}]}
\textcolor{RevisionAddGreen}{\cite{turner2023}}\IssueRef{M29}
\else
\ifrevisionapplychanges
\cite{mikolov2013,alain2017,turner2023}
\else
\cite{mikolov2013,alain2017,turner2024}
\fi
\fi.
Park et al.\ formalize these as the subspace, measurement, and intervention views of the ``Linear Representation Hypothesis'' \cite{park2024}.
\RevisionReplace{B15}{%
NLSVD is complementary: its directions are determined by global input--output gain and provides a formal framework for geometric analysis of its induced latent representation, with input-calibrated distances and functionally defined row, null, and singular directions.%
}{%
NLSVD is complementary.
Its directions are determined by global input--output gain, and it provides a formal framework for geometric analysis of its induced latent representation, with input-calibrated distances and functionally defined row, null, and singular directions.%
}
\RevisionComment{B15}{Split the closing sentence and fix subject--verb agreement (``directions \ldots provide'' $\to$ ``it provides'').}

\section{Mathematical Treatment}\label{sec:math}

\RevisionReplace{B16}{%
The foundational NLSVD theorem and the construction of the $v$ function are due to ~\cite{brown2025nlsvd}; we take that decomposition as our starting point.
In this section, we explore implications for associated generalizations of the idea of row space and null space.
We posit that these generalizations unify several disparate areas of machine learning, including synthetic data generation, model explainability, and adversarial robustness.%
}{%
Building on the foundational NLSVD theorem of \citet{brown2025nlsvd}, we make two advances: we prove that most modern neural networks admit an NLSVD, and we develop then explore nonlinear row/null-space generalizations.%
}
\RevisionComment{B16}{Bobby rewrite of the math-section lead-in; architecture coverage moved above row/null material.}

\subsection{Do modern NNs admit an NLSVD?}
\label{sec:nlsvd-for-which-nn}
\RevisionComment{B17}{Architecture-coverage block at the front of the math section; Theorem~\ref{thm:architecture} restored in-main (proof in supplement).}
We establish that networks with finite Lipschitz constants, esp. architectures on compact domains, admit an NLSVD.

\subsubsection{Terminology} 
A \emph{fixed-length token domain} means a subset of \(\R^{n\times d}\) with a fixed number \(n\) of tokens, each represented in \(d\) dimensions, with \(n,d<\infty\).  
A \emph{fixed attention mask} means that the additive mask used inside attention is chosen in advance for that sequence length rather than changing with the input.
A \emph{compact query--memory product domain} means a product \(\X_q\times\X_m\), where decoder/query states range over a compact \(\X_q\) and encoder/memory states range over a compact \(\X_m\).  
A \emph{bounded-length} sequence model means that only finitely many source and target lengths are allowed, e.g.\ \(n_s\le N_s\) and \(n_t\le N_t\).
An \emph{anchor} is simply a chosen point \(x_\star\in\X\) used to recenter the map when \(f(x_\star)\ne0\).  A map is \(L\)-\emph{Lipschitz} on \(\X\) when \(\|f(x)-f(y)\|_2\le L\|x-y\|_2\) for all \(x,y\in\X\); here ``finite gain''
means the supremum ratio defining the \(2\)-induced norm below is finite.

\subsubsection{Main Theoretical Results}
\RevisionReplace{B19}{%
In brief: Theorem~\ref{thm:compact-bound} shows that Lipschitz maps, after anchoring, have finite \(2\)-induced gain (hence admit an NLSVD). Theorem~\ref{thm:architecture} shows that modern architectures satisfy those Lipschitz hypotheses under standard domain restrictions, so they inherit finite gain and an NLSVD.
For each modern architecture we must guarantee the hypotheses of
the NLSVD theorem, namely, finite \(2\)-induced gain.
\citet{kim2021selfattentionlipschitz}
show that standard dot-product self-attention is not globally Lipschitz on
unbounded domains, while \citet{castin2024smoothattention} give
quantitative Lipschitz bounds for fixed-length unmasked and masked self-attention
on compact fixed-length token domains.
For a map \(f:\R^d\to\R^m\) and an application domain \(\X\subset\R^d\), write
\mbox{$\|f\|_{2\to2;\X}\triangleq\sup_{x\in\X\setminus\{0\}}\frac{\|f(x)\|_2}{\|x\|_2}$}.
For an anchor \(x_\star\in\X\), define the recentered domain
\mbox{\(\X_\star\triangleq\X-x_\star\)} and the deviation map \mbox{$f_\star(h)\triangleq f(x_\star+h)-f(x_\star),\qquad h\in\X_\star$}.
\begin{theorem}[Anchored Lipschitz maps have finite induced gain]
\label{thm:compact-bound}
Let \(\X\subset\R^d\), let \(x_\star\in\X\), and let \(f:\X\to\R^m\) be
\(L\)-Lipschitz on \(\X\).
Then \(f_\star(0)=0\) and
$\|f_\star\|_{2\to2;\X_\star}\le L<\infty$.
\end{theorem}
Thus any NLSVD theorem requiring finite \(2\)-induced gain applies to the anchored
deviation map \(f_\star\), and the original output is recovered by
$f(x)=f(x_\star)+f_\star(x-x_\star).$
\noindent\textit{Proof.} See Supplement~\ref{app:modern-architecture-proofs}.
The supplement gives explicit notation and Lipschitz bounds for affine
maps, convolutions, fixed pooling/downsampling, coordinatewise activations,
stabilized normalizations, concatenations, residual sums, and finite
compositions. Using Theorem~\ref{thm:compact-bound}, we immediately see
that the NLSVD theorem applies to neural networks composed of these elements.
Furthermore, U-Net denoisers used in diffusion models \cite{ronneberger2015unet, ho2020ddpm} satisfy the theorem's hypotheses
when viewed as deterministic maps on compact pixel or latent domains at a fixed diffusion time and fixed conditioning
(i.e., abstracting away non-Lipschitz behavior introduced by random sampling).
For attention-based architectures, decoder-only transformers are explicit functions of a finite context of previous tokens (unlike recurrent nets, which compress the past into a fixed-dimensional hidden state). To obtain a fixed-dimensional map, we fix the maximum input and output sequence lengths and use padding or length stratification.
The same argument covers stateful transformers over a fixed finite horizon when the carried state is finite-dimensional and restricted to a compact domain, since the state is included in the input and output of each update map. However, it does not provide a uniform finite-gain bound for unbounded state or arbitrarily long iteration.
\begin{theorem}[Modern architectures give anchored finite-gain maps]
\label{thm:architecture}
Let \(\X_0\) be compact, let \(T<\infty\), and for \(t=1,\ldots,T\) let
\(H_t:\X_{t-1}\to\X_t\) be blocks with \(\X_t=H_t(\X_{t-1})\).
Define the realized map by \(f=H_T\circ\cdots\circ H_1:\X_0\to\X_T\).
Suppose each block \(H_t\) is one
of the affine, convolutional, pooling/downsampling, coordinatewise-activation,
stabilized-normalization, residual, concatenation, feedforward, self-attention,
masked-attention, or encoder--decoder cross-attention blocks covered by
Lemma~\ref{lem:lipschitz-calculus}, the fixed-length attention
bound of \citet{castin2024smoothattention}, and Lemma~\ref{lem:cross-attention-lipschitz}.
For fixed finite token dimensions and fixed attention masks, let \(L_t<\infty\) be a Lipschitz constant for \(H_t\) on \(\X_{t-1}\).
Then \(\mathrm{Lip}_{\X_0}(f)\le \prod_{t=1}^T L_t < \infty\).
Consequently, for any fixed-length token domain anchor \(x_\star\in\X_0\), the recentered deviation map
\(f_\star(h)=f(x_\star+h)-f(x_\star)\) has finite induced gain by
Theorem~\ref{thm:compact-bound}. If only finitely many source/target sequence
lengths are allowed, the family constant is
\(\max_{(n_s,n_t)}\prod_t L_{t,n_s,n_t}<\infty\).
\end{theorem}
\noindent\textit{Proof.} See Supplement~\ref{app:modern-architecture-proofs}.
Table~\ref{tab:nlsvd-nn-coverage} summarizes the covered architectures and required restrictions.%
}{%
For each modern architecture we must guarantee the hypotheses of
the NLSVD theorem, namely, finite \(2\)-induced gain.
In brief: Theorem~\ref{thm:compact-bound} shows that Lipschitz maps, after anchoring, have finite \(2\)-induced gain, and Theorem~\ref{thm:architecture} shows that modern architectures satisfy those Lipschitz hypotheses under standard domain restrictions.
\citet{kim2021selfattentionlipschitz}
show that standard dot-product self-attention is not globally Lipschitz on
unbounded domains, while \citet{castin2024smoothattention} give
Lipschitz bounds for fixed-length unmasked and masked self-attention
on compact fixed-length token domains.
For a map \(f:\R^d\to\R^m\) and an application domain \(\X\subset\R^d\), write
\mbox{$\|f\|_{2\to2;\X}\triangleq\sup_{x\in\X\setminus\{0\}}\frac{\|f(x)\|_2}{\|x\|_2}$}.
For an anchor \(x_\star\in\X\), define the recentered domain
\mbox{\(\X_\star\triangleq\X-x_\star\)} and the deviation map \mbox{$f_\star(h)\triangleq f(x_\star+h)-f(x_\star),\qquad h\in\X_\star$}.
\begin{theorem}[Anchored Lipschitz maps have finite induced gain]
\label{thm:compact-bound}
Let \(\X\subset\R^d\), let \(x_\star\in\X\), and let \(f:\X\to\R^m\) be
\(L\)-Lipschitz on \(\X\).
Then \(f_\star(0)=0\) and
$\|f_\star\|_{2\to2;\X_\star}\le L<\infty$.
\end{theorem}
Thus any NLSVD theorem requiring finite \(2\)-induced gain applies to the anchored
deviation map \(f_\star\), and the original output is recovered by
$f(x)=f(x_\star)+f_\star(x-x_\star).$
Supplement~\ref{app:modern-architecture-proofs} gives the proof along with explicit notation and Lipschitz bounds for affine
maps, convolutions, fixed pooling/downsampling, coordinatewise activations,
stabilized normalizations, concatenations, residual sums, and finite
compositions.
Using Theorem~\ref{thm:compact-bound}, we immediately see
that the NLSVD theorem applies to neural networks composed of these elements.
Furthermore, U-Net denoisers used in diffusion models \cite{ronneberger2015unet, ho2020ddpm} satisfy the theorem's hypotheses
when viewed as deterministic maps on compact pixel or latent domains at a fixed diffusion time and fixed conditioning
(i.e., abstracting away non-Lipschitz behavior introduced by random sampling).
For attention-based architectures, decoder-only transformers are explicit functions of a finite context of previous tokens (unlike recurrent nets, which compress the past into a fixed-dimensional hidden state). To obtain a fixed-dimensional map, we fix the maximum input and output sequence lengths and use padding or length stratification.
The same argument covers stateful transformers over a fixed finite horizon when the carried state is finite-dimensional and restricted to a compact domain, since the state is included in the input and output of each update map. However, it does not provide a uniform finite-gain bound for unbounded state or arbitrarily long iteration.
\begin{theorem}[Modern NNs give anchored finite-gain maps]
\label{thm:architecture}
Let \(\X_0\) be compact, let \(T<\infty\), and for \(t=1,\ldots,T\) let
\(H_t:\X_{t-1}\to\X_t\) be blocks with \(\X_t=H_t(\X_{t-1})\).
Define the realized map by \(f=H_T\circ\cdots\circ H_1:\X_0\to\X_T\).
Suppose each block \(H_t\) is one
of the affine, convolutional, pooling/downsampling, coordinatewise-activation,
stabilized-normalization, residual, concatenation, feedforward, self-attention,
masked-attention, or encoder--decoder cross-attention blocks covered by
Lemma~\ref{lem:lipschitz-calculus}, the fixed-length attention
bound of \citet{castin2024smoothattention}, and Lemma~\ref{lem:cross-attention-lipschitz}.
For fixed finite token dimensions and fixed attention masks, let \(L_t<\infty\) be a Lipschitz constant for \(H_t\) on \(\X_{t-1}\).
Then \(\mathrm{Lip}_{\X_0}(f)\le \prod_{t=1}^T L_t < \infty\).
Consequently, for any fixed-length token domain anchor \(x_\star\in\X_0\), the recentered deviation map
\(f_\star(h)=f(x_\star+h)-f(x_\star)\) has finite induced gain by
Theorem~\ref{thm:compact-bound}. If only finitely many source/target sequence
lengths are allowed, the family constant is
\(\max_{(n_s,n_t)}\prod_t L_{t,n_s,n_t}<\infty\).
\end{theorem}
Supplement~\ref{app:modern-architecture-proofs} has the proof and details.
Table~\ref{tab:nlsvd-nn-coverage} summarizes the covered architectures and required restrictions.%
}
\RevisionComment{B19}{Bobby polish of Main Theoretical Results: tighter roadmap placement/wording, shorter proof pointers, theorem title ``Modern NNs\ldots''.}
\RevisionComment{B11}{Theorem~\ref{thm:architecture} kept in the main paper (statement here; proof in the supplement).}

\subsection{Nonlinear Row and Null Spaces}
The classical geometric intuition in nonlinear settings is that the row and null spaces ``bend'' into manifolds due to the map's nonlinearity.
We work instead with the domain of $\Sigma$, i.e., the lifted representation space $\R^{\ell} = \R^{m+n}$.
The norm-preserving lift $v(x)$ induces natural projections onto
$\mathcal{R}(\Sigma^\top) = \{\Sigma^\top y: y \in \R^m \}$ and $\mathcal{N}(\Sigma) = \{z\in \R^{\ell} : \Sigma z = 0\}$.
These projections define generalized notions of row/null structure by \emph{invariance}:
holding the row projection fixed yields score-equivalence classes (i.e., $U\Sigma z = U\Sigma (z+z_\mathcal{N})$ for all $z_\mathcal{N}\in \mathcal{N}(U\Sigma)= \mathcal{N}(\Sigma)$).
Similarly, holding the null projection fixed and moving in the row space yields row-like traversals (i.e., moves in the direction of most sensitivity).
\RevisionComment{B20}{Proofread row/null subsection: title casing, quotes, and comma splice before ``yields''.}

\begin{definition}[Row, null, and feature projections in the lifted space]
\label{def:zr_zn_of_x}
Assume $f=U\Sigma v:\mathbb{R}^n\to\mathbb{R}^m$ as in Theorem~\ref{thm:nlsvd}.
Let $P_{\mathrm{row}}$ and $P_{\mathrm{null}}$ denote the orthogonal projections in $\mathbb{R}^{\ell} = \R^{m+n}$ (range space of $v$)
onto the row space
$\mathcal{R}(\Sigma^\top)$ 
and null space $\mathcal{N}(\Sigma)$, respectively. (Equivalently, one may take
$P_{\mathrm{row}}=\Sigma^\dagger \Sigma$ and $P_{\mathrm{null}}=I_\ell-\Sigma^\dagger \Sigma$.)

Define the \emph{row} and \emph{null} components of the lifted input by
$z_r(x)\;\triangleq\;P_{\mathrm{row}}\,v(x),$ and 
$z_{\mathrm{n}}(x)\;\triangleq\;P_{\mathrm{null}}\,v(x).$
For each $i\in\{1,\dots,\ell\}$, define the \emph{$i$th feature projection} $P_i:\mathbb{R}^{\ell}\to\mathbb{R}^{\ell}$ by
$P_i z \triangleq (e_i^\top z)\,e_i$, where $\{e_i\}$ is the standard basis of $\mathbb{R}^{\ell}$.
\end{definition}

The intuition is that continuous movement in the input space corresponds to (not necessarily continuous) movement through the image of $v$.
As $v$ varies with $x$, so too does the projection onto $\mathcal{R}(\Sigma^\top)$ and $\mathcal{N}(\Sigma)$.
Perturbed inputs whose projections are \emph{invariant}---that project to the same location on either $\mathcal{R}(\Sigma^\top)$ or $\mathcal{N}(\Sigma)$ as the original input---are of particular interest because they contain the \emph{same amount} of the \emph{same feature} that is either mapped to zero (when projecting to the same lifted null location in $\mathcal{N}(\Sigma)$), or mapped to a particular output direction (when projecting to the same lifted row location in $\mathcal{R}(\Sigma^\top)$).
For a unitary $V^\top$, inner-product preservation ensures a special structure:
each projection level set is an affine subspace (a subspace at zero).
Furthermore, linearity (superposition) ensures that the same perturbation applied to different inputs has the same resulting effect on the projection onto $\mathcal{R}(\Sigma^\top)$ and $\mathcal{N}(\Sigma)$.
Nonlinear $f$ yields a $v$ that may rotate the same perturbation differently depending on the original input, resulting in different projections.
This motivates the generalized notions of null \emph{sets}, defined element-wise for specific inputs.

\begin{definition}[Generalized null set (row-invariance / score-equivalence)]
\label{def:generalized_null_set_x}
Under the assumptions of Definition~\ref{def:zr_zn_of_x}, define the \emph{generalized null set} at $x$ by
\begin{align*}
\mathcal{X}_{\mathrm{null}}(x)
\;&\triangleq\;
\{\,x'\in\mathbb{R}^n:\ z_r(x')=z_r(x)\,\} \\
&=
\{\,x'\in\mathbb{R}^n:\ P_{\mathrm{row}}v(x')=P_{\mathrm{row}}v(x)\,\}.
\end{align*}
Equivalently, since \(U\) is invertible, \(\Sigma\) has positive diagonal
entries on \(\mathcal{R}(\Sigma^\top)\), and \(\Sigma P_{\mathrm{null}}=0\),
this is the score-equivalence set: $\mathcal{X}_{\mathrm{null}}(x)=\{\,x'\in\mathbb{R}^n:\ f(x')=f(x)\,\}$.
\end{definition}

\begin{definition}[Generalized row set (null-invariance)]
\label{def:generalized_row_set_x}
Under the assumptions of Definition~\ref{def:zr_zn_of_x}, define the \emph{generalized row set} at $x$ by
\begin{align*}
\mathcal{X}_{\mathrm{row}}(x)
\;&\triangleq\;
\{\,x'\in\mathbb{R}^n:\ z_{\mathrm{n}}(x')=z_{\mathrm{n}}(x)\,\} \\
&=
\{\,x'\in\mathbb{R}^n:\ P_{\mathrm{null}}v(x')=P_{\mathrm{null}}v(x)\,\}.
\end{align*}
\end{definition}


\begin{remark}[Oblique motion relative to generalized row/null sets]
\label{rem:oblique_motion}
By Definitions~\ref{def:zr_zn_of_x}--\ref{def:generalized_row_set_x}, $\mathcal X_{\mathrm{null}}(x)$ and $\mathcal X_{\mathrm{row}}(x)$ are level sets of the orthogonal projections $P_{\mathrm{row}}v(\cdot)$ and $P_{\mathrm{null}}v(\cdot)$ in the lifted space $\mathbb R^{\ell}$.  However, the input trajectory $x\mapsto v(x)$ typically varies \emph{both} components $(P_{\mathrm{row}}v(x),\,P_{\mathrm{null}}v(x))$ at once, so small input perturbations move transversely across both foliations.  This mismatch is the geometric reason that ``linear motion'' in the lifted coordinates need not correspond to 1D motion in input space.  It motivates defining row-directed traversals as constrained/regularized pullbacks that explicitly control drift in the null component (Figure~\ref{fig:conceptual_figure_lifting}).
\end{remark}


\subsubsection{One-dimensional Row Traversals}
\label{sec:one_geometry_three_uses}

The \emph{bijective} property of $V^\top$ in linear SVD renders row/null space traversals of $\Sigma$ trivial.
NLSVD's $v$ is \emph{injective}, and thus a linear interpolation in the pre-image of $\Sigma$ may not lie in the image of $v$.
This motivates the following notion of an approximate inverse that involves a projection onto the \emph{image} of $v$ before taking $v$'s well-defined left inverse, per the following two definitions.

\begin{definition}[Row-space projection onto the lifted image]
\label{def:row_projection_to_image}

Let $\mathcal{M}\triangleq \mathrm{Im}(v)\subset\R^{\ell}$ denote the lifted image, and let
$P_{\mathrm{row}}:\R^{\ell}\to \mathcal{R}(\Sigma^\top)$ be the orthogonal projection onto the row space of $\Sigma$
(e.g., $P_{\mathrm{row}}=\Sigma^\dagger\Sigma$).

For $z\in\R^{\ell}$, define the (possibly set-valued) \emph{row-metric projection} onto $\mathcal{M}$ as a map into the lifted image:
\begin{align*}
P_{\mathcal{M},\mathrm{row}}(z)
&\triangleq
\left\{\,v(x^\star)\in\mathcal{M}:\right.\\
&\qquad\left.
x^\star\in\argmin_{x\in\R^{n}}
\bigl\|P_{\mathrm{row}}v(x)-P_{\mathrm{row}}z\bigr\|_2\,\right\}.
\end{align*}
whenever the minimum is attained. Thus the optimization is over input points
$x\in\R^n$, but the returned value is the corresponding lifted point(s)
$v(x^\star)\in\mathcal{M}\subset\R^{\ell}$.
\end{definition}
\begin{definition}[Projection onto the lifted image]
\label{def:metric_projection_to_image}
Let $\mathcal{M}\triangleq \mathrm{Im}(v)\subset\mathbb{R}^{\ell}$ denote the lifted image.
For $z\in\mathbb{R}^{\ell}$, define the (possibly set-valued) \emph{metric projection} onto $\mathcal{M}$ by
\mbox{$P_{\mathcal{M}}(z)\;\triangleq\;\{v(x^\star):x^\star\in\argmin_{x\in\mathbb{R}^n}\|v(x)-z\|_2\}$}
whenever the minimum is attained.
\end{definition}
\begin{definition}[Projected pseudo-inverse of the lift]
\label{def:projected_pseudoinverse_v}
Let $\mathcal{M}=\mathrm{Im}(v)$ and let
$\tilde P_{\mathcal{M}}:\R^{\ell}\to\mathcal{M}$ be either a
single-valued choice from Definition~\ref{def:metric_projection_to_image} or a
learned approximation whose image lies in $\mathcal{M}$.  Assume a left inverse
$v^{-L}:\mathcal{M}\to\mathbb{R}^n$ is available on $\mathcal{M}$.
Define the \emph{projected pseudo-inverse} of $v$ as the map
$v^\dagger:\R^{\ell}\to\mathbb{R}^n$ given by $v^\dagger(z)\;\triangleq\; v^{-L}\!\bigl(\tilde P_{\mathcal{M}}(z)\bigr),$
so the projected lifted point is always mapped back into the original input
space $\mathbb{R}^n$.
\end{definition}
\RevisionComment{B22}{Fix metric projection to return lifted points in $\mathcal{M}$ (was typed as returning $x\in\mathbb{R}^n$); rename $\tilde P_{\mathcal{M},\mathrm{row}}$ to $\tilde P_{\mathcal{M}}$.}
Given that $v^{\dagger}(z)$ is, in general, non-convex in $z$, we propose two methods of approximation.

\subsubsection{Learned Decomposition Approach}
For $g:\mathbb{R}^n\to\R^{\ell}$, the intended inverse (Definition~\ref{def:projected_pseudoinverse_v}) can be approximated by training a parametric left inverse $g^{-L}_\theta$ 
to minimize reconstruction error 
on the range $g(\mathbb{R}^n)$, e.g.,
$\min_{\theta}\ \mathbb{E}_{x\sim \mathcal{D}}\bigl[\|g^{-L}_\theta(g(x)) - x\|_2^2\bigr],$
so that $g^{-L}_\theta$ implicitly learns both a projection onto $\mathrm{Im}(g)$ and an approximate inverse on that manifold.

\subsubsection{Approximate Pseudo-Inverse by Construction}
In any case where $f$ can be evaluated directly, a left inverse on the image $\mathcal{M}\triangleq\mathrm{Im}(v)$ is given by the
following ``kernel recovery'' rule.  For any $z\in\mathbb{R}^{\ell}$, write $z = [z_{\mathrm{m}}, z_{\mathrm{n}}]^\top \in \R^{m\times n}=\R^{\ell}$
and define the \emph{naive extended decoder}
$\ \hat{v}^{-L}:\mathcal{D}\to\mathbb{R}^n$ on the domain $\mathcal{D}\triangleq
\{z\in\mathbb{R}^{\ell}:z_{\mathrm{n}}\neq 0\}$
by
$\hat{v}^{-L}(z)\;\triangleq\;({\|z\|_2}/{\|z_{\mathrm{n}}\|_2})\,z_{\mathrm{n}}.$
On $\mathcal{M}\cap\mathcal{D}$ this agrees with the canonical left inverse
$v^{-L}$ from the construction, i.e.\ $\hat v^{-L}(v(x))=x$ for \(x\ne0\);
the origin can be assigned separately: \(\hat v^{-L}(0)\triangleq 0\).
\RevisionTieredReplace{M25}{\todo{every neural network has an NLSVD representation in which its nonlinearity is left invertible.}}{}{%
\begin{remark}
A factorization \(f=U\Sigma v\) need not make \(f\) invertible, because the linear factor \(U\Sigma\) may discard directions; all nonlinearity is contained in the injective map \(v\), which is invertible on its image. Thus every network covered by Theorem~\ref{thm:architecture} admits a representation in which all the nonlinearity is left-invertible.
\end{remark}%
}
\begin{remark}[Why this decoder can misbehave]
The rule that
$\hat{v}^{-L}(z)\;\triangleq\; ({\|z\|_2}/{\|z_{\mathrm{n}}\|_2})\,z_{\mathrm{n}}$
fixes the \emph{direction} of $\hat x$ by $z_{\mathrm{n}}$ and forces $\|\hat x\|_2=\|z\|_2$.
It depends on $z_{\mathrm{m}}$ only through the scalar norm $\|z\|_2$.  Consequently, purely ``row'' changes in $z$
(i.e., changing $z_{\mathrm{m}}$ while $z_{\mathrm{n}}$ remains fixed)
can rescale $\hat x = \hat{v}^{-L}(z)$ even though they do not affect $\Sigma z$,
while purely ``null'' changes (changing $z_{\mathrm{n}}$ with $z_{\mathrm{m}}$ fixed) can alter $\hat x$ arbitrarily even though
they are annihilated by $\Sigma$.
The extension is ill-conditioned near $\{z:\|z_{\mathrm{n}}\|_2=0\}$, where the scaling
factor blows up; thus, off-manifold decoding is in general not related to the metric projection pullback
(Definition~\ref{def:metric_projection_to_image}).
\end{remark}

\begin{remark}[ML tasks as row and null set operations]
In the NLSVD coordinates $f(x)=U\Sigma v(x)$, several downstream tasks reduce to choosing a linear traversal in the $\Sigma$-pre-image space and pulling it back through $v$ to input space via a (projected) left inverse.
For example, output-preserving synthetic generation at \(x\) can be understood as
traversing \(v^{\dagger}(z_r(x)+\mathcal{N}(\Sigma))\) to vary the input
while preserving outputs as much as the geometry of \(v\) permits (Figure~\ref{fig:svdnet_visual_examples}, left panel).
Similarly, both model explainability and adversarial robustness can be understood as \emph{row-space} phenomena: explainability seeks \emph{structured} traversals aligned with $\mathcal{R}(\Sigma^\top)$ that induce interpretable score changes (Figure~\ref{fig:svdnet_visual_examples}, right panel), while model robustness can be understood as \emph{worst-case} (often constrained) row-space traversals that maximize score change under a norm or feature-invariance constraint.
\end{remark}
\RevisionComment{B21}{Proofread end of Section~2: subsubsection titles, remark formatting, and M25 remark grammar.}

\RevisionRemove{B29}{%
\begin{remark}
Comparing to the original SVD, NLSVD has known limitations.
Theorem~\ref{thm:nlsvd} relaxes the properties of $V^\top$ from the linear SVD---which requires $V^\top$ to be \emph{bijective} and \emph{inner-product preserving}---to a lifting function $v(x)$ that is only required to be \emph{injective} and \emph{norm-preserving}.
Several other properties of the classical linear SVD are relinquished to accommodate the broader class of functions with a bounded 2-induced norm.
Notably, the largest singular value $\sigma_1$ in the proposed construction is not necessarily a tight upper bound on the $\|f\|_{2\to2}$ norm, as seen in the constraint in \eqref{eq:sigma_constraint}.
This contrasts with the classical SVD of a linear operator, where the largest singular value is exactly the operator's induced norm.
In particular, this bound may be loose by a multiplicative factor as large as $\sqrt{m}$, since $\|f(x)\|_2^2 = \sum f_i(x)^2 \leq m \cdot \max_i f_i(x)^2$.
This same \(\sqrt m\) scale appears in the coordinate construction.
Writing \(\alpha_i=\|f_i\|_{2\to2}\), for any \(\epsilon>0\) the simultaneous
choice \(\sigma_i=(\sqrt m+\epsilon)\alpha_i\) for nonzero \(\alpha_i\) (and
any positive \(\sigma_i\) when \(\alpha_i=0\)) gives
$\sum_i {\alpha_i^2}/{\sigma_i^2}
\le {m}/{(\sqrt m+\epsilon)^2}<1$,
so the construction makes coordinate gains tight only up to an arbitrarily
small excess over the worst-case \(\sqrt m\) factor.  This scale is unavoidable
for proportional coordinate bounds in the aligned worst case: if all coordinate
gains are attained together and \(\sigma_i=c\alpha_i\), then the constraint
\eqref{eq:sum_norms} requires \(m/c^2<1\), hence \(c>\sqrt m\).
\end{remark}%
}
\RevisionComment{B29}{Known-limitations remark moved to Supplement after the NLSVD proof; Section~4 keeps a one-line pointer.}

\section{NLSVDNet: Training Explicit Components}\label{sec:svdnet}
\RevisionReplace{B23}{%
Here we present a method for training a neural network $f:\R^m \to \R^n$
with an explicit NLSVD-style factorization.
The full neural network is then expressed as $f(x) = K g(x)$, where $K$ is the final linear layer, and $g$ contains all previous layers, including all nonlinearities.
We will force $\|g(x)\| = \|x\|$ by simply defining the desired network layers $g_0(x)$, and setting $g(x) = \|x\| g_0(x)/\|g_0(x)\|$ for each $x$, if \(g_0(x)\neq 0\), and \(g(x)=0\) otherwise.
Next, we simultaneously train a second neural network $g^{-L}: \mathbb{R}^{\ell} \to \mathbb{R}^m$ to serve as the left inverse of $g$.
We augment the primary loss function with a term that enforces the left invertibility of $g$, such as $\|g^{-L}(g(x))-x\|_2^2$, allowing $g^{-L}$, $g$, and $K$ to be optimized concurrently.
Finally, we compute the SVD of the linear layer $K = U \Sigma V^\top$.
If \(g\) is norm-preserving and injective, then \(f(x)=U\Sigma v(x)\),
where \(v(x)\triangleq V^\top g(x)\), satisfies the NLSVD lift constraints.
The reconstruction loss encourages, rather than proves, this exact condition.
For finite-gain maps satisfying the NLSVD hypotheses after anchoring when needed,
Theorem \ref{thm:nlsvd} ensures that enforcing such a decomposition exists, in theory,
without loss of generality so long as $\ell = m+n$.%
}{%
Here we present a method for training a neural network $f:\R^m \to \R^n$
with an explicit NLSVD-style factorization.
Write $f(x)=Kg(x)$, where $K$ is the final linear layer and $g$ contains all previous layers (including all nonlinearities).
We enforce $\|g(x)\|_2=\|x\|_2$ by defining the desired network layers $g_0(x)$ and setting
$g(x)=\tfrac{\|x\|_2\,g_0(x)}{\|g_0(x)\|_2}$ when \(g_0(x)\neq 0\), and $g(x)=0$ otherwise.
Next, we simultaneously train a second network $g^{-L}:\mathbb{R}^{\ell}\to\mathbb{R}^m$ to serve as a left inverse of $g$.
We augment the primary loss with a reconstruction term such as $\|g^{-L}(g(x))-x\|_2^2$, so that $g^{-L}$, $g$, and $K$ can be optimized concurrently.
Finally, we compute the SVD of the linear layer, $K=U\Sigma V^\top$.
If \(g\) is norm-preserving and injective, then \(f(x)=U\Sigma v(x)\) with
\(v(x)\triangleq V^\top g(x)\) satisfies the NLSVD lift constraints.
The reconstruction loss encourages, rather than proves, this exact condition.
For finite-gain maps satisfying the NLSVD hypotheses (after anchoring when needed),
Theorem~\ref{thm:nlsvd} guarantees that such a decomposition exists whenever $\ell=m+n$.%
}
\RevisionComment{B23}{Proofread NLSVDNet lead-in: restore title NLSVD, norms with subscript 2, tfrac, Theorem-ref spacing, and grammar on the existence claim.}

\subsection{NLSVDNet Decoding is ``Metric-meaningful''}
\RevisionTieredReplace{M27}{\todo{movement in the latent space corresponds to the same distance in the input space}}{}{The norm-preserving representation calibrates latent radius to input radius. In particular, \(\|g(x)\|_2=\|x\|_2\) measures both radii from the origin. This identity does not assert preservation of arbitrary pairwise distances or path lengths.}
\RevisionReplace{B24}{%
This left-inverse of $v$ in the NLSVD of a neural network, or equivalently $g$ in the NLSVDNet, is a key distinction from an ordinary neural
latent space: a generic embedding may map distinct inputs to the same latent representation,
in which case no strict inverse exists on the data image.
The NLSVD lift is constructed to be injective, and NLSVDNet encourages the learned representation to retain that property, making the projected pseudo-inverse a structured pullback to $\mathbb{R}^m$.
Neural networks equipped with decoding architectures also learn a reconstruction heuristic for latent spaces.
While autoencoders indeed can render a left-inverse for a latent representation of an input, they do not
specify how gain is distributed between the encoder and decoder, so latent directions can be arbitrarily warped while preserving reconstruction quality.%
}{%
The left inverse of $v$ in an NLSVD---equivalently, of $g$ in NLSVDNet---is a key distinction from an ordinary neural latent space: a generic embedding may map distinct inputs to the same latent representation, in which case no strict inverse exists on the data image.
The NLSVD lift is constructed to be injective, and NLSVDNet encourages the learned representation to retain that property, making the projected pseudo-inverse a structured pullback to $\mathbb{R}^m$.
Autoencoders likewise learn a reconstruction map for latent codes, but reconstruction alone does not specify how gain is distributed between encoder and decoder, so latent directions can be arbitrarily warped while preserving reconstruction quality.%
}
\RevisionReplace{M27}{The NLSVDNet provides both,
and by preserving the norm in the embedding, ensures that perturbations are proportional to perturbations in the input space. We exploit this in the next section \ref{sec:adversarial_experiments}.}{NLSVDNet provides both. Norm preservation calibrates representation radius to input radius, while \(K\) identifies which components affect the output.}
\RevisionComment{B24}{Proofread metric-meaningful subsection: tighten autoencoder contrast; keep accepted M27 green text.}

\begin{proposition}[Necessity of $v$'s norm preservation for optimal latent traversal]
\label{prop:latent-metric-nonidentifiability}
Let \(f=Kg\) on a domain \(\Omega\), and suppose
\(g^{-L}(g(x))=x\) for all \(x\in\Omega\).
For any invertible matrix
\(S\in\R^{\ell\times\ell}\), define $\tilde g\triangleq Sg,\  \tilde K\triangleq KS^{-1},\ 
\tilde g^{-L}(z)\triangleq g^{-L}(S^{-1}z)$ for $z\in \tilde g(\Omega)$.
Then, \(\tilde K\tilde g=f\) and \(\tilde g^{-L}(\tilde g(x))=x\) on
\(\Omega\).  However, for any \(x,x'\in\Omega\), \mbox{$\|\tilde g(x')-\tilde g(x)\|_2=\|S(g(x')-g(x))\|_2$},
while the pulled-back perturbation remains \(x'-x\).  In particular, choosing
\(S=cI\) changes every nonzero latent perturbation norm by the arbitrary factor
\(c>0\) without changing either the input-output map or the exact left inverse.
\end{proposition}
\noindent\textit{Proof.}
Substitute:
\(\tilde K\tilde g=KS^{-1}Sg=Kg=f\) and
\(\tilde g^{-L}(\tilde g(x))=g^{-L}(S^{-1}Sg(x))=g^{-L}(g(x))=x\). \qed

\section{Constructing NLSVD from Black-box $f$}
\label{sec:construction}
\RevisionTieredReplace{M28}{\todo{reference the theorems that give us permission to do this}}{}{Theorem~\ref{thm:nlsvd} guarantees an exact NLSVD for maps with finite \(2\)-induced gain. Theorems~\ref{thm:compact-bound} and \ref{thm:architecture} establish this condition for the anchored neural-network classes considered here. Algorithm~\ref{alg:nlsvd-construction} turns the constructive proof of Theorem~\ref{thm:nlsvd} into a black-box estimation procedure.}
We consider estimating an NLSVD of a black-box classifier or regressor $f$ with bounded $2$-induced norm.
By ``black-box'' we mean we can observe $f(x)$ for any $x$, but do not have access to $f$'s weights.
\RevisionOptionalRemove{M28}{This section assumes \(f\), or its anchored deviation
\(f_\star(h)=f(x_\star+h)-f(x_\star)\), is in one of the finite-gain classes identified in Theorems~\ref{thm:compact-bound}, \ref{thm:architecture}.}
\begin{algorithm}[t]
\caption{Coordinatewise Gain-Lifting Construction}
\label{alg:nlsvd-construction}
\footnotesize
\begin{algorithmic}[1]
\INPUT Black-box $f: \mathbb{R}^m \to \mathbb{R}^n$, dataset $X \subset \mathbb{R}^m$, $\epsilon > 0$.
\OUTPUT Unitary (permutation) matrix $U$, diagonal matrix $\Sigma$, norm-preserving injective lift $v(x)$, left inverse $v^{-L}$.
\STATE Estimate gains $\alpha_i=\max_{x\in X}\tfrac{\|f_i(x)\|_2}{\|x\|_2}$. 
\STATE Sort indices $q$ so $\alpha_{q(1)}\ge\cdots\ge\alpha_{q(n)}$, and let $U$ be the corresponding permutation.
\STATE Set $\sigma_i\triangleq  \alpha_{q(i)}\sqrt{\tfrac{n}{1-\epsilon}}$ and $\Sigma\triangleq  \diag(\sigma_1,\ldots,\sigma_n)$.
\STATE Compute $\gamma(x)=1-\sum_{j=1}^{n}\tfrac{f_{q(j)}(x)^2}{\sigma_j^2\|x\|_2^2}$ and $\delta_i(x)=\tfrac{f_{q(i)}(x)}{\sigma_i\sqrt{\gamma(x)}}$.
\STATE Define $x_\delta(x)\triangleq [\vec{\delta}(x),x]^\top\in\mathbb{R}^{n\times m}$ and $v(x) \triangleq \tfrac{\|x\|_2\,x_\delta(x)}{\|x_\delta(x)\|_2}$.
\STATE For $z=[z_\delta,z_x]^\top\in\mathrm{Im}(v)\subset\mathbb{R}^{n\times m}$, set $v^{-L}(z)\triangleq \tfrac{\|z\|_2\,z_x}{\|z_x\|_2}$.
\end{algorithmic}
\end{algorithm}
\RevisionComment{B25}{Algorithm~1: $\mathrm{Im}(v)$; product space $\mathbb{R}^n\times\mathbb{R}^m$ (not $\mathbb{R}^{n\times m}$); restore ``define'' on the left-inverse line. Note: input/output dims here are $m\to n$, flipped relative to Theorem~\ref{thm:nlsvd}'s $n\to m$.}

We provide a cleaner proof (\RevisionReplace{B08}{Appendix}{Supplement}~\ref{pf:thm-nlsvd})
to illuminate the construction of NLSVD components.
Once one can identify $\sigma_i$ satisfying Eq.~(\ref{eq:sum_norms}), defining $U$, $\Sigma$, and $v$ is straightforward.
The central task is estimating $\|f_i\|_{2\to2}$.
Given a ``representative'' dataset $X$ (not necessarily $f$'s training set), one can simply observe the norms of $f_i$.
(Here ``representative'' means how well $\|f_i\|_{2\to2}\big|_X$ approximates $\|f_i\|_{2\to2}$.)
Notably, since $f(X)$ can be computed via batched inference, this step is relatively fast.
If gradients or reliable finite differences are available, the sampling estimate can be strengthened by gain search, e.g., gradient ascent on
\(\phi_i(x)=|f_i(x)|/\|x\|_2\) (or the anchored analogue), initialized from high-ratio data points.
This yields an empirical norm for the Algorithm~\ref{alg:nlsvd-construction} construction.
The construction yields a left inverse for $v$, and a pseudo-inverse for $f$.
\RevisionAdd{B28}{%
Supplement~\ref{sec:construction-validation} (Table~\ref{tab:nlsvd-construction-results}) checks the construction on MNIST and Fashion-MNIST NLSVDNet classifiers treated as black boxes.
Floating-point reconstruction, left-inverse, and norm-preservation errors stay near machine precision, so numerical error does not obstruct the constructive identities.
The substantive metric is gain recovery relative to the singular values of $K$: data sampling alone recovers about $64\%$, while gradient gain search raises this to $97.1\%$ (MNIST) and $95.8\%$ (Fashion-MNIST).%
}
\RevisionComment{B28}{Short summary: identity checks vs gain-recovery findings; details in supplement.}

\subsection{Discussion \& Limitations}

\subsubsection{What if $\|f_i\|_{2\to2}$ is under-estimated?}
Operationally, Algorithm~\ref{alg:nlsvd-construction} estimates each
\(\sigma_i\) as an upper bound on the observed ratios
\(|f_i(x)|/\|x\|_2\) (or the anchored analogue
\(|f_{\star,i}(h)|/\|h\|_2\)).
This requires possibly non-convex optimization over the input space, which presents a challenge for practical implementations.
If only black-box queries are available, the same quantities can be estimated from representative query batches.
The construction is also partially self-checking.
The lift contains the square root of $\gamma$ from step~4 of Algorithm~\ref{alg:nlsvd-construction}.
Therefore, an underestimated
\(\sigma_i\) is exposed by any validation point for which the residual becomes
negative, i.e., the construction will yield complex values, often causing an error in code and easily checked in use.
In experiments,
one should therefore report the most negative value of $\gamma$, since positive $\gamma$ is exactly the
finite-sample certificate that the learned \(\Sigma\) is large enough on the
tested domain.


\subsubsection{Can the black-box decomposition recover NLSVDNet?}
That is, can this construction recover NLSVDNet's $K$ and $g$?
Unfortunately, the construction is too structurally limited:
first, $U$ is a permutation matrix, not an arbitrary unitary;
second, Eq.~(\ref{eq:sum_norms}) must be satisfied, resulting in $\sigma_i$ that are inflated relative to the actual coordinate gains (up to a $\sqrt{m}$ factor in the aligned worst case; Supplement~\ref{rem:nlsvd-svd-limitations});
third, $v$ is limited to its specific structure in Eqs.~(\ref{eq:v_support})--(\ref{eq:v}).
In particular, the reconstruction algorithm cannot recover the internal parameters $(K,g)$ of an NLSVDNet factorization $f(x)=Kg(x)$.
This is demonstrated by a simple counterexample in \RevisionReplace{B08}{Appendix}{Supplement}~\ref{sec:construction-limitation}.
\RevisionReplace{M04}{Nevertheless, the NLSVD construction yields geometric properties and input-output sensitivities giving valuable and useful insights, which we exploit in Section \ref{sec:adversarial_experiments}.}{Nevertheless, the construction provides a left inverse, a norm-preserving lift, and finite-sample gain diagnostics.}
We note that the construction, applied to linear $f$, does not yield a linear $v$, because it does not find a general $U$.
An example is in \RevisionReplace{B08}{Appendix}{Supplement}~\ref{sec:nlsvd-linear-example}.
\RevisionComment{B27}{Proofread structural-limitation paragraph: fix ``i.e,.'', ``Eq.s'', and wording.}

\section{Experiments}\label{sec:experiments}

\begin{figure}[t]
\centering
\caption{\textbf{NLSVDNet pullbacks.} In each subplot, the top row is MNIST and the bottom row is FashionMNIST. Left: samples from the extended null space of a classification NLSVDNet $f=Kg$. Right: linear interpolation in output space visualized through the learned left inverse $g^{-L}$.
}
\label{fig:svdnet_visual_examples}

\begin{subfigure}[t]{0.48\linewidth}
\centering
\includegraphics[width=\linewidth]{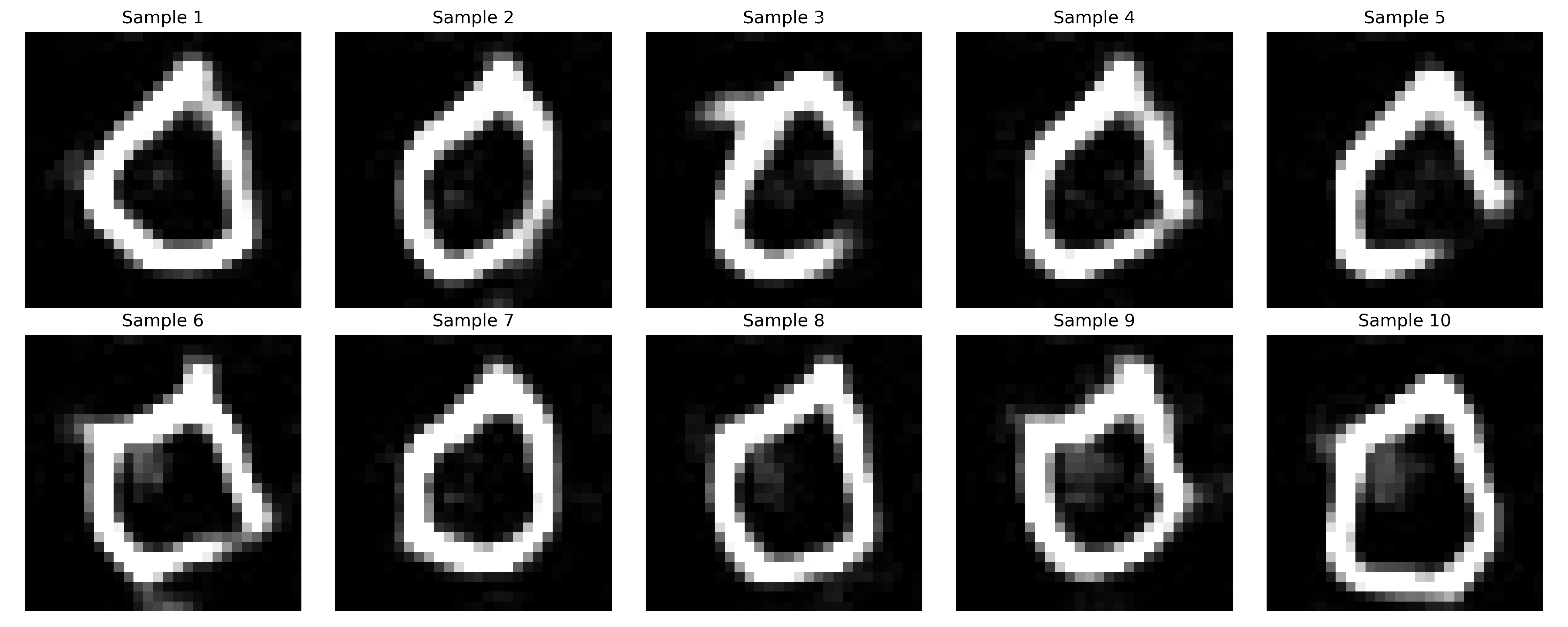}
\vspace{0.5ex}

\includegraphics[width=\linewidth]{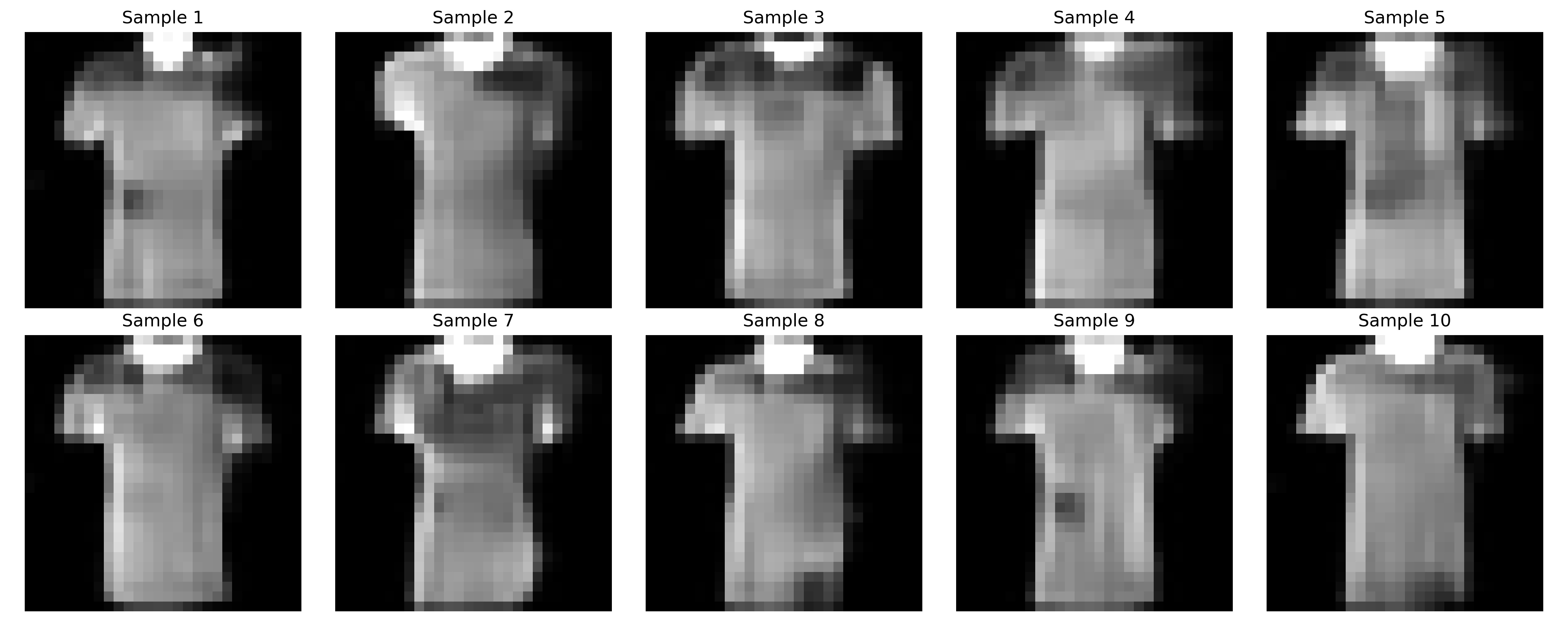}
\caption{Extended ``null space'' samples. For fixed logit $y_1$ and $z_{\mathcal{N}}\in\mathcal{N}(K)$, each image is generated as $g^{-L}(K^\dagger y_1+z_{\mathcal{N}})$. Nullspace pullback provides exploration and visualization of unrealized latent-spaces representations of each class.}
\label{fig:null_space_main}
\end{subfigure}
\hfill
\vrule width 0.4pt
\hfill
\begin{subfigure}[t]{0.48\linewidth}
\centering
\includegraphics[width=\linewidth]{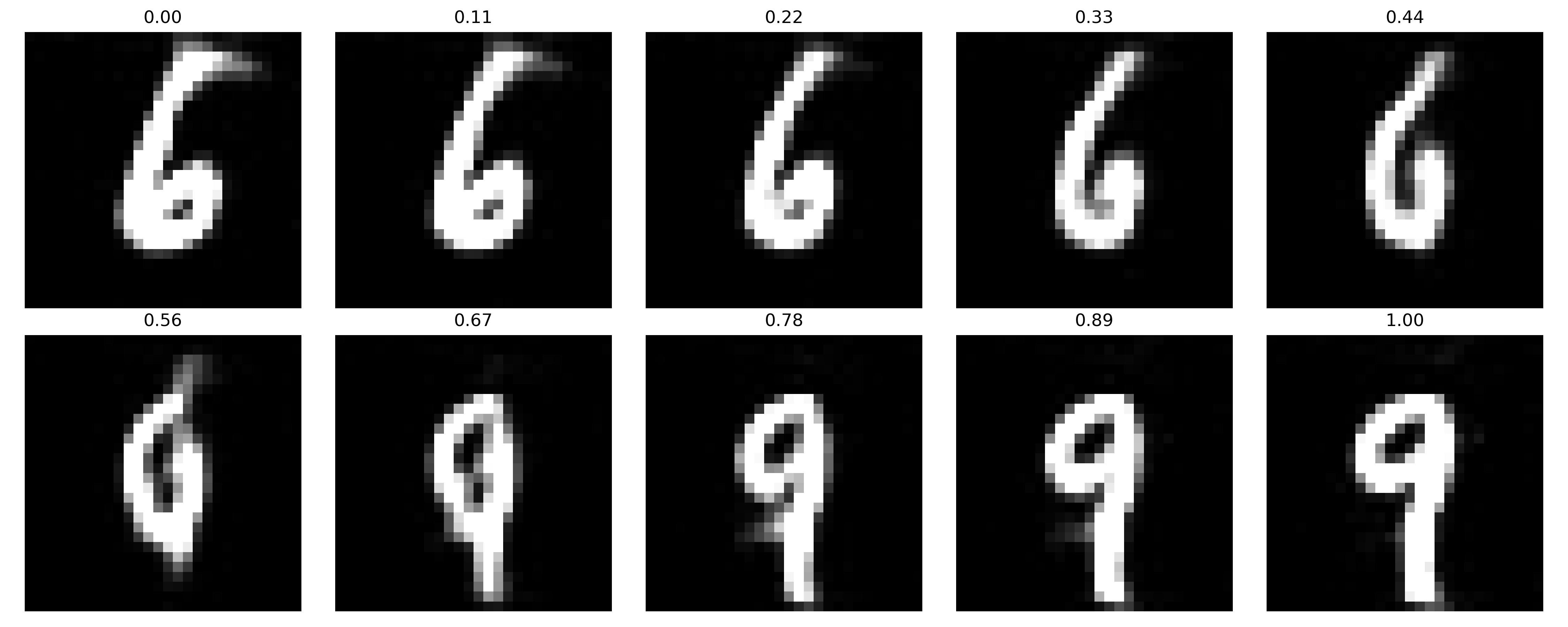}
\vspace{0.5ex}

\includegraphics[width=\linewidth]{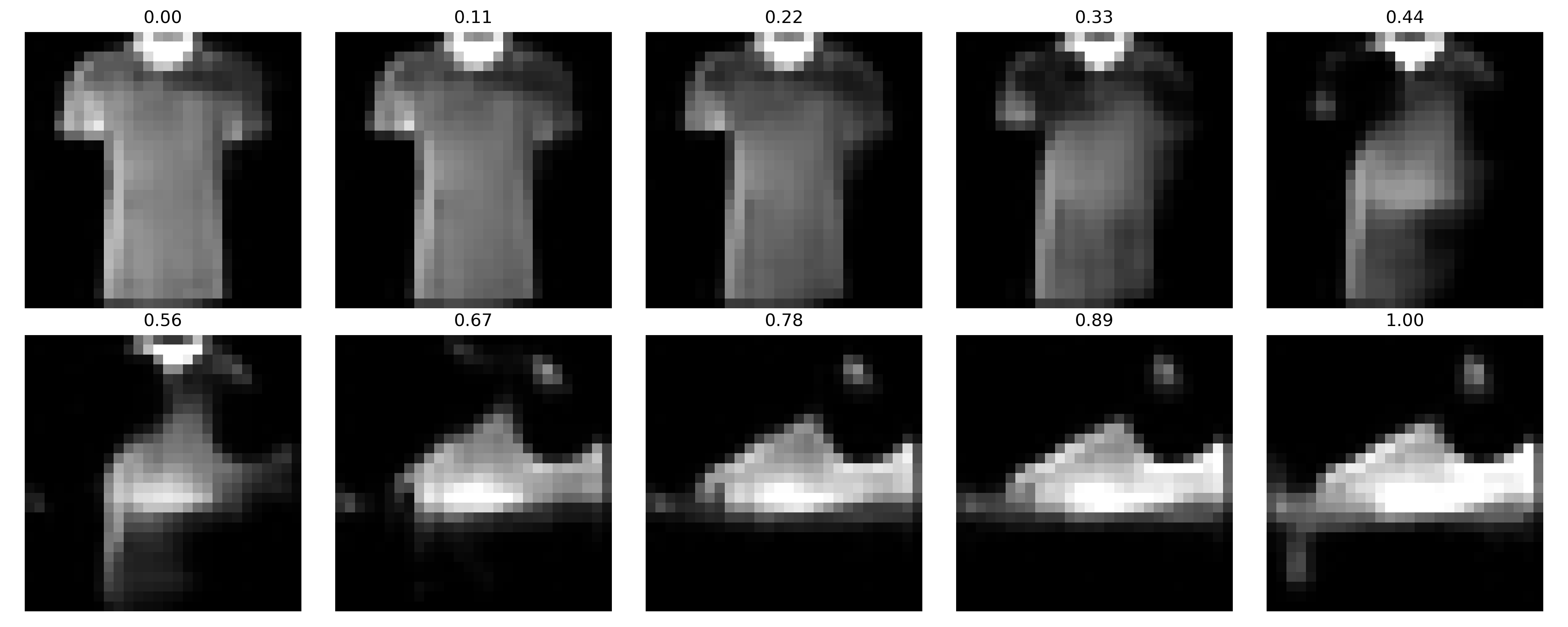}
\caption{Output-space interpolation. For logits $y_1,y_2$, set $y_t=(1-t)y_1+ty_2$; each image is generated as $g^{-L}(K^\dagger y_t)$. Row space pullback images show a non-linear deformation that is informative about the model's out-of-distribution performance.
}
\label{fig:row_space_main}
\end{subfigure}
\end{figure}

\subsection{NLSVDNet Examples}\label{sec:svdnet-examples}
\RevisionTieredReplace{M27}{\todo{make sure to connect back to proposition 1}
As with any en-/decoder pair that is approximately information-preserving, the row and null spaces of the final linear layer of the network (after the encoding) are meaningful relative to the input space, e.g., Definitions \ref{def:generalized_null_set_x} and \ref{def:generalized_row_set_x}.}{As with any en-/decoder pair that is approximately information-preserving, the row and null spaces of the final linear layer of the network (after the encoding) are meaningful relative to the input space, as described in Definitions \ref{def:generalized_null_set_x} and \ref{def:generalized_row_set_x}.}{Proposition~\ref{prop:latent-metric-nonidentifiability} shows that a decoder does not by itself assign a stable scale to a latent traversal. NLSVDNet rules out such arbitrary rescaling by requiring \(\|g(x)\|_2=\|x\|_2\). 
The radius of each decoded row- or null-space sample can therefore be interpreted relative to the input origin, while \(K\) determines which components affect the output.}
Figure \ref{fig:svdnet_visual_examples} provides visual examples of the distinction between row and null space traversals in the latent space: the first fixes a logit $y_1$ and applies $g^{-L}$ to elements in $K^{\dagger}(y_1+\mathcal{N}(K))$, while the second exploits the row-space of $K$ to interpolate between two logits.
Like en-/decoder architectures, the NLSVD components realized via NLSVDNet training facilitate model explainability---but unlike traditional encoding schemes, the embeddings are calibrated by Euclidean distance in the input space, meaning the SVD of the final layer can be used to rank the relative importance of embedded features.
See the Appendix \ref{sec:svdnet-examples-details} for training details.

\RevisionRemove{M27}{\todo{add a 2-d plot showing the distance traveled in latent space and distance traveled in input space. And then compare to autoencoder distances.}}
\RevisionComment{M27}{Radial calibration vs matched AE is baked into 52-nlsvd-vs-ae.tex. The old 51-radial-calibration-brian revision block is no longer input from main.}

\subsection{Radial Calibration AutoEncoder Comparison}\label{sec:radial-calibration}
\begin{figure}[t]
\centering
\caption{\textbf{Raw decoder radius along row- and null-space traversals.}
NLSVDNet remains on the radius-equality line, while the architecture-Matched AutoEncoder (Matched AE) does not calibrate decoded radius to latent radius.}
\label{fig:radial-distance-row-null}
\includegraphics[width=.9\linewidth]{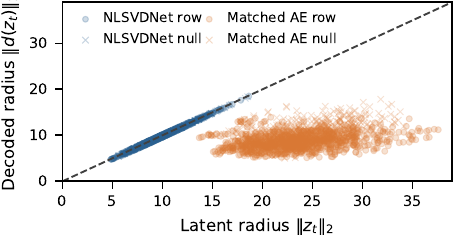}
\end{figure}

We isolate the decoder calibration used by these traversals by training an NLSVDNet and an architecture-matched autoencoder on MNIST. The models each contain \(1{,}622{,}809\) trainable parameters, begin from identical weights, receive the same minibatches, and attain held-out accuracies of \(97.45\%\) and \(97.55\%\), respectively. Figure~\ref{fig:radial-distance-row-null} compares latent radius with the radius of the raw decoded output after moving through the row and null spaces of the final linear layer. The mean relative radial error for NLSVDNet is \(4.13\times10^{-8}\) in the row space and \(4.50\times10^{-8}\) in the null space, compared with \(0.621\) and \(0.572\) for the matched autoencoder. The measurement uses radius from the zero input. Pairwise distance, traversal length, and arc length remain outside its scope. The result establishes radial calibration of the ambient decoder without asserting that the sampled latent points lie on the learned encoder image.
This exhibits a clear contrast between AEs and NLSVDNet approaches.

  \ifrevisionmarked
    \begingroup
    \color{RevisionRemoveRed}%
    \subsection{Input Perturbation Attack Experiment}
\label{sec:adversarial_experiments}
\todo{connect back to the proposition about distances being meaningful in the latent space.}
\todo{refer to C\&W as seminal}
Here we present an example of how the geometric insights made possible by NLSVD can be used to find an elegant and effective solution in a classical machine learning task---black-box adversarial input perturbation (or counterfactual generation \cite{zufiria2024multicriteria}).\IssueRef{M04}\IssueRef{M06}
Given a classifier $f$ and input $x_0$ with predicted class $u_0=\argmax_j f_j(x_0)$, we seek a small $\eta$ such that $\argmax_j f_j(x_0+\eta)\neq u_0$.
Unlike iterative high-dimensional attacks such as C\&W \cite{carlini2017towards}, the NLSVD estimate $f(x)=U\Sigma v(x)$ gives a guided 1D search: writing $u[i]=Ue_i$, the lifted boundary between $u_0=u[i_0]$ and target $u[i]$ is crossed when $\text{Gap}_i = \sigma_{i_0} v_{i_0}(x) - \sigma_i v_i(x) < 0$. Our attack chooses $i^*\neq i_0$ as the target class.
When $i_* = \argmin_{i}\text{Gap}_i$
the attack targets the closest alternate class in the model's latent space.\IssueRef{M05}
Next, estimate the contrastive gain direction\IssueRef{M07}
$\mathbf{d}=\nabla_x[\sigma_i v_i(x)-\sigma_{i_0}v_{i_0}(x)]_{x=x_0}$ by finite differences, and performs a 1D search over $\eta(r)=r\mathbf d/\|\mathbf d\|$.
The explicit algorithm is given in Appendix \ref{sec:algorithms}; the baseline is the original C\&W $L^2$ implementation, which optimizes $\min_{\eta}\|\eta\|_2^2+c\,\mathcal L(f(x_0+\eta))$ using ADAM and binary search over $c$.\footnote{\url{https://github.com/carlini/nn_robust_attacks/blob/master/l2_attack.py}\ .}\IssueRef{M01}\IssueRef{M02}


The NLSVD is computed once on \ManuscriptPerturbationNLSVDPoints \; held-out data points for each dataset.
We then evaluate on \ManuscriptPerturbationEvalPointsMnistFashion \; samples for MNIST and Fashion-MNIST, and on \ManuscriptPerturbationEvalPointsCifar \; samples for each of CIFAR-10 and CIFAR-100.
Detailed success, perturbation, and query metrics are reported in Table \ref{tab:perturbation-attack-results}.\IssueRef{M03}\IssueRef{M10}
\begin{table*}[h]
\centering
\begin{threeparttable}
\caption{{\bf Results of un-targeted perturbation attacks} compare the ``black-box'' NLSVD Directional Search (using reconstructed $U\Sigma v$) and \citealt{carlini2017towards} (C\&W) $L^2$ attack. The C\&W baseline uses the authors' original implementation with \texttt{MAX\_ITER = \ManuscriptCWMainMaxIterations}; Appendix \ref{sec:cw-maxiter-full} reports the larger-budget setting \texttt{MAX\_ITER = \ManuscriptCWBigMaxIterations}.}
\label{tab:perturbation-attack-results}
\revdel{White-box NLSVD Directional Search columns using $f=Kg$ were removed to avoid mixing access models in the main comparison.}
{
\setlength{\tabcolsep}{3pt}
\footnotesize
\begin{tabular}{lcccccccc}
\toprule
& \multicolumn{2}{c}{\textbf{MNIST}} & \multicolumn{2}{c}{\textbf{Fashion-MNIST}} & \multicolumn{2}{c}{\textbf{CIFAR-10}} & \multicolumn{2}{c}{\textbf{CIFAR-100}} \\
& $U\Sigma v$ & {C\&W} & $U\Sigma v$ & {C\&W} & $U\Sigma v$ & {C\&W} & $U\Sigma v$ & {C\&W} \\
\cmidrule(lr){2-3} \cmidrule(lr){4-5} \cmidrule(lr){6-7} \cmidrule(lr){8-9}
Success percent (\%) & \textbf{97.00} & 90.40 & \textbf{96.60} & 91.60 & \textbf{100.00} & 87.00 & \textbf{100.00} & 99.00 \\
Avg. Perturbation Norm & 18.30 & \textbf{13.47} & \textbf{6.58} & 11.53 & \textbf{0.22} & 13.70 & \textbf{0.19} & 14.89 \\
Avg. Queries/Sample & 1413.5 & \textbf{1360.0} & \textbf{1359.6} & 1360.0 & 3076.2 & \textbf{1360.0} & 3075.9 & \textbf{1360.0} \\
\bottomrule
\end{tabular}
}
\end{threeparttable}
\end{table*}

The black-box NLSVD attack achieves higher success than C\&W on all four datasets;
its perturbation norm is larger on MNIST, but one to three orders of magnitude smaller on Fashion-MNIST, CIFAR-10, and CIFAR-100, for similar to $1-2\times$ queries.\IssueRef{M08}\IssueRef{M09}\IssueRef{M10}
Average NLSVD wall-clock times are \ManuscriptBlackBoxMnistTimeSec s, \ManuscriptBlackBoxFashionTimeSec s, and \ManuscriptBlackBoxCifarTimeSec s on MNIST, Fashion-MNIST, and CIFAR-10, respectively.
\todo{maybe add in a model inversion (training data recovery) example????}
    \endgroup
    \IssueRef{M01}%
  \else
    \ifrevisionapplychanges
    \else
    \fi
  \fi

\RevisionComment{M01}{The full attack section and table are proposed for removal. The issue register retains the corrected numerical comparison, the norm-accounting error, the access-model mismatch, and the complete algebraic reduction.}
\subsection{Model Bias Detection} \label{sec:bias}
The NLSVD decomposition also exposes geometric observables for studying dataset imbalance and representation collapse.
Intuitively, $\sigma_i$ weight features to influence the output classes and may provide indicators of bias.
To examine the relationship, for each of MNIST, Fashion-MNIST, and CIFAR-10 we build biased variants by keeping a target class fully represented while under-sampling all other classes, across eight sampling ratios (log-scaled), and train one NLSVDNet per variant ($10$ classes $\times 8$ sampling ratios).
\RevisionAdd{M20}{ This intervention changes class proportions, training-set size, and the number of optimizer updates together.}\RevisionReplace{M17}{ Because the network factors as $f = K\,g$ with $K = U\Sigma V^{\top}$, the NLSVD cleanly separates the anisotropic gain structure $\Sigma$ from the nonlinear representation geometry $g(x)$, letting us probe each side directly---quantities that are typically entangled in ordinary networks.}{ Because \(f=K\,g\), we report the singular structure (\textit{gain}) of the head \(K\) separately from labeled feature statistics (\textit{representation}) computed on \(g(x)\).}
\RevisionComment{M20}{The observed trends respond to the complete under-sampling protocol. The experiment does not isolate class proportion from the accompanying changes in data volume and training updates.}

\begin{figure}[t]
\centering
\caption{\textbf{Model bias detection with NLSVD.}
NLSVDNets $f=Kg$ trained under increasing class imbalance.
As the minority sampling ratio decreases (stronger imbalance; leftward on the log $x$-axis), mean bias metrics are shown for MNIST, Fashion-MNIST, and CIFAR-10.
Shaded bands are $\pm 1$ std across target-class choices.
Left: gain-based metric from $K$; right: representation-based metric from $g$.
Table~\ref{tab:bias-cross-dataset} records additional observables.
}
\label{fig:bias-crossdataset}
\begin{subfigure}[t]{0.48\columnwidth}
\centering
\includegraphics[width=\linewidth]{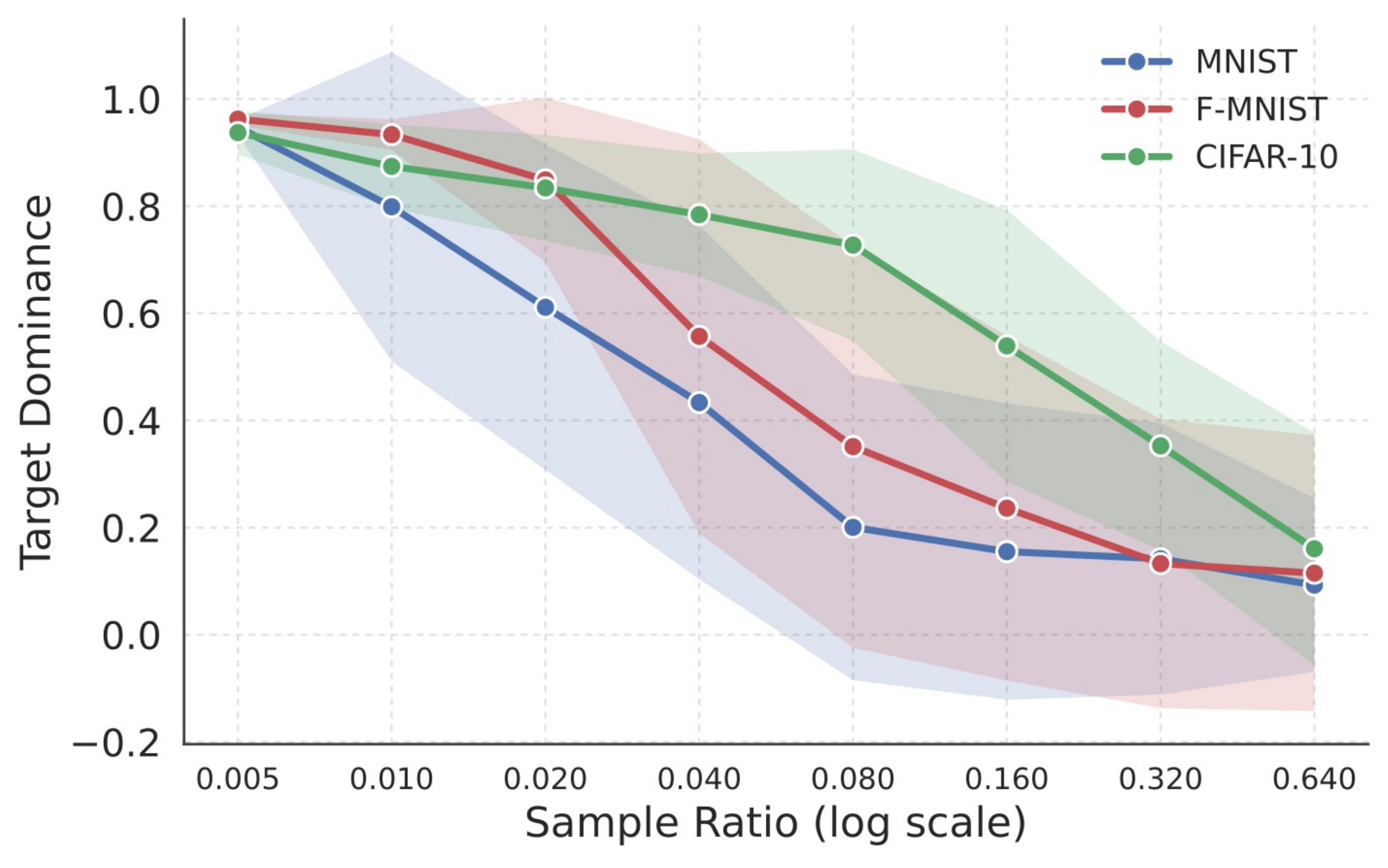}
\caption{Target dominance metric: fraction of signal-space energy on the leading singular direction of $K$. A consistent gain-side signal---dominance rises from ${\approx}0.1$ toward ${\approx}0.95$ as imbalance strengthens.}
\label{fig:bias-dominance}
\end{subfigure}
\hfill
\begin{subfigure}[t]{0.48\columnwidth}
\centering
\includegraphics[width=\linewidth]{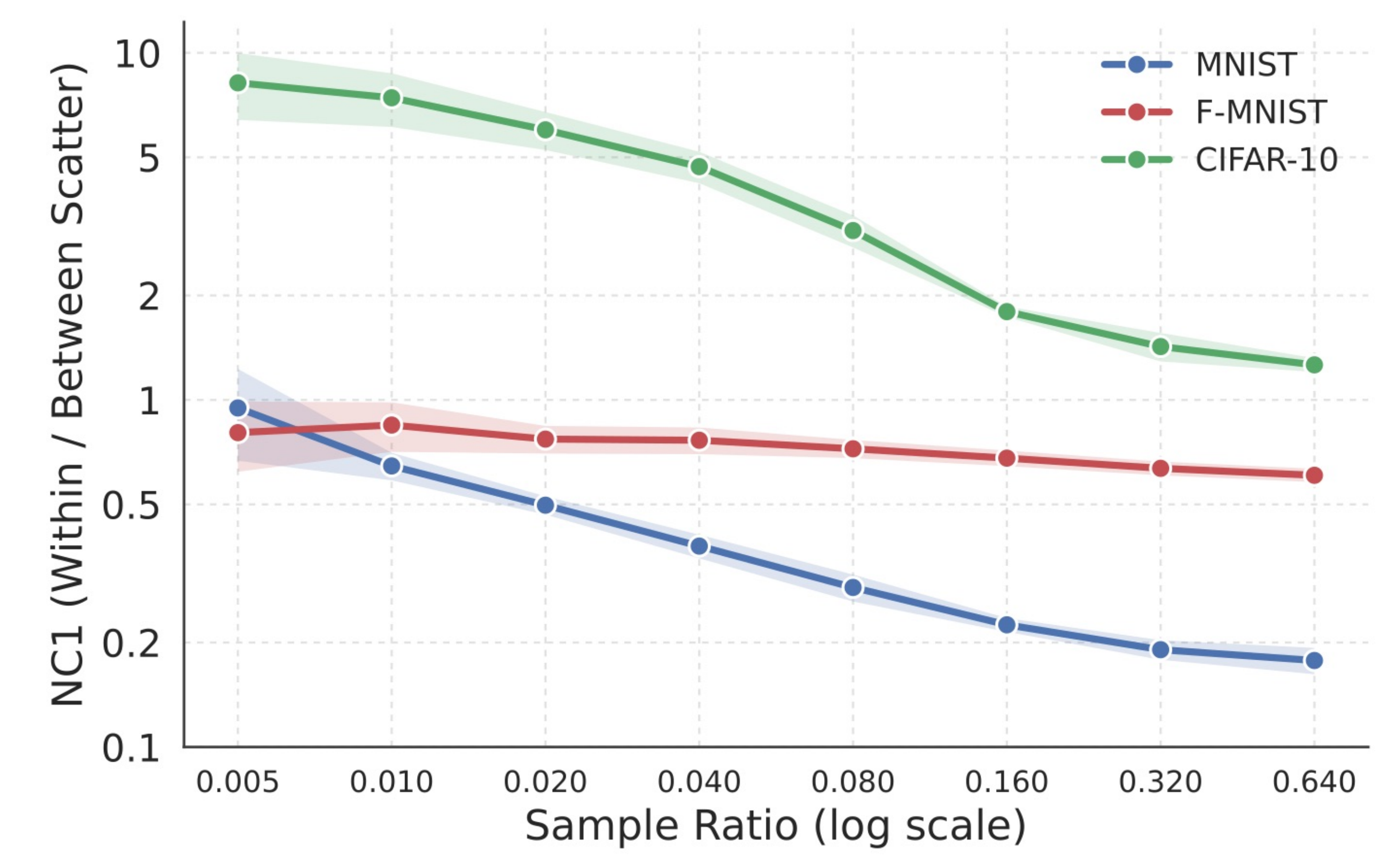}
\caption{NC1 metric: within- / between-class scatter ratio of $g(x)$ (log scale). The response is monotone but scales with task difficulty---mild on Fashion-MNIST, severe on CIFAR-10, as features become entangled.}
\label{fig:bias-nc1}
\end{subfigure}
\end{figure}

Figure~\ref{fig:bias-crossdataset} isolates the two robust signatures.
On the \emph{gain} side, we introduce \emph{target dominance}---the fraction of signal-space energy on the leading singular direction of $K$, averaged on target-class samples (with a minority-class analogue defined in Supplement~\ref{sec:bias-appendix})---and observe that it rises from ${\approx}0.1$ toward ${\approx}0.95$ as the leading direction aligns with the over-represented class; \RevisionReplace{M18}{the smallest singular values contract, so the directions associated with the under-sampled classes are progressively ``starved.''}{the smallest singular values also contract as under-sampling becomes more severe.}
On the \emph{representation} side, \RevisionReplace{M19}{class separability in $g(x)$ degrades under imbalance (a rising ratio of within- to between-class scatter, NC1); the direction of this effect is consistent, but its magnitude scales with task difficulty (negligible on Fashion-MNIST yet severe on CIFAR-10, where features become genuinely entangled).}{a standard labeled-feature diagnostic, NC1~\citep{papyan2020}, records a rising ratio of within-class to between-class scatter in \(g(x)\). The direction of this effect is consistent, and its magnitude scales with task difficulty. It is negligible on Fashion-MNIST and severe on CIFAR-10, where features become genuinely entangled.}
By contrast, the more commonly cited summaries---peaked spectra ($\sigma_1/\sigma_2$) and reduced effective rank~\citep{roy2007effective}---trend only in a dataset-dependent way, and we observe no migration of minority-class energy into the null space of $K$.
Table~\ref{tab:bias-cross-dataset} summarizes which observables replicate.
\RevisionReplace{M15}{These experiments are exploratory, but they point to a practical use of the NLSVD viewpoint: because imbalance imprints itself on the singular structure of $K$ and the geometry of $g(x)$, the direction and severity of dataset bias can be read directly off a trained model's coordinates—as spectral concentration, directional dominance, and representation collapse—\textit{without access to its training distribution}.}{These controlled experiments show that the reported head and feature statistics vary with the known target class and under-sampling factor in this intervention.}
Full metric descriptions, per-metric traces, and additional measurements are given in Supplement~\ref{sec:bias-appendix}.
\RevisionComment{M15}{A blinded CIFAR-10 follow-up does not support recovery of an unknown training distribution. On nine non-collapsed cells, the labeled NLSVD probe identifies \(6/9\) favored classes, while output frequency and mean predicted probability each identify \(8/9\). The probe has no uniquely correct cell. Overall accuracy also predicts severity more accurately than the probe.}
\RevisionComment{M17}{The \(6/9\) NLSVD probe result differs from the \(0/9\) radial-normalization ablation result, which supports an internal alignment distinction. It does not provide an incremental detector because ordinary output behavior contains every correct probe answer.}

\begin{table}[t]
\caption{\textbf{Bias geometry across datasets.} Does each NLSVD observable trend monotonically as imbalance increases? $\checkmark$~= clear trend, $\sim$~= weak / dataset-dependent, $\times$~= not observed.}
\label{tab:bias-cross-dataset}
\centering
\small
\setlength{\tabcolsep}{3.5pt}
\begin{tabular}{@{}lccc@{}}
\toprule
Observable (as imbalance $\uparrow$) & MNIST & F-MNIST & CIFAR \\
\midrule
Target--direction alignment $\uparrow$        & $\checkmark$ & $\checkmark$ & $\checkmark$ \\
Tail $\sigma_{\min}\!\downarrow$ (starvation) & $\checkmark$ & $\checkmark$ & $\checkmark$ \\
Feat. separability worsens (NC1 $\uparrow$)          & $\checkmark$ & $\sim$       & $\checkmark$ \\
Peaked spectrum $\sigma_1/\sigma_2\!\uparrow$          & $\sim$       & $\checkmark$ & $\checkmark$ \\
Effective-rank reduction                  & $\sim$       & $\checkmark$ & $\sim$       \\
Minority null-space migration                          & $\times$     & $\times$     & $\times$     \\
\bottomrule
\end{tabular}
\end{table}

\subsection{Membership Inference Robustness}
\label{sec:mia}

The NLSVD factorization $f = Kg$ also invites a privacy question: does separating an explicit gain matrix $K$ from a representation $g$ change how much a network memorizes its training data? We probe this with the Likelihood-Ratio Attack (LiRA), the seminal  membership-inference attack (MIA) \cite{carlini2022membership}, which infers whether a given example was in a model's training set.

We run online LiRA using CIFAR-10:
train $128$ shadow and $25$ target models per condition, each on an independent random $50\%$ membership mask of the training pool; fit per-example in/out
Gaussians over the shadows' logit-scaled confidences;
score each target by the resulting likelihood ratio (Appendix~\ref{sec:mia-appendix}). We compare three conditions: NLSVDNet,
a conventional CNN, and (to isolate architecture from objective) the NLSVDNet architecture retrained with a plain cross-entropy objective (bijection loss and $\sigma$-regularization removed). Since the generalization gap is the classical driver of leakage \cite{yeom2018privacy}, we match it between the NLSVDNet and the CNN. We report attack AUC and the true-positive rate at a low false-positive rate (TPR@$10^{-3}$), the regime that matters for privacy \cite{carlini2022membership}.

\begin{figure}[t]
\centering
\caption{\textbf{LiRA ROC (log-log) on CIFAR-10}. Solid line = median over the $25$ target models; shaded band = their $5$--$95$th percentile at each false-positive rate (target-to-target variability, not a
confidence interval). The NLSVDNet sits below the plain CNN throughout, with the gap widening in the low-FPR regime, indicating greater privacy.}
\label{fig:mia-roc}

\includegraphics[width=.7\linewidth]{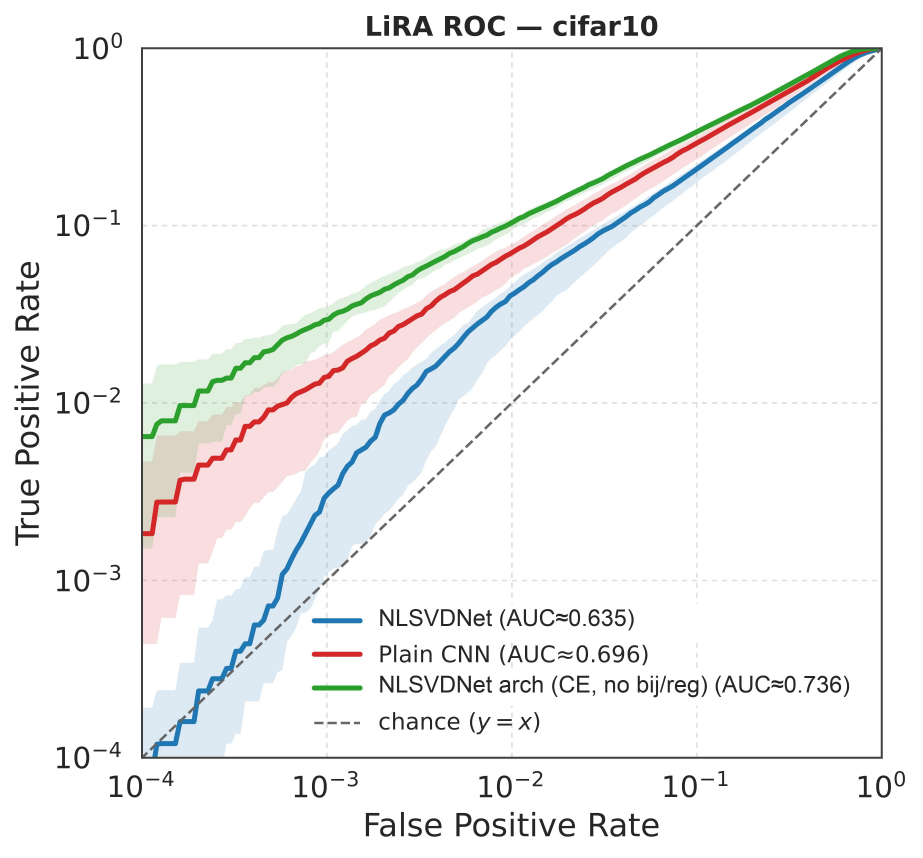}
\end{figure}

  \ifrevisionmarked
    \begingroup
    \color{RevisionRemoveRed}%
    At a matched $15.6\%$ generalization gap, full-objective NLSVDNet leaks less than a plain CNN (AUC $.635$ vs.\ $.696$; TPR@$10^{-3}$ $.003$ vs.\ $.014$; Figure~\ref{fig:mia-roc}), while being more accurate ($72.5\%$ vs.\ $67.1\%$). Training the same architecture with plain CE removes the advantage (AUC $.736$, TPR@$10^{-3}$ $.030$, gap $26.5\%$), so the gain is associated with the training regimen (Table~\ref{tab:mia-results}). Appendix~\ref{sec:mia-appendix} further reports that sweeping the $\sigma$-cutoff of $K$ alone does not move leakage, and that leakage tracks the generalization gap rather than spectral observables of $K$.
    \endgroup
    \IssueRef{M22}%
  \else
    \ifrevisionapplychanges
    \else
      At a matched $15.6\%$ generalization gap, full-objective NLSVDNet leaks less than a plain CNN (AUC $.635$ vs.\ $.696$; TPR@$10^{-3}$ $.003$ vs.\ $.014$; Figure~\ref{fig:mia-roc}), while being more accurate ($72.5\%$ vs.\ $67.1\%$). Training the same architecture with plain CE removes the advantage (AUC $.736$, TPR@$10^{-3}$ $.030$, gap $26.5\%$), so the gain is associated with the training regimen (Table~\ref{tab:mia-results}). Appendix~\ref{sec:mia-appendix} further reports that sweeping the $\sigma$-cutoff of $K$ alone does not move leakage, and that leakage tracks the generalization gap rather than spectral observables of $K$.
    \fi
  \fi

  \ifrevisionmarked
    \begingroup
    \color{RevisionAddGreen}%
    At a matched $15.6\%$ generalization gap, full-objective NLSVDNet leaks less than a plain CNN (AUC $.635$ vs.\ $.696$; TPR@$10^{-3}$ $.003$ vs.\ $.014$; Figure~\ref{fig:mia-roc}), while being more accurate ($72.5\%$ vs.\ $67.1\%$). Training the same architecture with plain CE removes the advantage (AUC $.736$, TPR@$10^{-3}$ $.030$, gap $26.5\%$), so the gain is associated with the training regimen (Table~\ref{tab:mia-results}). Appendix~\ref{sec:mia-appendix} further reports that sweeping the $\sigma$-cutoff of $K$ alone does not move leakage, and that leakage tracks the generalization gap rather than spectral observables of $K$.
    \endgroup
    \IssueRef{M22}%
  \else
    \ifrevisionapplychanges
      At a matched $15.6\%$ generalization gap, full-objective NLSVDNet leaks less than a plain CNN (AUC $.635$ vs.\ $.696$; TPR@$10^{-3}$ $.003$ vs.\ $.014$; Figure~\ref{fig:mia-roc}), while being more accurate ($72.5\%$ vs.\ $67.1\%$). Training the same architecture with plain CE removes the advantage (AUC $.736$, TPR@$10^{-3}$ $.030$, gap $26.5\%$), so the gain is associated with the training regimen (Table~\ref{tab:mia-results}). Appendix~\ref{sec:mia-appendix} further reports that sweeping the $\sigma$-cutoff of $K$ alone does not move leakage, and that leakage tracks the generalization gap rather than spectral observables of $K$.
    \fi
  \fi

\RevisionComment{M22}{The comparison supports an association at these operating points. It does not isolate architecture from objective because several loss terms and the generalization gap change together.}

\FloatBarrier

  \ifrevisionmarked
    \begingroup
    \color{RevisionRemoveRed}%
    \subsection{Membership Inference Attack}
LiRA \cite{carlini2022membership} is the strongest membership-inference attack available, but an expensive one: it trains a farm
of $128$ shadow models to estimate, for each point, the confidence a model would place on it had it \emph{not} been trained on it. Because the NLSVDNet exposes structure a conventional network
hides (the null space of the head $K$, its singular spectrum, the learned bijection $g^{-L}\!\circ g$) it is natural to ask whether an NLSVDNet can match LiRA's performance \emph{without} that farm by exploiting the structure instead. Using the same setup as in Section \ref{sec:mia} (reusing the same $25$ target models but \emph{zero} shadow models), and calibrating against a set of public non-members (the CIFAR-10 test set), our strongest attack turns out to be a shadow-free \emph{confidence} test. For a
candidate example $(x,y)$ we compute its logit-scaled confidence
$\phi(x)=\ell_y-\log\!\sum_{k\neq y}e^{\ell_k}=\log\tfrac{p_y}{1-p_y}$, which corresponds to the model's log-odds on the
true class $y$, and turn it into a per-example membership score by
$z$-scoring it against the reference non-members of the same class, so the threshold adapts to how confident the model is on that class in general. This exploratory experiment reaches AUC $0.596$, just short of LiRA's $0.635$ score in our setup. Clearly above a random guessing baseline of $0.5$, thi shows it is possible to achieve comparable leakage detection at a tiny fraction of the cost.\IssueRef{M21} More details in (Appendix~\ref{sec:mia-attack-appendix}).
    \endgroup
    \IssueRef{M21}%
  \else
    \ifrevisionapplychanges
    \else
      \subsection{Membership Inference Attack}
LiRA \cite{carlini2022membership} is the strongest membership-inference attack available, but an expensive one: it trains a farm
of $128$ shadow models to estimate, for each point, the confidence a model would place on it had it \emph{not} been trained on it. Because the NLSVDNet exposes structure a conventional network
hides (the null space of the head $K$, its singular spectrum, the learned bijection $g^{-L}\!\circ g$) it is natural to ask whether an NLSVDNet can match LiRA's performance \emph{without} that farm by exploiting the structure instead. Using the same setup as in Section \ref{sec:mia} (reusing the same $25$ target models but \emph{zero} shadow models), and calibrating against a set of public non-members (the CIFAR-10 test set), our strongest attack turns out to be a shadow-free \emph{confidence} test. For a
candidate example $(x,y)$ we compute its logit-scaled confidence
$\phi(x)=\ell_y-\log\!\sum_{k\neq y}e^{\ell_k}=\log\tfrac{p_y}{1-p_y}$, which corresponds to the model's log-odds on the
true class $y$, and turn it into a per-example membership score by
$z$-scoring it against the reference non-members of the same class, so the threshold adapts to how confident the model is on that class in general. This exploratory experiment reaches AUC $0.596$, just short of LiRA's $0.635$ score in our setup. Clearly above a random guessing baseline of $0.5$, thi shows it is possible to achieve comparable leakage detection at a tiny fraction of the cost.\IssueRef{M21} More details in (Appendix~\ref{sec:mia-attack-appendix}).
    \fi
  \fi

\RevisionComment{M21}{The strongest shadow-free score is calibrated confidence, which is available from any classifier. The NLSVD-native scores do not improve on it, so this section is proposed for removal from the main paper while the negative result remains documented in the issue register.}
\section{Conclusions}

We prove that  NLSVD theorem of~\citet{brown2025nlsvd} applies to most modern architectures providing global factorizations with strong, exploitable structure. In these coordinates, nonlinear representation geometry is separated from anisotropic gain structure, yielding nonlinear analogues of singular directions, row/null geometry, and pseudo-inverse structure. These coordinates preserve model left-invertibility up until a final linear layer\RevisionRemove{M04}{, and preserve perturbation norms between the embedding space and perturbations in the input space, which we exploit in a proof-of-concept for adversarial perturbation analysis}. Future work could include extended analysis of model bias and inversion, and targeted model adaptation, each uniquely leveraging the geometric results developed throughout this paper.
\RevisionComment{M04}{The conclusion must no longer claim that the invalid attack comparison demonstrates an NLSVD application.}

\paragraph{Software and Data.}
Code and experiment assets will be released publicly with this preprint.
Until the repository URL is finalized, the accompanying software package
matches the assets described in
Appendix~\ref{sec:reproducibility-assets-compute}.

\paragraph{Impact Statement.}
This paper advances basic machine-learning research.  The framework may support
algorithmic development and understanding of nonlinear functions, especially
neural networks.  As our results illustrate, techniques stemming from this work
may also identify, exacerbate, or help remediate AI-security issues.

\section*{Acknowledgments}
M.M. was funded by Swedish Army Staff.

\bibliography{references}

\appendix
\section{Additional Related Work}

\subsection{NLSVD Related Work} \label{sec:related-works-appendix}
The acronym \emph{GSVD} already refers to several matrix factorizations, most commonly a simultaneous factorization of a matrix pair $(A,B)$ sharing a column dimension \cite{vanLoan1976GeneralizingSVD,paigeSaunders1981TowardsGSVD}. We use NLSVD to refer to the nonlinear map factorization in this paper.
Other variants modify inner products, impose restrictions, or provide taxonomies of SVD-like matrix factorizations \cite{deMoorGolub1991generalizations, deMoorZha1991TreeGeneralizationsOSVD,deMoorGolub1991RSVDPropertiesApplications,zha1992NumericalAlgorithmRSVDTriplets}.
These are valuable but distinct from the nonlinear factorization studied here, where $f:\mathbb R^n\to\mathbb R^m$ is represented as $f(x)=U\Sigma v(x)$ with $v$ a norm-preserving injective lift.

Network invertibility is also a long-standing ML research area. Invertible networks support bijective mapping and density estimation \cite{kingma2018glow}, extend to ill-posed inverse problems \cite{ardizzone2019}, and often require stability controls such as Lipschitz constraints \cite{behrmann2018}. The NLSVD lift is different in that it applies to almost all modern architectures, and proves that there exists a representation that remains left invertibility up until a final linear layer.

\paragraph{Contrast with existing geometric tools.}
Jacobian saliency and SmoothGrad-style attributions \cite{simonyan2013saliency,smilkov2017smoothgrad} are local first-order probes; NTK analyses describe training dynamics in infinite-width or linearized regimes \cite{jacot2018neuralTangentKernel,lee2019wideLinearModels}; and spectral methods bound scalar gains through layerwise norms or Lipschitz certificates \cite{miyato2018spectralNormGAN,virmaux2018lipschitzRegularity}. The NLSVD is complementary---for a trained finite network, it gives a global factorization $f=U\Sigma v$ exposing coordinate gains, generalized row/null structure, and a constructive projected pseudo-inverse, rather than only a local linearization, a training-time limit, or a scalar upper bound.

\section{Appendix: Proofs}

\paragraph{Note on the form of this proof.}
~\citet{brown2025nlsvd} already prove Theorem~\ref{thm:nlsvd}
constructively.  We restate the construction here only to fix the notation used
by Algorithm~\ref{alg:nlsvd-construction}: the support component
$v_{\mathrm{support}}$ in~\eqref{eq:v_support}, the kernel component
$v_{\mathrm{kernel}}$ in~\eqref{eq:v_kernel}, and the gain
constraint~\eqref{eq:sum_norms} on the singular values are the quantities
estimated from data in the black-box setting.

\begin{proof}[Proof of Theorem \ref{thm:nlsvd}]\label{pf:thm-nlsvd}
It suffices to construct one admissible factorization, so take \(U=I_m\).  Write
\[
\alpha_i\triangleq \|f_i\|_{2\to2}
=\sup_{x\ne0}\frac{|f_i(x)|}{\|x\|_2}<\infty .
\]
Choose positive numbers \(\sigma_i\) such that
\begin{align}
\label{eq:sum_norms}
\rho^2\triangleq \sum_{i=1}^m \frac{\alpha_i^2}{\sigma_i^2}<1 .
\end{align}
Such a choice always exists; for example, one may take all \(\sigma_i=S\) with
\(S^2>\sum_i\alpha_i^2\), which also makes the diagonal nonincreasing.  Let
\[
\Sigma=\begin{bmatrix}\diag(\sigma_1,\ldots,\sigma_m)&\mathbf 0_{m\times n}\end{bmatrix}.
\]
For \(x\ne0\), define
\begin{align}
\label{eq:v_support}
v_{\mathrm{support}}(x)
&\triangleq \Sigma^\dagger f(x)
=
\begin{bmatrix}
f_1(x)/\sigma_1\\
\vdots\\
f_m(x)/\sigma_m\\
\mathbf 0_n
\end{bmatrix},\\
a(x)
&\triangleq
\left(1-\frac{\|v_{\mathrm{support}}(x)\|_2^2}{\|x\|_2^2}\right)^{1/2},\\
\label{eq:v_kernel}
v_{\mathrm{kernel}}(x)
&\triangleq
\begin{bmatrix}
\mathbf 0_m\\
a(x)x
\end{bmatrix},
\end{align}
and define, for \(x\ne0\),
\begin{equation}
\label{eq:v}
v(x)\triangleq v_{\mathrm{support}}(x)+v_{\mathrm{kernel}}(x).
\end{equation}
Set \(v(0)=0\).  Since \(f(0)=0\), this definition is consistent with the
factorization at the origin.

The scale \(a(x)\) is real and uniformly positive away from zero.  Indeed, for
\(x\ne0\),
\begin{align}
\label{eq:sigma_constraint}
\frac{\|v_{\mathrm{support}}(x)\|_2^2}{\|x\|_2^2}
&=
\sum_{i=1}^m \frac{f_i(x)^2}{\sigma_i^2\|x\|_2^2}
\le
\sum_{i=1}^m \frac{\alpha_i^2}{\sigma_i^2}
=\rho^2<1,
\end{align}
so \(a(x)\ge \sqrt{1-\rho^2}>0\).  Moreover
\(v_{\mathrm{support}}(x)\perp v_{\mathrm{kernel}}(x)\), and therefore
\begin{align*}
\|v(x)\|_2^2
&=\|v_{\mathrm{support}}(x)\|_2^2+a(x)^2\|x\|_2^2\\
&=\|v_{\mathrm{support}}(x)\|_2^2+
\|x\|_2^2-\|v_{\mathrm{support}}(x)\|_2^2\\
&=\|x\|_2^2 .
\end{align*}
Thus \(v\) is norm-preserving.

The same positive lower block gives injectivity.  If \(v(x)=v(y)\), then
norm preservation implies \(\|x\|_2=\|y\|_2\).  If either point is zero, both
are zero.  Otherwise the final \(n\) coordinates give \(a(x)x=a(y)y\).  Taking
norms and using \(\|x\|_2=\|y\|_2>0\) gives \(a(x)=a(y)\); substituting back
gives \(x=y\).

Finally,
\[
\Sigma v(x)
=
\begin{bmatrix}\diag(\sigma_i)&\mathbf 0_{m\times n}\end{bmatrix}
\begin{bmatrix}
f_1(x)/\sigma_1\\
\vdots\\
f_m(x)/\sigma_m\\
a(x)x
\end{bmatrix}
=f(x),
\]
and since \(U=I_m\), \(f=U\Sigma v\).  This proves the claim.
\end{proof}

\begin{remark}[Limitations relative to linear SVD]
\label{rem:nlsvd-svd-limitations}
Comparing to the original SVD, NLSVD has known limitations.
Theorem~\ref{thm:nlsvd} relaxes the properties of $V^\top$ from the linear SVD---which requires $V^\top$ to be \emph{bijective} and \emph{inner-product preserving}---to a lifting function $v(x)$ that is only required to be \emph{injective} and \emph{norm-preserving}.
Several other properties of the classical linear SVD are relinquished to accommodate the broader class of functions with a bounded 2-induced norm.
Notably, the largest singular value $\sigma_1$ in the proposed construction is not necessarily a tight upper bound on the $\|f\|_{2\to2}$ norm, as seen in the constraint in \eqref{eq:sigma_constraint}.
This contrasts with the classical SVD of a linear operator, where the largest singular value is exactly the operator's induced norm.
In particular, this bound may be loose by a multiplicative factor as large as $\sqrt{m}$, since $\|f(x)\|_2^2 = \sum f_i(x)^2 \leq m \cdot \max_i f_i(x)^2$.
This same \(\sqrt m\) scale appears in the coordinate construction.
Writing \(\alpha_i=\|f_i\|_{2\to2}\), for any \(\epsilon>0\) the simultaneous
choice \(\sigma_i=(\sqrt m+\epsilon)\alpha_i\) for nonzero \(\alpha_i\) (and
any positive \(\sigma_i\) when \(\alpha_i=0\)) gives
$\sum_i {\alpha_i^2}/{\sigma_i^2}
\le {m}/{(\sqrt m+\epsilon)^2}<1$,
so the construction makes coordinate gains tight only up to an arbitrarily
small excess over the worst-case \(\sqrt m\) factor.  This scale is unavoidable
for proportional coordinate bounds in the aligned worst case: if all coordinate
gains are attained together and \(\sigma_i=c\alpha_i\), then the constraint
\eqref{eq:sum_norms} requires \(m/c^2<1\), hence \(c>\sqrt m\).
\end{remark}
\RevisionComment{B29}{Moved from end of Section~2 (main) to sit with the constructive NLSVD proof.}

\subsection{The NLSVDNet left-inverse obstruction}
\begin{proposition}[Strict left inverses forbid full-dimensional compression]
\label{prop:left-inverse-no-compression}
Let \(\Omega\subset\R^m\) contain a nonempty open set, and let
\(g:\Omega\to\R^\ell\) be continuous.  If there is a map
\(g^{-L}:g(\Omega)\to\Omega\) satisfying \(g^{-L}(g(x))=x\) for every
\(x\in\Omega\), then \(g\) is injective on \(\Omega\).  Consequently, when
\(\ell<m\), no such continuous strictly left-invertible \(g\) can exist on a
full-dimensional input domain.
\end{proposition}
\begin{proof}[Proof of Proposition~\ref{prop:left-inverse-no-compression}]
\label{app:proof-left-inverse-no-compression}
If \(g(x_1)=g(x_2)\), then applying the left inverse gives
\[
x_1=g^{-L}(g(x_1))=g^{-L}(g(x_2))=x_2,
\]
so \(g\) is injective.  The final claim is the standard topological dimension
obstruction.  If \(\ell<m\), let \(J:\R^\ell\to\R^m\) be the coordinate
inclusion \(J(y)=(y,0)\).  Then \(J\circ g:\Omega\to\R^m\) is continuous and
injective.  By invariance of domain, \((J\circ g)(\Omega)\) must be open in
\(\R^m\) \citep[Thm.~62.3]{munkres2000topology}.  But
\((J\circ g)(\Omega)\subset \R^\ell\times\{0\}^{m-\ell}\), a set with empty
interior in \(\R^m\), a contradiction.  Therefore strict continuous
left-invertibility over a full-dimensional domain requires latent dimension at
least \(m\).
\end{proof}

\subsection{Architecture classes covered by the finite-gain argument}
\label{app:architecture-coverage}

\RevisionComment{B11}{Do not re-state Theorem~\ref{thm:architecture} in the supplement (avoids duplicate labels); proof follows below.}

\begin{table}[H]
\centering
\small
\caption{Architecture classes covered by the finite-gain NLSVD argument.}
\label{tab:architecture-counts-for-nlsvd}
\begin{tabularx}{\linewidth}{@{}>{\raggedright\arraybackslash}p{0.23\linewidth}>{\raggedright\arraybackslash}p{0.28\linewidth}>{\raggedright\arraybackslash}p{0.35\linewidth}@{}}
\toprule
Architecture class & Counts for NLSVD? & Required domain hypothesis \\
\midrule
MLPs, CNNs, ResNets, finite feedforward nets & Yes & Finite compositions of affine/convolutional maps, fixed pooling/downsampling, Lipschitz activations, residual sums, and stabilized normalizations; use anchoring whenever the chosen reference output is nonzero. \\
Fixed-length transformers & Yes & Compact fixed-length token domain and fixed attention mask; self-attention Lipschitz bounds apply on such compact domains. \\
Encoder--decoder cross-attention & Yes & Compact query--memory product domain with fixed source/target lengths. \\
Bounded variable-length sequence models & Yes, after fixing length & Pad or stratify by length and take the maximum over finitely many length pairs. \\
\ifrevisionmarked
\textcolor{RevisionAddGreen}{Stateful transformers} &
\textcolor{RevisionAddGreen}{Yes, over a fixed finite horizon} &
\textcolor{RevisionAddGreen}{Finite-dimensional carried state restricted to a compact domain and a fixed finite number of update steps.}\IssueRef{M26} \\
\else
\ifrevisionapplychanges
\ifcsname RevisionIncludeM26\endcsname
Stateful transformers & Yes, over a fixed finite horizon & Finite-dimensional carried state restricted to a compact domain and a fixed finite number of update steps. \\
\fi
\fi
\fi
Unbounded-length transformers or unbounded attention inputs & Not covered & Standard dot-product attention is not globally Lipschitz on unbounded domains. \\
Diffusion U-Net denoisers/samplers & Yes as deterministic maps & Fix time/conditioning and compact pixel/latent domain; the random law obtained by resampling noise is not itself the deterministic map decomposed here. \\
\bottomrule
\end{tabularx}
\end{table}

\subsection{Proofs for modern-architecture finite-gain claims}
\label{app:modern-architecture-proofs}

\begin{definition}[Block domains and norms]\label{def:block_domains_norms}
All architecture claims below are finite-dimensional.  Vector spaces
\(\R^d\) use the Euclidean norm.  Token or image tensors are identified with
Euclidean vectors by vectorization; for matrices \(X\in\R^{n\times d}\) this is
the Frobenius norm \(\|X\|_F^2=\sum_{i,j}X_{ij}^2\).  Product domains such as a
decoder-query and encoder-memory pair use
\[
\|(x_1,\ldots,x_k)\|_{\rm prod}^2=\sum_{j=1}^k\|x_j\|_2^2 .
\]
For a map \(H:\X\to\mathcal Y\), \(\mathrm{Lip}(H;\X)\le L\) means
\(\|H(x)-H(y)\|\le L\|x-y\|\) for all \(x,y\in\X\), using the norms just defined.
\end{definition}

\begin{definition}[Elementary network blocks]\label{def:elementary_network_blocks}
The non-attention blocks used in Theorem~\ref{thm:architecture} are the following
finite-dimensional maps, with fixed weights.
\begin{enumerate}
\item An affine layer is \(T(x)=Ax+b\).
\item A convolutional layer with fixed input/output tensor shapes is the linear
map \(C\) represented by that convolution.  Fixed strided downsampling and
average pooling are also linear maps.
\item A finite-window max-pooling layer with fixed windows \(W_j\) is
\((P_{\max}x)_j=\max_{i\in W_j}x_i\).
\item A coordinatewise activation is \(\Phi(x)_j=\phi(x_j)\).
\item A stabilized normalization block acts on a feature group of size \(r\) by
\begin{align*}
N_{\gamma,\beta,\epsilon}(x)
&=\Gamma\frac{Px}{\sqrt{r^{-1}\|Px\|_2^2+\epsilon}}+\beta,\\
P&=I-\frac{1}{r}\mathbf 1\mathbf 1^\top,\qquad \epsilon>0,
\end{align*}
where \(\Gamma=\diag(\gamma)\).  Applying this formula independently over
several fixed groups covers LayerNorm and the same epsilon-stabilized groupwise
normalization pattern.
\item A residual block has the form \(R(x)=x+G(x)\), where \(G\) is another
Lipschitz block with the same input and output dimension.
\item A concatenation block is \(C_F(x)=(F_1(x),\ldots,F_k(x))\).
\item A feedforward subnetwork is a finite composition of the preceding blocks.
\end{enumerate}
\end{definition}

\begin{lemma}[Elementary block Lipschitz bounds]
\label{lem:lipschitz-calculus}
The elementary blocks above have finite Lipschitz constants under the following
explicit bounds.
\begin{enumerate}
\item If \(T(x)=Ax+b\), then \(\mathrm{Lip}(T)\le \|A\|_2\), the spectral norm
of \(A\)~\cite{miyato2018spectralNormGAN,gouk2021enforcingLipschitz}.
\item If \(C\) is a convolutional, fixed strided downsampling, or average-pooling
layer, then \(\mathrm{Lip}(C)\le \|C\|_2\); Sedghi et
al.~\cite{sedghi2019convSingularValues} characterize these singular/operator
norms for standard convolutional layers.  If \(P_{\max}\) is finite-window
max-pooling and each input coordinate appears in at most \(M\) pooling windows,
then \(\mathrm{Lip}(P_{\max})\le\sqrt M\), so nonoverlapping max-pooling is
\(1\)-Lipschitz.
\item If \(\phi\) is \(L_\phi\)-Lipschitz on the interval containing all
coordinates of \(\X\), then \(\Phi\) is \(L_\phi\)-Lipschitz on \(\X\).  ReLU is
globally \(1\)-Lipschitz; smooth activations with bounded derivative on a compact
interval are Lipschitz there by the mean-value theorem.
\item The stabilized normalization block satisfies
\[
\mathrm{Lip}(N_{\gamma,\beta,\epsilon})\le \frac{2\|\Gamma\|_2}{\sqrt{\epsilon}}.
\]
This is the standard LayerNorm form of \citet{ba2016layerNormalization},
with the positive numerical stability parameter made explicit; inference-mode
BatchNorm is affine once its running statistics are fixed.  Training-time
BatchNorm is covered only if the whole batch is treated as the deterministic
input; stochastic training behavior is not the map considered here.
\item If \(F_j:\X\to\R^{d_j}\) are \(L_j\)-Lipschitz, then
\[
\mathrm{Lip}\big((F_1,\ldots,F_k);\X\big)
\le \Big(\sum_{j=1}^k L_j^2\Big)^{1/2}.
\]
If \(F,G:\X\to\R^d\) are \(L_F,L_G\)-Lipschitz, then
\[
\mathrm{Lip}(F+G;\X)\le L_F+L_G ;
\]
in particular \(\mathrm{Lip}(I+G;\X)\le1+L_G\).  If
\(F:\mathcal{Y}\to\R^p\) and \(G:\X\to\mathcal{Y}\) are Lipschitz, then
\[
\mathrm{Lip}(F\circ G;\X)\le \mathrm{Lip}(F;\mathcal{Y})\,\mathrm{Lip}(G;\X).
\]
\end{enumerate}
\end{lemma}

\begin{proof}[Proof of Theorem~\ref{thm:compact-bound}]
For \(h\in\X_\star\setminus\{0\}\), both \(x_\star+h\) and \(x_\star\) lie in
\(\X\), so
\[
\frac{\|f_\star(h)\|_2}{\|h\|_2}
=\frac{\|f(x_\star+h)-f(x_\star)\|_2}{\|h\|_2}
\le L .
\]
Taking the supremum gives \(\|f_\star\|_{2\to2;\X_\star}\le L<\infty\).
\end{proof}

\begin{proof}[Proof of Lemma~\ref{lem:lipschitz-calculus}]
For affine maps,
\[
\|T(x)-T(y)\|_2=\|A(x-y)\|_2\le \|A\|_2\|x-y\|_2 .
\]
A convolutional layer, fixed strided downsampling, and average pooling are
special linear maps after fixing input/output tensor shapes, so the same
inequality applies with \(A=C\).  For max-pooling,
\begin{align*}
|(P_{\max}x)_j-(P_{\max}y)_j|
&\le \max_{i\in W_j}|x_i-y_i|\\
&\le \left(\sum_{i\in W_j}|x_i-y_i|^2\right)^{1/2}.
\end{align*}
Squaring and summing over \(j\) gives
\[
\|P_{\max}x-P_{\max}y\|_2^2
\le \sum_j\sum_{i\in W_j}|x_i-y_i|^2
\le M\|x-y\|_2^2 .
\]
For a coordinatewise activation,
\begin{align*}
\|\Phi(x)-\Phi(y)\|_2^2
&=\sum_i|\phi(x_i)-\phi(y_i)|^2\\
&\le L_\phi^2\sum_i|x_i-y_i|^2\\
&=L_\phi^2\|x-y\|_2^2 .
\end{align*}

For stabilized normalization, put \(z=Px\) and
\(q(x)=\sqrt{r^{-1}\|z\|_2^2+\epsilon}\).  Since \(P=P^\top=P^2\) and
\(\|P\|_2=1\), write \(D_xG(x)[u]\) for the Fréchet derivative of a
map \(G\) at \(x\), applied to the perturbation \(u\); equivalently
\(D_xG(x)[u]=\frac{d}{dt}|_{t=0}G(x+tu)\).  Then
\[
D_x(Px)[u]=Pu,\qquad
D_xq(x)[u]=\frac{\langle Px,Pu\rangle}{r\,q(x)} ,
\]
where the second identity follows by differentiating
\(q(x)^2=r^{-1}\|Px\|_2^2+\epsilon\).  Thus, by the quotient rule,
\[
D_x\!\left(\frac{Px}{q(x)}\right)[u]
=\frac{Pu}{q(x)}
-\frac{Px\,\langle Px,Pu\rangle}{r\,q(x)^3}.
\]
Because \(q(x)\ge\sqrt{\epsilon}\) and \(r^{-1}\|Px\|_2^2\le q(x)^2\),
\[
\left\|D_x\!\left(\frac{Px}{q(x)}\right)[u]\right\|_2
\le \frac{\|u\|_2}{q(x)}
+\frac{\|Px\|_2^2\|u\|_2}{r\,q(x)^3}
\le \frac{2}{\sqrt{\epsilon}}\|u\|_2 .
\]
Multiplication by \(\Gamma\) gives the claimed bound.  Biases do not affect
Lipschitz constants.  Applying the same estimate independently to finitely many
fixed groups gives a finite groupwise normalization constant.  For concatenation,
\begin{align*}
&\|(F_1,\ldots,F_k)(x)-(F_1,\ldots,F_k)(y)\|_2^2\\
&\quad=\sum_j\|F_j(x)-F_j(y)\|_2^2\\
&\quad\le \Big(\sum_j L_j^2\Big)\|x-y\|_2^2,
\end{align*}
\begin{align*}
\|(F+G)(x)-(F+G)(y)\|_2
&\le \|F(x)-F(y)\|_2\\
&\quad+\|G(x)-G(y)\|_2\\
&\le (L_F+L_G)\|x-y\|_2,
\end{align*}
\begin{align*}
\|(I+G)(x)-(I+G)(y)\|_2
&\le \|x-y\|_2\\
&\quad+\|G(x)-G(y)\|_2\\
&\le (1+L_G)\|x-y\|_2,
\end{align*}
and for composition,
\begin{align*}
\|F(G(x))-F(G(y))\|_2
&\le L_F\|G(x)-G(y)\|_2\\
&\le L_F L_G\|x-y\|_2 .
\end{align*}
\end{proof}

\begin{lemma}[Cross-attention is Lipschitz on compact fixed-length domains]
\label{lem:cross-attention-lipschitz}
Fix query length \(n_q\), memory length \(n_m\), finite feature dimensions, and a
fixed finite additive mask \(M\).  For one attention head, define
\[
\begin{gathered}
Q=X_qW_Q,\qquad K=X_mW_K,\qquad V=X_mW_V,\\
A=\mathrm{softmax}_{\rm row}\!\left(\frac{QK^\top}{\sqrt{d_k}}+M\right),
\qquad H(X_q,X_m)=AV .
\end{gathered}
\]
Let \(\X_q\times\X_m\) be compact and set
\begin{align*}
B_q&=\sup_{\X_q}\|X_q\|_F,&
B_m&=\sup_{\X_m}\|X_m\|_F,\\
a_Q&=\|W_Q\|_2,& a_K&=\|W_K\|_2,& a_V&=\|W_V\|_2,
\end{align*}
\begin{align*}
R_Q&=a_QB_q,& R_K&=a_KB_m,\\
R_V&=a_VB_m,\\
L_S&=\frac{\sqrt{(R_Ka_Q)^2+(R_Qa_K)^2}}{\sqrt{d_k}} .
\end{align*}
Then, with the product norm
\(\Delta^2=\|X_q-X_q'\|_F^2+\|X_m-X_m'\|_F^2\),
\[
\|H(X_q,X_m)-H(X_q',X_m')\|_F
\le \bigl(\sqrt{n_q}\,a_V+R_VL_S\bigr)\Delta .
\]
Thus one cross-attention head is Lipschitz. A finite multi-head cross-attention
module followed by an output projection is Lipschitz with constant at most
\[
\|W_O\|_2\left(\sum_{h=1}^H
\bigl(\sqrt{n_q}\,a_{V,h}+R_{V,h}L_{S,h}\bigr)^2\right)^{1/2}.
\]
Fixed support masks with \(-\infty\) entries obey the same bound after deleting
the masked coordinates from each row, provided every query row has at least one
unmasked key.
\end{lemma}

\begin{proof}
First, row-wise softmax is \(1\)-Lipschitz in Euclidean norm.  If
\(p=\mathrm{softmax}(s)\), then
\[
D\mathrm{softmax}(s)=\diag(p)-pp^\top ,
\]
and for every vector \(z\),
\begin{align*}
z^\top(\diag(p)-pp^\top)z
&=\sum_jp_jz_j^2-\Big(\sum_jp_jz_j\Big)^2\\
&\le \sum_jp_jz_j^2\\
&\le \|z\|_2^2 .
\end{align*}
Hence \(\|D\mathrm{softmax}(s)\|_2\le1\), and the mean-value theorem gives
\(\|\mathrm{softmax}(s)-\mathrm{softmax}(s')\|_2\le\|s-s'\|_2\).  Applying this
row by row gives
\[
\|A-A'\|_F\le \|S-S'\|_F,\qquad
S=\frac{QK^\top}{\sqrt{d_k}}+M .
\]
The mask cancels because it is fixed.  Next,
\[
\begin{gathered}
\|Q-Q'\|_F\le a_Q\|X_q-X_q'\|_F,\\
\|K-K'\|_F\le a_K\|X_m-X_m'\|_F,\\
\|V-V'\|_F\le a_V\|X_m-X_m'\|_F,
\end{gathered}
\]
and compactness gives the same bounds for primed and unprimed variables:
\(\|Q\|_F,\|Q'\|_F\le R_Q\), \(\|K\|_F,\|K'\|_F\le R_K\), and
\(\|V\|_F,\|V'\|_F\le R_V\).  Therefore
\[
\begin{aligned}
\|QK^\top-Q'K'^\top\|_F
&\le \|(Q-Q')K^\top\|_F\\
&\quad+\|Q'(K-K')^\top\|_F\\
&\le R_Ka_Q\|X_q-X_q'\|_F\\
&\quad+R_Qa_K\|X_m-X_m'\|_F,
\end{aligned}
\]
so by Cauchy's inequality,
\[
\|S-S'\|_F
\le L_S\Delta .
\]
Finally, each row of \(A\) is a probability vector, so
\(\|A\|_2\le\|A\|_F\le\sqrt{n_q}\).  Thus
\[
\begin{aligned}
\|AV-A'V'\|_F
&\le \|A(V-V')\|_F+\|(A-A')V'\|_F\\
&\le \|A\|_2\|V-V'\|_F+\|A-A'\|_F\|V'\|_F\\
&\le \sqrt{n_q}\,a_V\|X_m-X_m'\|_F+R_VL_S\Delta\\
&\le \bigl(\sqrt{n_q}\,a_V+R_VL_S\bigr)\Delta .
\end{aligned}
\]
The multi-head bound follows by the concatenation inequality in
Lemma~\ref{lem:lipschitz-calculus} and one final linear projection.
\end{proof}

\begin{remark}[Self-attention as a fixed-length special case]
A self-attention head is the diagonal restriction \(S(X)=H(X,X)\) of the
cross-attention map above, with query and memory equal.  If the one-head
cross-attention constant is \(L_H\) in the product norm, then
\begin{align*}
\|S(X)-S(X')\|_F
&\le L_H\sqrt{\|X-X'\|_F^2+\|X-X'\|_F^2}\\
&=\sqrt{2}\,L_H\|X-X'\|_F .
\end{align*}
Thus fixed-length self-attention with a fixed mask is Lipschitz on compact token
domains by the same algebra.  We cite \citet{castin2024smoothattention}
in the main theorem because they give sharper quantitative bounds for masked and
unmasked self-attention and analyze normalization effects; \citet{kim2021selfattentionlipschitz} explain why no unbounded-domain global
Lipschitz claim should be made for standard dot-product attention.
\end{remark}

\begin{proof}[Proof of Theorem~\ref{thm:architecture}]
For each non-attention block, Lemma~\ref{lem:lipschitz-calculus} gives a finite
constant \(L_t\): affine, convolutional, and linear pooling/downsampling layers
are bounded linear maps, finite-window max-pooling is Lipschitz by
Lemma~\ref{lem:lipschitz-calculus},
coordinatewise activations are Lipschitz by assumption, stabilized normalization
has finite constant because \(\epsilon>0\), residual additions use
\(\mathrm{Lip}(I+G)\le1+\mathrm{Lip}(G)\), and feedforward sublayers are finite
compositions.  For fixed-length self-attention and masked self-attention on
compact token domains, \citet{castin2024smoothattention} provide
finite constants; \citet{kim2021selfattentionlipschitz} show why this
compact-domain restriction is necessary on unbounded domains.  For fixed-length
encoder--decoder cross-attention, Lemma~\ref{lem:cross-attention-lipschitz}
gives a finite constant.

This covers the usual named architectures by reduction to their finite
computation graphs.  MLPs are finite compositions of affine maps and
coordinatewise activations.  CNNs replace affine maps by convolutional and
pooling/downsampling operators.  ResNets add residual maps \(I+G\), which are
covered above.  A deterministic U-Net denoiser at fixed time and fixed
conditioning is likewise a finite graph of convolutions, activations, stabilized
normalizations, residual or skip concatenations, pooling/downsampling, and
sometimes attention.  A finite-step diffusion sampler is
covered only after fixing the time schedule, conditioning, and noise path, so it
is an ordinary deterministic finite composition; the theorem is not a statement
about the random law obtained by resampling noise.

It remains only to multiply constants.  Put \(F_t=H_t\circ\cdots\circ H_1\).
Inductively,
\[
\mathrm{Lip}(F_t;\X_0)
\le \mathrm{Lip}(H_t;\X_{t-1})\,\mathrm{Lip}(F_{t-1};\X_0)
\le \prod_{j=1}^t L_j .
\]
At \(t=T\), this gives \(\mathrm{Lip}(f;\X_0)\le\prod_{t=1}^T L_t<\infty\).
The anchored finite-gain statement is then Theorem~\ref{thm:compact-bound}.
For bounded variable length, the possible pairs \((n_s,n_t)\) form a finite set,
so
\[
L_{\max}:=\max_{(n_s,n_t)}\prod_t L_{t,n_s,n_t}<\infty .
\]
This gives the stated uniform family bound.
\end{proof}

\FloatBarrier

\section{Conceptual Figures}
\begin{figure}[H]
    \centering
    \includegraphics[width=0.78\columnwidth]{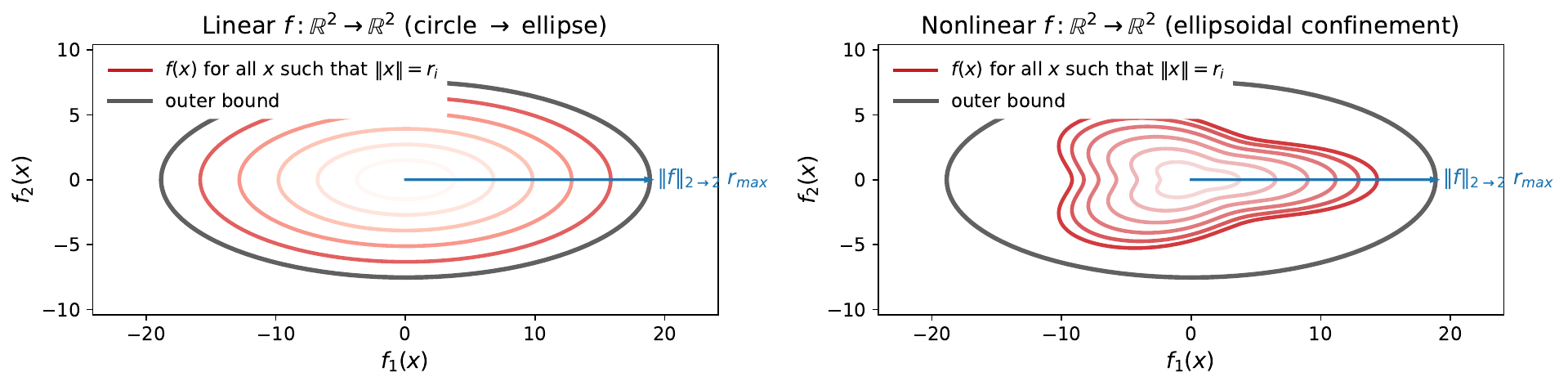}
    \caption{\textbf{Linear and nonlinear maps under a common induced-norm bound.}
    Red curves show $\{f(x):\|x\|_2=r_i\}$ for increasing radii $r_i$.
    A linear map sends the circles to ellipses (left), whereas a nonlinear map
    may produce non-elliptical images (right). In both cases the gray ellipse
    records the outer bound supplied by $\|f\|_{2\to2}$ at the largest plotted
    radius.}
    \label{fig:conceptual_figure_combined}
\end{figure}

\begin{figure}[H]
    \centering
    \includegraphics[width=0.78\columnwidth]{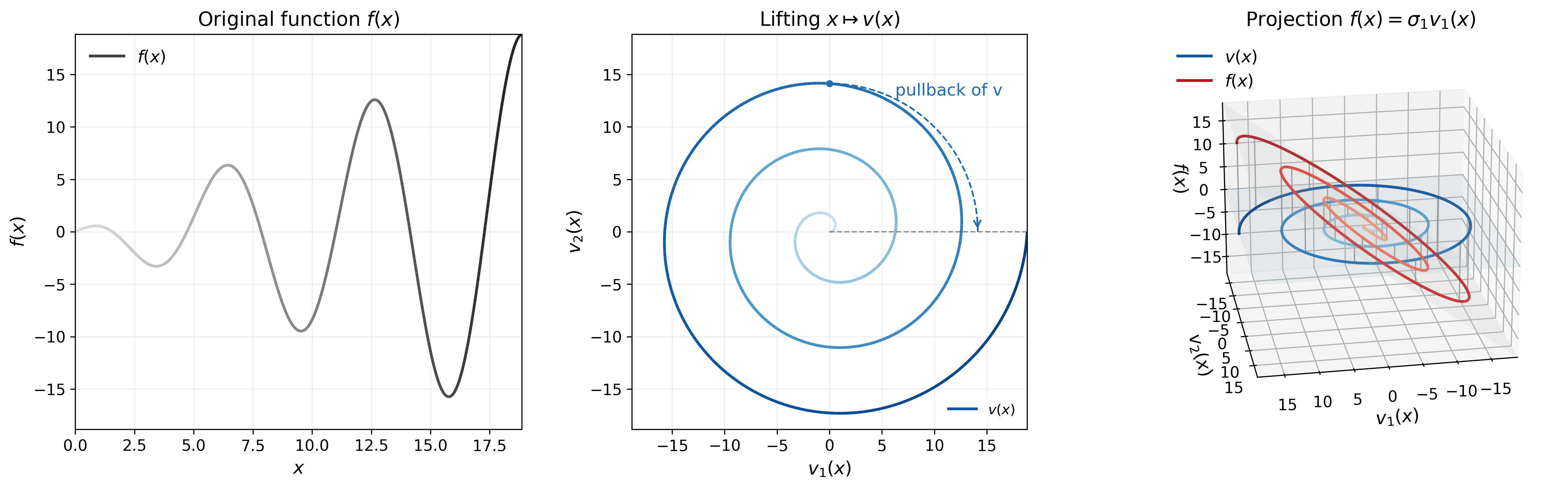}
    \caption{\textbf{Norm-preserving lift followed by linear scaling.}
    Left: a scalar nonlinear map. Middle: a curved lift $v(x)$ satisfying
    $\|v(x)\|_2=\|x\|_2$, with the dashed arrow indicating the pullback on its
    image. Right: the blue lifted curve and its red image under the fixed linear
    output map.}
    \label{fig:conceptual_figure_lifting}
\end{figure}

\FloatBarrier

\section{Appendix: Details for NLSVDNet Examples}
\label{sec:svdnet-examples-details}

\textbf{Datasets:} MNIST (60,000 train / 10,000 test) and Fashion-MNIST (60,000 train / 10,000 test), using the standard torchvision splits with no separate validation set.

\textbf{Model Architecture:} NLSVDNet consists of: (1) a norm-preserving encoder $g$: Conv2d ($1 \to 10$ channels) $\to$ \ManuscriptNLSVDNetResLayers ResNet layers (\ManuscriptNLSVDNetResChannels channels) $\to$ Linear ($7840 \to \ManuscriptNLSVDNetEncodingDims$) with norm preservation subject to the nonzero-code convention in Section~\ref{sec:svdnet}; (2) a linear classifier $K$: Linear($\ManuscriptNLSVDNetEncodingDims \to 10$, no bias); (3) a decoder $g^{-L}$: Linear($\ManuscriptNLSVDNetEncodingDims \to 7840$) $\to$ \ManuscriptNLSVDNetResLayers ResNet layers (\ManuscriptNLSVDNetResChannels channels) $\to$ Conv2d ($10 \to 1$ channels) with norm preservation. ResNet layers use $3 \times 3$ convolutions with ReLU activation, no batch normalization, and no dropout by default.

\textbf{Training:} Adam optimizer with learning rate $\ManuscriptNLSVDNetLearningRate$, batch size \ManuscriptNLSVDNetBatchSize, for \ManuscriptNLSVDNetMnistEpochs epochs (MNIST) or \ManuscriptNLSVDNetFashionEpochs epochs (Fashion-MNIST). Three loss terms per batch: prediction loss (MSE between $K(g(x))$ and one-hot targets with on-value \ManuscriptNLSVDNetLabelOnValue, off-value \ManuscriptNLSVDNetLabelOffValue), bijection loss (MSE between $g^{-L}(g(x))$ and $x$), and regularization loss (penalty on singular values of $K$ exceeding cutoff \ManuscriptNLSVDNetRegularizationCutoff). All losses are optimized jointly with equal weighting.

\textbf{Hyperparameters:} Encoding dimension: \ManuscriptNLSVDNetEncodingDims; ResNet layers: \ManuscriptNLSVDNetResLayers; ResNet channels: \ManuscriptNLSVDNetResChannels; encoding scale: \ManuscriptNLSVDNetMnistEncodingScale (MNIST) / \ManuscriptNLSVDNetFashionEncodingScale (Fashion-MNIST); regularization cutoff: \ManuscriptNLSVDNetRegularizationCutoff; label on-value: \ManuscriptNLSVDNetLabelOnValue; label off-value: \ManuscriptNLSVDNetLabelOffValue.

\textbf{Nullspace Generation:} For category $c$, compute base encoding as $\text{pinv}(K) \cdot \mathbf{e}_c$ where $\mathbf{e}_c$ is the one-hot vector. Extract nullspace basis from SVD of $K$ (right singular vectors corresponding to zero/negligible singular values). Generate samples by adding scaled noise from the nullspace to the base encoding, then decode via $g^{-L}$.

\textbf{Evaluation:} Test set accuracy is reported. No hyperparameter tuning or early stopping; fixed epoch counts are used.

  \ifrevisionmarked
    \begingroup
    \color{RevisionAddGreen}%
    \subsection{Held-Out Encoder Calibration}
\label{sec:held-out-encoder-calibration}

We additionally compare the encoded radius of NLSVDNet with that of the architecture-matched autoencoder used in Figure~\ref{fig:radial-distance-row-null}. The models each contain \(1{,}622{,}809\) trainable parameters, begin from identical weights, receive the same minibatches, and use the same classifier and reconstruction losses. Over \(2{,}000\) held-out MNIST inputs, NLSVDNet satisfies \(\|g(x)\|_2=\|x\|_2\) with mean relative error \(3.44\times10^{-8}\). The matched autoencoder has median relative error \(1.31\). Figure~\ref{fig:radial-distance-held-out} shows the corresponding radii.

\begin{figure}[t]
\centering
\includegraphics[width=0.78\linewidth]{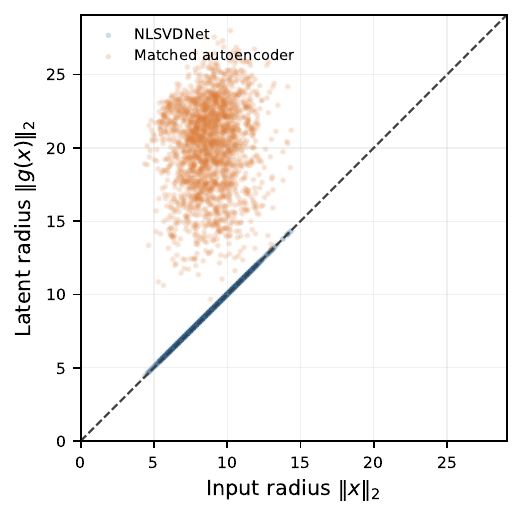}
\caption{\textbf{Held-out encoder radial calibration.}
The NLSVDNet encoding lies on the radius-equality line by construction. The matched autoencoder provides comparable classification and reconstruction performance without assigning the latent representation the input radius.}
\label{fig:radial-distance-held-out}
\end{figure}
    \endgroup
    \IssueRef{M27}%
  \else
    \ifrevisionapplychanges
      \ifcsname RevisionIncludeM27\endcsname
        \subsection{Held-Out Encoder Calibration}
\label{sec:held-out-encoder-calibration}

We additionally compare the encoded radius of NLSVDNet with that of the architecture-matched autoencoder used in Figure~\ref{fig:radial-distance-row-null}. The models each contain \(1{,}622{,}809\) trainable parameters, begin from identical weights, receive the same minibatches, and use the same classifier and reconstruction losses. Over \(2{,}000\) held-out MNIST inputs, NLSVDNet satisfies \(\|g(x)\|_2=\|x\|_2\) with mean relative error \(3.44\times10^{-8}\). The matched autoencoder has median relative error \(1.31\). Figure~\ref{fig:radial-distance-held-out} shows the corresponding radii.

\begin{figure}[t]
\centering
\includegraphics[width=0.78\linewidth]{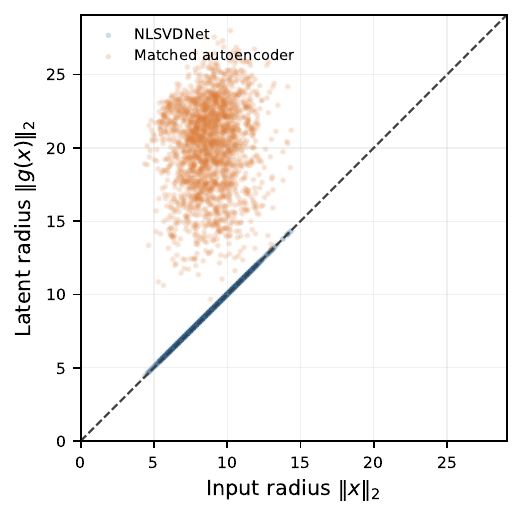}
\caption{\textbf{Held-out encoder radial calibration.}
The NLSVDNet encoding lies on the radius-equality line by construction. The matched autoencoder provides comparable classification and reconstruction performance without assigning the latent representation the input radius.}
\label{fig:radial-distance-held-out}
\end{figure}
      \fi
    \fi
  \fi

\section{Appendix: Construction Limitation Examples} \label{sec:construction-limitation}

\subsection{NLSVDNet Un-recoverable Counter Example}
First we provide a counter-example showing that the NLSVD Construction Algorithm \ref{alg:nlsvd-construction} cannot recover NLSVDNet (Section \ref{sec:svdnet}) decomposition $K,g.$

Consider a scalar input $x \in \mathbb{R}$ and a norm-preserving lift $g: \mathbb{R} \to \mathbb{R}^3$ defined by $g(x) = [x \cos x, x \sin x \cos x, x \sin^2 x]^T$. Let $K = \text{diag}(10, 1, 0) \in \mathbb{R}^{2 \times 3}$, resulting in the black-box function:
\[ f(x) = \begin{bmatrix} 10x \cos x \\ x \sin x \cos x \end{bmatrix}. \]

\textbf{1. Construction of Singular Values ($\epsilon = 0.1$):}
The empirical gains are $\|f_1\| = 10$ and $\|f_2\| = 0.5$. With $p=2$, the construction in Algorithm \ref{alg:nlsvd-construction} yields singular values:
\[ \sigma_1 = 10 \sqrt{\frac{2}{0.9}} \approx 14.91, \quad \sigma_2 = 0.5 \sqrt{\frac{2}{0.9}} \approx 0.745. \]
Note the significant discrepancy from the internal weights $S' = \{10, 1\}$.
Changing $\epsilon$ of course affects $\sigma_i$.

\textbf{2. Reconstruction of $v(x)$:}
The slack factor $\gamma(x) = 1 - ( \frac{f_1(x)^2}{14.91^2 x^2} + \frac{f_2(x)^2}{0.745^2 x^2} )$ is used to define:
\[ v(x) = \frac{|x|}{\sqrt{\delta_1^2 + \delta_2^2 + x^2}} \begin{bmatrix} \delta_1(x) \\ \delta_2(x) \\ x \end{bmatrix}. \]

\begin{center}
\begin{tabular}{lcc}
\toprule
 & \textbf{NLSVDNet} & \textbf{Construction} \\
\cmidrule{2-3}
$\sigma_1$ & 10.0 & 14.91 \\
$\sigma_2$ & 1.0 & 0.745 \\
$U$ & Latent Orientation & Fixed Coordinates \\

\multirow{2}{*}{Lift} & Semantic Space & Norm-Preserving Lift \\
 & $g(x) \in \mathbb{R}^3$ & $v(x) \in \mathbb{R}^3$ \\
\bottomrule
\end{tabular}
\end{center}

\subsection{NLSVD on Linear Map Example}\label{sec:nlsvd-linear-example}
Next, we show that even if $f$ is linear, $v$ need not be linear (hence, not reducing directly to $V^\top$ in the linear SVD).
Consider the linear map $f:\R^2\to\R$ defined by $f(x) = ax_1$, $a>0$.
Clearly $\|f\|_{2 \to 2} = a$.
Given the slack parameter $\epsilon \in (0, 1)$, we define the singular value $\sigma = a\kappa$ where $\kappa = \sqrt{(1-\epsilon)^{-1}}$.
The slack function $\gamma(x)$ represents the residual energy after accounting for the output of $f$:
$$\gamma(x) = 1 - \frac{(ax_1)^2}{(a\kappa)^2 \|x\|_2^2} = 1 - \frac{x_1^2}{\kappa^2(x_1^2 + x_2^2)}.$$
The intermediate lifted component $\delta(x)$ is defined as:
$$\delta(x) = \frac{f(x)}{\sigma \sqrt{\gamma(x)}} = \frac{x_1}{\sqrt{\kappa^2(x_1^2 + x_2^2) - x_1^2}}.$$
The full lifted representation is the concatenation $v(x) = \frac{\|x\|_2}{\|x_\delta\|_2} x_\delta$, where $x_\delta = [\delta(x), x_1, x_2]^\top$.
By construction, this lift satisfies $\|v(x)\|_2 = \|x\|_2$ and preserves the original function via $f(x) = \sigma v_1(x)$.
Note that $v(x)$ is nonlinear because the normalization factor $\gamma(x)$ depends on the direction of $x$.
Specifically, the first component $v_1(x)$ can be expressed explicitly as:
$$v_1(x) = \frac{\|x\|_2 \delta(x)}{\sqrt{\delta(x)^2 + \|x\|_2^2}} = \frac{\|x\|_2 x_1}{\sqrt{x_1^2 + \|x\|_2^2 (\kappa^2 \|x\|_2^2 - x_1^2)}}.$$
\section{NLSVD Construction Validation}\label{sec:construction-validation}
We empirically probe Algorithm~\ref{alg:nlsvd-construction} on MNIST and Fashion-MNIST classifiers trained as NLSVDNets, so that the singular values of the final linear layer $K$ provide known ground-truth coordinate gains $\sigma_i(K)$.
Treating each classifier as a black-box map $f$, we estimate gains from a held-out set $X$ of \ManuscriptNLSVDRepresentativePoints inputs (data sampling), optionally refined by gradient ascent on $\phi_i(x)=|f_i(x)|/\|x\|_2$, and form $\hat U,\hat\Sigma,\hat v$ as in Algorithm~\ref{alg:nlsvd-construction}.

The nontrivial checks are whether the estimated gains are large enough for the residual $\gamma$ to stay nonnegative on the queried points, and how well they recover $\sigma_i(K)$ (Table~\ref{tab:nlsvd-construction-results}, rows d--e).
Mean gain recovery rises from about $64\%$ with data sampling alone to $97.1\%$ (MNIST) and $95.8\%$ (Fashion-MNIST) after gain search.

Rows a--c are numerical checks of identities that the construction enforces pointwise once $\hat\Sigma$ is fixed: for each queried $x$, the lift $\hat v(x)$ is built from $f(x)$ and $\hat\Sigma$, so $\|f(x)-\hat U\hat\Sigma\hat v(x)\|_2\approx0$, $\|\hat v(x)\|_2\approx\|x\|_2$, and $\hat v^{-L}(\hat v(x))\approx x$ hold by algebra (up to floating point) whenever $\gamma(x)\ge0$.
They are not a held-out test of a fitted surrogate that predicts $f$ without querying $f$.
Table~\ref{tab:nlsvd-construction-results} reports means over \ManuscriptNLSVDNumTrials runs.

\begin{table}[t]
\centering
\begin{threeparttable}
\caption{\textbf{NLSVD construction diagnostics} (Algorithm~\ref{alg:nlsvd-construction}) on MNIST \& Fashion-MNIST NLSVDNet classifiers treated as black boxes. Gains are estimated from a held-out set of \ManuscriptNLSVDRepresentativePoints points (with optional gradient gain search). Rows a--c check constructive identities at queried $x$ (using $f(x)$ to form $v(x)$); rows d--e compare estimated gains to the singular values of $K$. Means over \ManuscriptNLSVDNumTrials runs.}
\label{tab:nlsvd-construction-results}
\begin{tabular}{lcc}
\toprule
\textbf{Metric} & \textbf{MNIST} & \textbf{Fashion-MNIST} \\
\midrule
Reconstruct. Error\tnote{a} & $5.58\times 10^{-14}$ & $4.77\times 10^{-14}$ \\
Left Inverse Error\tnote{b} & $6.23\times 10^{-7}$ & $8.48\times 10^{-7}$ \\
Norm Preserv. Error\tnote{c} & $2.38\times 10^{-7}$ & $2.98\times 10^{-7}$ \\
Mean gain recovery\tnote{d} & $63.7\% / 97.1\%$ & $64.0\% / 95.8\%$ \\
Max gain violation\tnote{e} & $0$ & $0$ \\
Wall-clock Time\tnote{f} & $0.033$ s & $0.035$ s \\
\bottomrule
\end{tabular}
\begin{tablenotes}
\footnotesize
\item[a] Mean squared error $\|f(x) - U \Sigma v(x)\|_2^2$ at queried $x$ (identity check; $v(x)$ uses $f(x)$).
\item[b] Left inverse average norm error: $\|v^{-L}(v(x)) - x\|_2$ (identity check on $\mathrm{Im}(v)$).
\item[c] Norm preservation average absolute error: $|\|v(x)\|_2 - \|x\|_2|$ (identity check).
\item[d] Mean over canonical output coordinates of the data-sampled gain divided by the NLSVDNet head singular value, followed by the same quantity after gradient gain search.
\item[e] Maximum positive excess $(\hat g_i-\sigma_i(K))_+$ after gain search in the NLSVDNet canonical basis.
\item[f] Compute specifications in Section~\ref{sec:reproducibility-assets-compute}.
\end{tablenotes}
\end{threeparttable}
\end{table}

  \ifrevisionmarked
    \begingroup
    \color{RevisionRemoveRed}%
    \section{Appendix: Perturbation Attack Algorithms}
\label{sec:algorithms}

\begin{algorithm}[H]
\caption{``Black Box'' NLSVD Directional Search Attack}
\label{alg:nlsvd_directional_search}
\begin{algorithmic}[1]
\REQUIRE Black box function $f: \mathbb{R}^m \to \mathbb{R}^n$, NLSVD parameters $(U, \Sigma, v)$, input $x_0 \in \mathbb{R}^m$, step size $\Delta > 0$, budget $B > 0$, finite-difference epsilon $\epsilon > 0$
\ENSURE Adversarial perturbation $\delta \in \mathbb{R}^m$ or failure

\STATE \textbf{Step 1: Identify Base State}
\STATE Compute $y_0 = f(x_0)$ and identify current class $u_0 = \arg\max_j (y_0)_j$
\STATE Determine index $i_0$ such that the $i_0$-th component of the lifting corresponds to $u_0$ via $U$

\STATE \textbf{Step 2: Target Selection (Analytical)}
\STATE Compute feature representation $z = v(x_0) \in \mathbb{R}^{n+m}$
\FOR{each rival index $i \in \{1, \dots, n\} \setminus \{i_0\}$}
    \STATE Compute logit gap: $\text{Gap}_i = \sigma_{i_0} z_{i_0} - \sigma_i z_i$
\ENDFOR
\STATE Select target index $i^* = \arg\min_{i \neq i_0} \text{Gap}_i$

\STATE \textbf{Step 3: Direction Computation}
\FOR{each input dimension $j \in \{1, \dots, m\}$}
    \STATE $\nabla_x v_{i^*}[j] \gets \frac{v_{i^*}(x_0 + \epsilon e_j) - v_{i^*}(x_0)}{\epsilon}$
    \STATE $\nabla_x v_{i_0}[j] \gets \frac{v_{i_0}(x_0 + \epsilon e_j) - v_{i_0}(x_0)}{\epsilon}$
\ENDFOR
\STATE $\mathbf{d} \gets \sigma_{i^*} \nabla_x v_{i^*} - \sigma_{i_0} \nabla_x v_{i_0}$
\STATE $\hat{\mathbf{d}} \gets \mathbf{d} / \|\mathbf{d}\|_2$

\STATE \textbf{Step 4: Linear Probe}
\STATE Initialize $r \gets \Delta$
\WHILE{$r < B$}
    \STATE $x_{\text{pert}} \gets \text{clip}(x_0 + r \hat{\mathbf{d}}, 0, 1)$
    \IF{$\arg\max_j f_j(x_{\text{pert}}) \neq u_0$}
        \OUTPUT $\delta = x_{\text{pert}} - x_0$ \COMMENT{Attack successful}
    \ELSE
        \STATE $r \gets r + \Delta$
    \ENDIF
\ENDWHILE
\OUTPUT failure
\end{algorithmic}
\end{algorithm}
    \endgroup
    \IssueRef{M01}%
  \else
    \ifrevisionapplychanges
    \else
      \section{Appendix: Perturbation Attack Algorithms}
\label{sec:algorithms}

\begin{algorithm}[H]
\caption{``Black Box'' NLSVD Directional Search Attack}
\label{alg:nlsvd_directional_search}
\begin{algorithmic}[1]
\REQUIRE Black box function $f: \mathbb{R}^m \to \mathbb{R}^n$, NLSVD parameters $(U, \Sigma, v)$, input $x_0 \in \mathbb{R}^m$, step size $\Delta > 0$, budget $B > 0$, finite-difference epsilon $\epsilon > 0$
\ENSURE Adversarial perturbation $\delta \in \mathbb{R}^m$ or failure

\STATE \textbf{Step 1: Identify Base State}
\STATE Compute $y_0 = f(x_0)$ and identify current class $u_0 = \arg\max_j (y_0)_j$
\STATE Determine index $i_0$ such that the $i_0$-th component of the lifting corresponds to $u_0$ via $U$

\STATE \textbf{Step 2: Target Selection (Analytical)}
\STATE Compute feature representation $z = v(x_0) \in \mathbb{R}^{n+m}$
\FOR{each rival index $i \in \{1, \dots, n\} \setminus \{i_0\}$}
    \STATE Compute logit gap: $\text{Gap}_i = \sigma_{i_0} z_{i_0} - \sigma_i z_i$
\ENDFOR
\STATE Select target index $i^* = \arg\min_{i \neq i_0} \text{Gap}_i$

\STATE \textbf{Step 3: Direction Computation}
\FOR{each input dimension $j \in \{1, \dots, m\}$}
    \STATE $\nabla_x v_{i^*}[j] \gets \frac{v_{i^*}(x_0 + \epsilon e_j) - v_{i^*}(x_0)}{\epsilon}$
    \STATE $\nabla_x v_{i_0}[j] \gets \frac{v_{i_0}(x_0 + \epsilon e_j) - v_{i_0}(x_0)}{\epsilon}$
\ENDFOR
\STATE $\mathbf{d} \gets \sigma_{i^*} \nabla_x v_{i^*} - \sigma_{i_0} \nabla_x v_{i_0}$
\STATE $\hat{\mathbf{d}} \gets \mathbf{d} / \|\mathbf{d}\|_2$

\STATE \textbf{Step 4: Linear Probe}
\STATE Initialize $r \gets \Delta$
\WHILE{$r < B$}
    \STATE $x_{\text{pert}} \gets \text{clip}(x_0 + r \hat{\mathbf{d}}, 0, 1)$
    \IF{$\arg\max_j f_j(x_{\text{pert}}) \neq u_0$}
        \OUTPUT $\delta = x_{\text{pert}} - x_0$ \COMMENT{Attack successful}
    \ELSE
        \STATE $r \gets r + \Delta$
    \ENDIF
\ENDWHILE
\OUTPUT failure
\end{algorithmic}
\end{algorithm}
    \fi
  \fi

\newpage
\section{Additional Bias Geometry Measurements}
\label{sec:bias-appendix}

The experiments in Section~\ref{sec:bias} suggest that dataset imbalance induces measurable geometric changes in the NLSVD coordinates of a trained network. Here we give the full set of measurements, define each observable and its motivation, and report how each one responds to imbalance across all three datasets. These experiments are descriptive rather than a fairness benchmark: our aim is to illustrate the kinds of geometric quantities that the explicit factorization $f(x)=K\,g(x)$ makes directly accessible, and which of them behave consistently.

\paragraph{Setup, base model, and the well-trained requirement.}
Each model is an NLSVDNet, $f(x)=K\,g(x)$: a norm-preserving 20-layer convolutional ResNet encoder $g:\mathbb{R}^n\!\to\!\mathbb{R}^E$ with embedding dimension $E=1000$, followed by a bias-free linear head $K\in\mathbb{R}^{C\times E}$ whose ordinary SVD $K=U\Sigma V^\top$ \emph{is} the network's NLSVD. For each of MNIST, Fashion-MNIST, and CIFAR-10 we construct biased training sets by choosing a single \emph{target} class, keeping all of its examples, and under-sampling every other class by a factor (the \emph{sampling ratio}) swept over eight log-scaled values in $[0.005,0.64]$; sweeping all ten choices of target yields $8\times 10 = 80$ models per dataset. The fully represented target is the \emph{majority} and the collectively under-sampled remainder is the \emph{minority}; the test set is the original balanced split.

A prerequisite for this analysis is that the models be well-trained. The singular structure of $K$ and the geometry of $g$ describe how a network allocates a learned solution under imbalance, so a model that fails to fit the task yields degenerate coordinates regardless of bias. In the extreme it simply collapses to always predicting the majority class, and the resulting geometry is uninformative. We therefore tuned the models to reach well above the $0.1$ chance level in the near-balanced regime. MNIST and Fashion-MNIST train readily with the base configuration ($10$ residual channels, learning rate $10^{-3}$, $5$ epochs), reaching ${\approx}0.98$ and ${\approx}0.91$ balanced-regime accuracy. Fitting CIFAR-10 required $16$ residual channels with two strided spatial down-sampling stages and a reduced learning rate ($8\times 10^{-4}$) trained for longer ($20$ epochs), after which it reaches ${\approx}0.75$ balanced-regime accuracy. The reported CIFAR-10 measurements use this configuration; the accuracy rows of Table~\ref{tab:bias-appendix-quant} confirm all three model families are well above chance in the balanced regime.\IssueRef{M16}

\paragraph{Gain-structure metrics (functions of $K$).}
These probe the anisotropic gain $\Sigma$---how the head allocates amplification across output directions. Note that writing $z = V^\top g(x)\in\mathbb{R}^E$ with $C$ classes and $E$ embedding dimensions, the lifted energy splits into a \emph{row-space} (signal) part and a \emph{null-space} part:
\[
\|z\|_2^2 \;=\; \underbrace{\textstyle\sum_{i=1}^{C} z_i^2}_{\text{signal}} \;+\; \underbrace{\textstyle\sum_{i=C+1}^{E} z_i^2}_{\text{null space } N(K)} .
\]

\begin{itemize}
\item \textbf{Sigma ratio} $\sigma_1/\sigma_2$: ratio of the two largest singular values. The simplest scalar summary of spectral peakedness; large values mean a single direction dominates the map.
\item \textbf{Effective rank} $r_{\mathrm{eff}} = 1/\sum_i p_i^2$ with $p_i=\sigma_i^2/\sum_j\sigma_j^2$ (inverse participation ratio of the spectrum)~\citep{roy2007effective}: a continuous count of ``active'' output directions that decreases as energy concentrates.
\item \textbf{Tail singular value} $\sigma_{\min}$: the smallest singular value of $K$, i.e.\ the gain of the weakest class direction. A shrinking $\sigma_{\min}$ signals that some class directions are being \emph{starved} of gain, and is a more sensitive tail probe than $\sigma_1/\sigma_2$, which only sees the top of the spectrum.\IssueRef{M18}
\item \textbf{Target / minority dominance}: with $z=V^\top g(x)$, the fraction of \emph{signal}-space energy carried by the leading singular direction, $z_1^2/\sum_{i=1}^{C} z_i^2$, averaged separately over target-class and minority-class test samples. This NLSVD-native probe asks whether the top direction of $K$ encodes the majority class, and (evaluated on minority samples) whether minority inputs \emph{also} collapse onto that direction.
\item \textbf{Minority null-space energy}: $\sum_{i>C} z_i^2 / \|z\|_2^2$, averaged over minority samples. The fraction of a sample's lifted energy that lands in $N(K)$. This tests the original hypothesis that minority information is displaced into the null space (the model becoming ``blind'' to it).
\end{itemize}

\paragraph{Representation-geometry metrics (functions of $g(x)$).}
These probe the encoder $g$ directly, independently of $K$'s spectrum: the side of the factorization that ordinary analyses cannot isolate.
\begin{itemize}
\item \textbf{Feature effective rank}: the inverse participation ratio of the eigenvalues of the feature covariance $\mathrm{Cov}(g(x))$ over the test set (reported globally, and per class)~\citep{roy2007effective}. It measures the intrinsic dimensionality of the representation itself, and hence \emph{feature collapse}.
\item \textbf{NC1} (neural-collapse ratio)~\citep{papyan2020}: $\mathrm{tr}(\Sigma_W)/\mathrm{tr}(\Sigma_B)$, the within-class over between-class feature scatter, where $\Sigma_W=\sum_c\sum_{i\in c}\|g(x_i)-\mu_c\|^2$ and $\Sigma_B=\sum_c n_c\|\mu_c-\mu\|^2$. A standard diagnostic of class separability: it rises as classes become less linearly separable in feature space.\IssueRef{M19}
\item \textbf{Signal participation}: the effective \emph{number} of active signal directions, $\big(\sum_{i\le C} z_i^2\big)^2/\sum_{i\le C} z_i^4$, per group. It is the ``count'' complement of dominance ($1$ = fully collapsed onto one direction, $C$ = energy spread uniformly).
\end{itemize}

\paragraph{Quantitative results.}
Table~\ref{tab:bias-appendix-quant} reports, for every metric, its value in the balanced regime, its value under the strongest imbalance, and the Spearman correlation $\rho$ between the sampling ratio and the (per-ratio mean) metric; $|\rho|\!\to\!1$ indicates a monotone trend, and $\rho$ is negative for quantities that \emph{increase} as imbalance strengthens (since the sampling ratio decreases). Three effects are robust across all three datasets: the leading singular direction of $K$ aligns ever more tightly with the majority class (target dominance $\to\!0.95$ everywhere), the tail singular value $\sigma_{\min}$ contracts (direction starvation), and separability in feature space degrades (NC1 rises monotonically). Notably, $\sigma_1/\sigma_2$ and effective rank are only \emph{dataset-dependent} here: both are clean on Fashion-MNIST but weak or noisy on MNIST and CIFAR-10, so concentration surfaces more reliably as directional dominance and tail-gain starvation than as global rank reduction. Main-paper Figure~\ref{fig:bias-crossdataset} contrasts the two sides directly: target dominance climbs almost identically on all three datasets (a uniform \emph{gain}-side effect), whereas the magnitude of NC1's rise scales with task difficulty: negligible on Fashion-MNIST but a $6\times$ increase on CIFAR-10, where features become genuinely entangled.

\begin{table*}[t]
\centering
\small
\caption{\textbf{Bias geometry across datasets.} For each observable we list its mean value in the balanced regime (sampling ratio $0.64$) $\to$ under the strongest imbalance (ratio $0.005$), with the Spearman correlation $\rho$ in parentheses. $|\rho|\!\to\!1$ is monotone; $\rho<0$ marks quantities that increase under imbalance. ``Consistency'' summarizes the cross-dataset verdict.}
\label{tab:bias-appendix-quant}
\setlength{\tabcolsep}{3pt}
\begin{tabular}{@{}lcccl@{}}
\toprule
Metric (as sampling ratio decreases) & MNIST & Fashion-MNIST & CIFAR-10 & Consistency \\
\midrule
\multicolumn{5}{@{}l}{\textit{Gain structure} --- from $K=U\Sigma V^\top$}\\
\quad $\sigma_1/\sigma_2$ (spectral peakedness)      & $1.17\!\to\!1.67$ ($-0.17$) & $1.11\!\to\!1.70$ ($-1.0$) & $1.04\!\to\!1.33$ ($-1.0$)  & dataset-dependent \\
\quad Effective rank                                 & $9.0\!\to\!5.9$ ($+0.62$)   & $9.0\!\to\!5.3$ ($+1.0$)   & $9.6\!\to\!8.8$ ($+0.62$)   & dataset-dependent \\
\quad Tail singular value $\sigma_{\min}$ (starvation)& $1.29\!\to\!0.50$ ($+1.0$)  & $1.22\!\to\!0.48$ ($+1.0$) & $0.64\!\to\!0.40$ ($+0.88$) & \textbf{robust} \\
\quad Target dominance                               & $0.09\!\to\!0.95$ ($-1.0$)  & $0.11\!\to\!0.96$ ($-1.0$) & $0.16\!\to\!0.94$ ($-1.0$)  & \textbf{robust} \\
\quad Minority dominance                             & $0.09\!\to\!0.29$ ($-0.38$) & $0.09\!\to\!0.27$ ($-0.76$)& $0.08\!\to\!0.74$ ($-0.98$) & scales w/ difficulty \\
\quad Minority null-space energy                     & $0.48\!\to\!0.49$ ($+0.31$) & $0.70\!\to\!0.58$ ($+0.86$)& $0.83\!\to\!0.74$ ($+0.24$) & no consistent trend \\
\midrule
\multicolumn{5}{@{}l}{\textit{Representation geometry} --- from $g(x)$}\\
\quad Feature effective rank                          & $9.9\!\to\!4.8$ ($+0.67$)   & $6.3\!\to\!3.5$ ($+0.98$)  & $6.9\!\to\!4.7$ ($+0.57$)   & non-monotone \\
\quad NC1 (within / between scatter)                  & $0.18\!\to\!0.95$ ($-1.0$)  & $0.61\!\to\!0.80$ ($-0.98$)& $1.26\!\to\!8.19$ ($-1.0$)  & \textbf{robust}; scales w/ difficulty \\
\quad Signal participation (target)                  & $3.9\!\to\!1.1$ ($+1.0$)    & $2.7\!\to\!1.1$ ($+0.93$)  & $3.5\!\to\!1.1$ ($+1.0$)    & \textbf{robust} \\
\midrule
\multicolumn{5}{@{}l}{\textit{Induced bias} --- prediction-side sanity checks}\\
\quad Accuracy gap (target $-$ minority recall)      & $0.01\!\to\!0.61$ ($-1.0$)  & $0.03\!\to\!0.56$ ($-1.0$) & $0.08\!\to\!0.95$ ($-1.0$)  & \textbf{robust} \\
\quad Bias amplification                             & $1.0\!\to\!4.9$ ($-1.0$)    & $1.1\!\to\!4.7$ ($-1.0$)   & $1.2\!\to\!8.8$ ($-1.0$)    & \textbf{robust} \\
\bottomrule
\end{tabular}
\end{table*}

%

\paragraph{Spectral concentration.}
Figures~\ref{fig:sigma-spectrum} and~\ref{fig:sigma-ratio} illustrate the concentration effect, drawing one ghost trace per target class, a $\pm 1$ std band, and a mean line. To show each effect where it is clearest, Figure~\ref{fig:sigma-spectrum} pairs the effective rank of $K$ with the target-dominance ratio on Fashion-MNIST (where the effective-rank reduction is cleanest), and Figure~\ref{fig:sigma-ratio} shows $\sigma_1/\sigma_2$ on CIFAR-10. As Table~\ref{tab:bias-appendix-quant} indicates, target dominance is the reliable indicator across all three datasets, whereas $\sigma_1/\sigma_2$ and effective rank are dataset-dependent (see also main-paper Figure~\ref{fig:bias-crossdataset}).

\begin{figure}[]
\centering
\includegraphics[width=0.9\columnwidth]{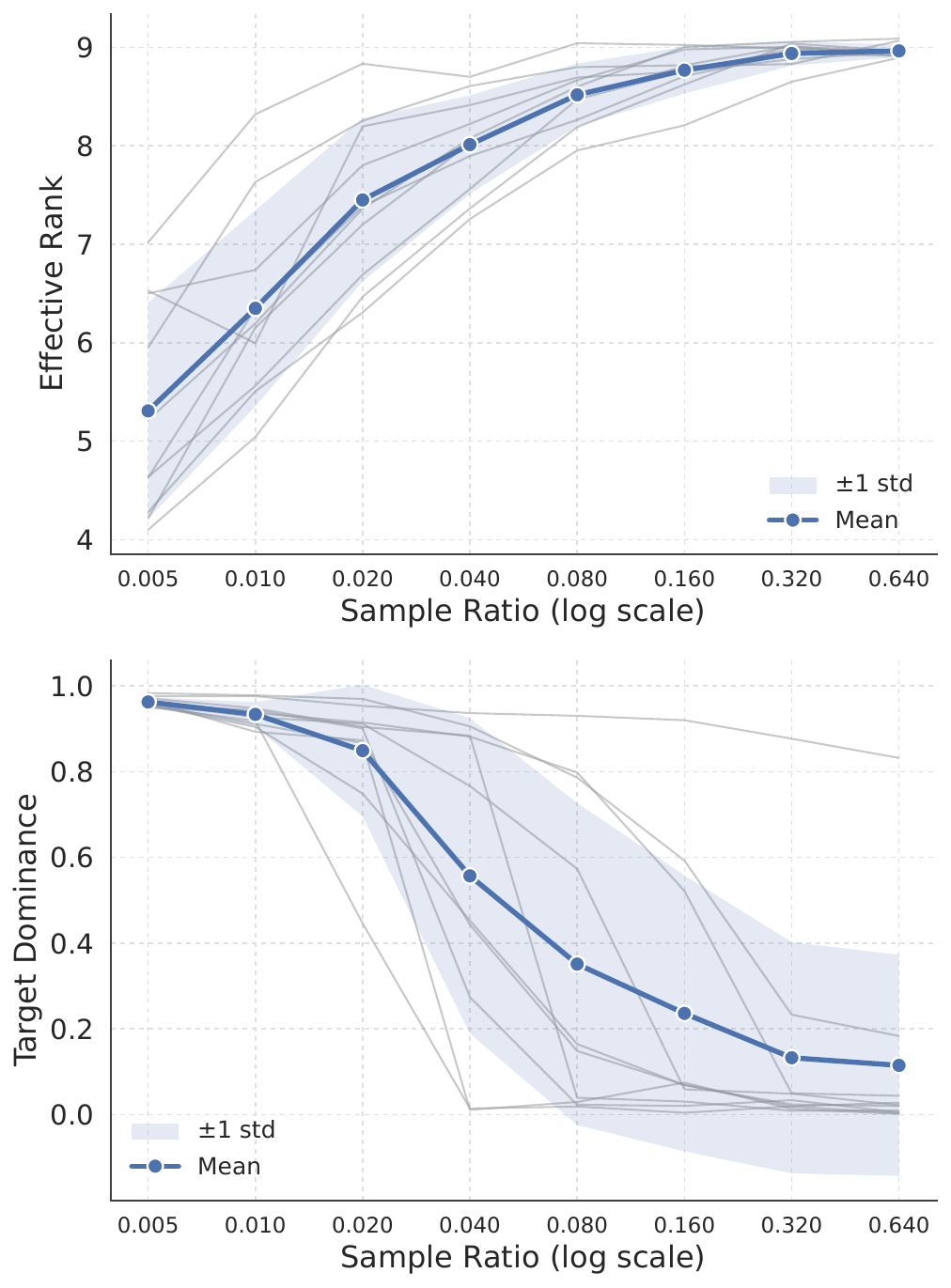}
\caption{
\textbf{Singular spectrum trend under dataset imbalance.}
Effective rank of $K$ (top) and target-dominance ratio (bottom) across biased Fashion-MNIST models: the dataset where the effective-rank reduction is clearest. Stronger imbalance (smaller sampling ratios) lowers the effective rank and concentrates signal-space energy onto the leading singular direction; see Table~\ref{tab:bias-appendix-quant} for all three datasets.
}
\label{fig:sigma-spectrum}
\end{figure}

\begin{figure}[]
\centering
\includegraphics[width=0.9\columnwidth]{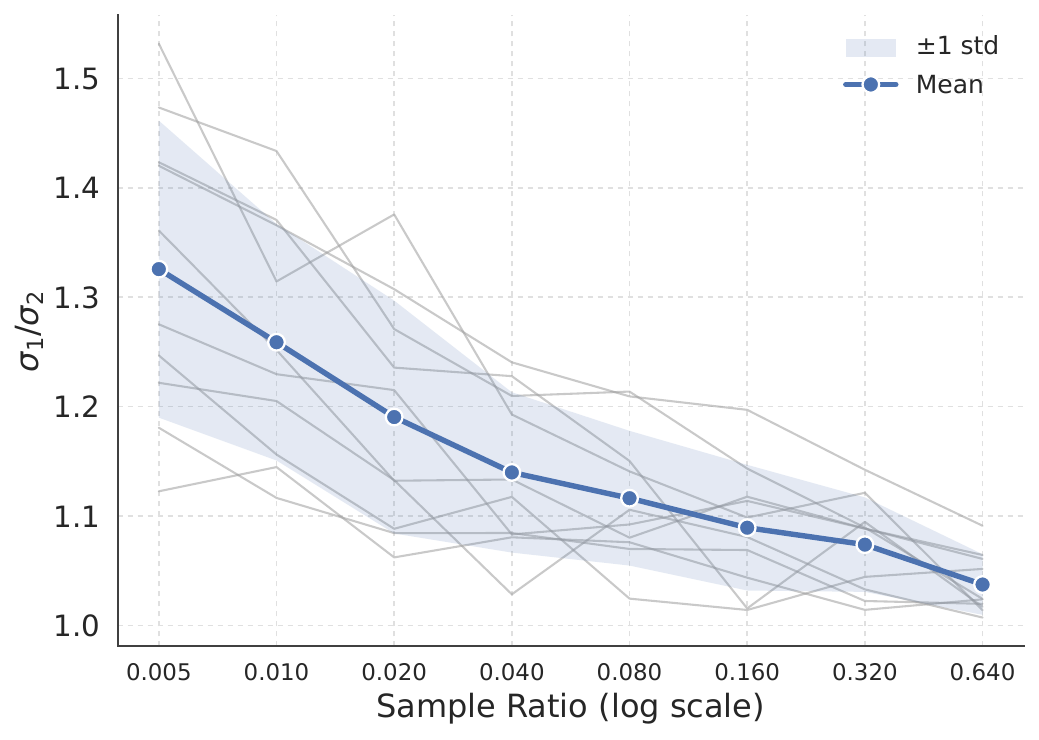}
\caption{
\textbf{Sigma ratios (CIFAR-10).}
Ratio $\sigma_1/\sigma_2$ between the two leading singular values of $K$ across biased CIFAR-10 models; stronger imbalance widens the gap. This summary is clean on CIFAR-10 and Fashion-MNIST but noisy across targets on MNIST (Table~\ref{tab:bias-appendix-quant}).
}
\label{fig:sigma-ratio}
\end{figure}

\paragraph{Concentration vs.\ collapse.}
The target/minority dominance pair distinguishes two qualitatively different regimes. Target dominance rises to ${\approx}0.95$ on every dataset (the majority always captures the leading direction), but \emph{minority} dominance behaves very differently: it stays low on MNIST ($0.09\!\to\!0.29$, $\rho=-0.38$): minority classes retain their own directions even under heavy imbalance. However, it  rises sharply on CIFAR-10 ($0.08\!\to\!0.74$, $\rho=-0.98$; Figure~\ref{fig:minority-dominance}), where minority inputs are increasingly funneled onto the majority axis. A large target--minority dominance gap therefore indicates \emph{concentration without collapse} (as on MNIST), while a small gap indicates genuine \emph{representation collapse} (as on CIFAR-10); this mirrors the NC1 trend, which is largest exactly where the gap is smallest.

\begin{figure}[]
\centering
\includegraphics[width=0.9\columnwidth]{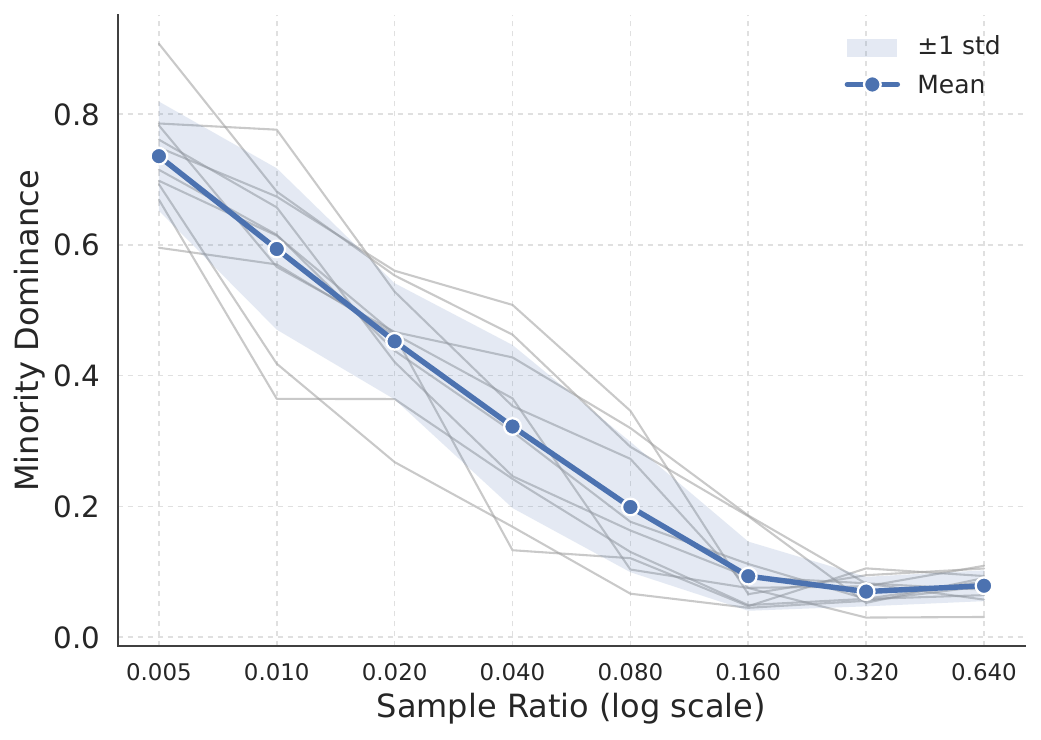}
\caption{
\textbf{Minority-class dominance ratio (CIFAR-10).}
Fraction of minority-class signal-space energy on the leading singular direction, across biased CIFAR-10 models. It rises sharply under imbalance ($0.08\!\to\!0.74$), indicating that minority inputs are increasingly funneled onto the majority axis (representation collapse). On MNIST the same quantity stays modest---concentration without collapse (Table~\ref{tab:bias-appendix-quant}).
}
\label{fig:minority-dominance}
\end{figure}

\paragraph{Null-space migration is not observed.}
Finally, we tested whether minority information is displaced into $N(K)$. It is not: the minority null-space energy fraction shows no consistent trend with imbalance on any dataset (Table~\ref{tab:bias-appendix-quant}; Figure~\ref{fig:nullspace-energy}). Combined with the strong dominance and tentative NC1 effects, this indicates that bias-induced representation collapse is driven by concentration into dominant \emph{nonzero} singular directions, rather than by projection into the null space.

\begin{figure}[]
\centering
\includegraphics[width=0.9\columnwidth]{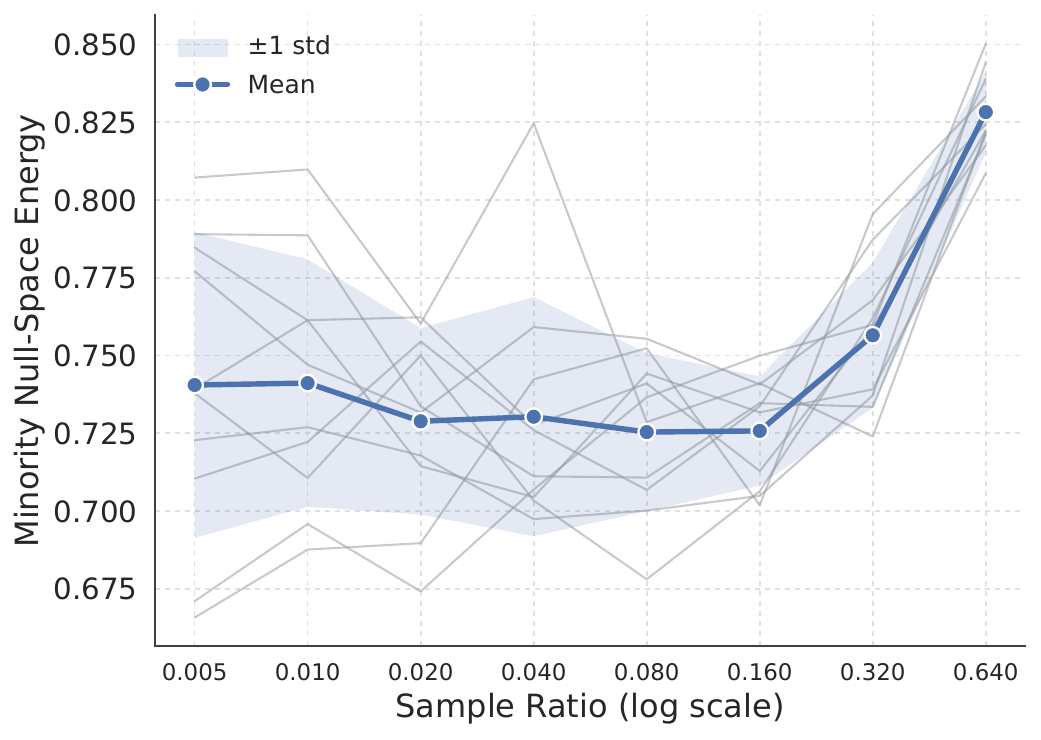}
\caption{
\textbf{Minority-class null-space energy (CIFAR-10).}
Fraction of lifted energy projected into $N(K)$ for minority-class samples, across biased CIFAR-10 models. Contrary to the original hypothesis, strong imbalance does not consistently increase null-space occupancy: the effect is flat or weak on all three datasets.
}
\label{fig:nullspace-energy}
\end{figure}

\paragraph{Summary.} These experiments show that the NLSVD decomposition turns dataset imbalance bias into a set of directly measurable coordinates, spanning \emph{gain} structure (directional dominance, tail-singular-value starvation) and \emph{representation} geometry (feature separability), plus the negative control of null-space allocation. We view them as a proof of concept motivating further study of bias implications on feature collapse, pruning, and representation geometry through an NLSVD lens.

\FloatBarrier

\section{Membership Inference Robustness}
\label{sec:mia-appendix}

This appendix expands on the membership-inference experiment of Section \ref{sec:mia}: the attack, the setup, and the two supporting analyses (the singular-value causal test and the leakage predictor).

\paragraph{Attack.}
We use online LiRA \cite{carlini2022membership}, the likelihood-ratio refinement of shadow-model
membership inference \cite{shokri2017membership}. The attack pool is the full CIFAR-10 training set;
each shadow and target model is trained on an independent random $50\%$ mask of the pool, so every
example is a member for about half of the models. For each pool example we take the logit-scaled
confidence of the true class, $\phi = \log p_y - \log\!\big(\sum_{y'\neq y} p_{y'}\big)$, which is
loss-agnostic and thus directly comparable across the MSE-trained NLSVDNet and the cross-entropy CNN.
Across shadows we fit per-example Gaussians $\mathcal{N}(\mu_{\text{in}},\sigma_{\text{in}}^2)$ and
$\mathcal{N}(\mu_{\text{out}},\sigma_{\text{out}}^2)$; each target/example pair is then scored by the
log-likelihood ratio of these two Gaussians and thresholded to trace one ROC \emph{per target model} ($25$ curves per model type). We report the median of these curves together with a shaded band
running from their $5$th to $95$th percentile, taken across the target models at each false-positive rate (the band reflects target-to-target variability, not a statistical confidence interval).

\paragraph{Setup.}
CIFAR-10, pool size $50{,}000$ (${\approx}25{,}000$ members per model). The NLSVDNet uses encoding dimension $1{,}000$, $20$ residual layers, $16$ residual channels, spatial downsampling in $g$, trained $20$ epochs at learning rate $8\times10^{-4}$). The CNN baseline is a compact two-conv-block classifier; its width is chosen so that its generalization gap (train accuracy - test accuracy) matches the NLSVDNet's ($15.6\%$), isolating architecture from overfitting. To separate the NLSVDNet architecture from its training objective we add a third condition: the identical NLSVDNet architecture trained with a plain cross-entropy objective, with the bijection reconstruction loss and $\sigma$-regularization removed. We train $128$ shadow and $25$ target models per condition.

\paragraph{Model comparison.}
Table~\ref{tab:mia-results} records accuracies, generalization gaps, and LiRA scores for the three conditions; the corresponding ROC curves appear in main-paper Figure~\ref{fig:mia-roc}. At matched gap, the full NLSVDNet ROC lies below the CNN across false-positive rates, with the largest separation at low FPR, and the plain-CE ablation of the same architecture leaks more than either baseline. Because the NLSVDNet and CNN share a generalization gap, the CNN's higher leakage is not explained by greater overfitting; the ablation indicates that the training objective, rather than the architecture alone, is associated with the reduced leakage. The inflated generalization gap under plain CE is consistent with the leakage predictor below.

\begin{table}[H]
\caption{\textbf{Membership Inference (LiRA) Robustness.}
NLSVDNet vs.\ a plain CNN (matched
generalization gap) and an ablation that keeps the NLSVDNet architecture but trains it with a plain cross-entropy objective. Accuracy and gap in \% on CIFAR-10; AUC / TPR@$10^{-3}$ (lower is better).}
\label{tab:mia-results}
\centering
\small
\setlength{\tabcolsep}{7.5pt}
\begin{tabular}{@{}lcccc@{}}
\toprule
Model & Acc. & Gap & AUC & TPR@$10^{-3}$ \\
\midrule
NLSVDNet, full obj. & $\mathbf{72.5}$ & $15.6$ & $\mathbf{.635}$ & $\mathbf{.003}$ \\
NLSVDNet, plain CE  & $68.7$          & $26.5$ & $.736$          & $.030$ \\
Plain CNN         & $67.1$          & $15.6$ & $.696$          & $.014$ \\
\bottomrule
\end{tabular}
\end{table}

\paragraph{Singular-value causal test.}
A natural hypothesis is that the singular-value regularization of $K$ (the closest thing the NLSVDNet has to a capacity/privacy dial) controls leakage. We test it directly by sweeping the regularization
cutoff over $\{1,2,4,8,16,64\}$ ($6$ seeds each) and scoring every model against the shadow farm. Leakage is flat (Table~\ref{tab:mia-sigma}, Figure \ref{fig:mia-predictor}): the variation across cutoffs is no larger than the seed-to-seed noise within a cutoff. We conclude that spectral regularization of the head $K$ is not a lever on memorization \RevisionReplace{M23}{. This is consistent with the interpretation that memorization resides in the encoder $g$.}{over the tested cutoff range. This sweep does not localize memorization to either the encoder or the head.}
\RevisionComment{M23}{A null response to one head regularization parameter does not locate membership information in the encoder because the logits depend on both \(g\) and \(K\).}

\paragraph{Leakage predictor.}
Mirroring the bias study in Section \ref{sec:bias}, we ask whether a cheap quantity read off a single trained model predicts its leakage. We correlate per-model observables against LiRA AUC across the
shadow farm (Spearman $\rho$, Table~\ref{tab:mia-predictor}). The generalization gap (Figure \ref{fig:mia-predictor}) and train accuracy are the strong predictors, while \emph{every} spectral observable of $K$ (the ratio $\sigma_1/\sigma_2$, effective/entropy rank, $\sigma_{\min}$, $\sigma_{\max}$) is essentially uncorrelated. This is the key contrast with the bias setting, where $K$'s spectrum \emph{does} track the induced bias.

\begin{figure}[htbp]
\centering
\includegraphics[width=\linewidth]{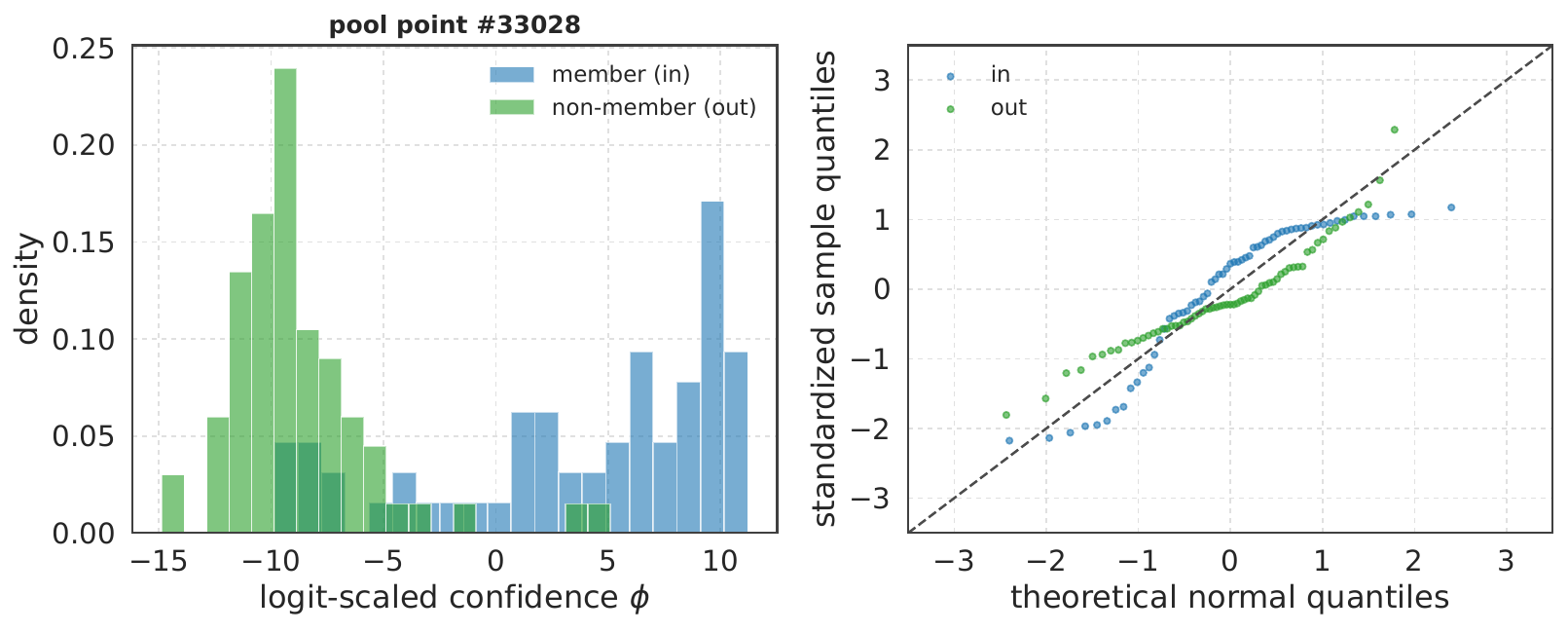}
\caption{Per-example in/out confidences.}
\label{fig:mia-qqplot}
\end{figure}


\begin{table}[H]
\caption{Leakage vs.\ singular-value regularization cutoff (mean\,$\pm$\,std over $6$ seeds). No
trend beyond seed noise.}
\label{tab:mia-sigma}
\centering
\small
\setlength{\tabcolsep}{5pt}
\begin{tabular}{@{}lcccccc@{}}
\toprule
Cutoff & $1$ & $2$ & $4$ & $8$ & $16$ & $64$ \\
\midrule
LiRA AUC        & $.640$ & $.642$ & $.640$ & $.642$ & $.644$ & $.641$ \\
$\pm$std        & $.008$ & $.007$ & $.004$ & $.009$ & $.005$ & $.008$ \\
TPR@$10^{-3}$   & $.0028$ & $.0035$ & $.0030$ & $.0028$ & $.0035$ & $.0037$ \\
\bottomrule
\end{tabular}
\end{table}

\begin{table}[H]
\caption{Spearman $\rho$ between per-model observables and LiRA leakage.}
\label{tab:mia-predictor}
\centering
\small
\setlength{\tabcolsep}{18pt}
\begin{tabular}{@{}lcc@{}}
\toprule
Observable & $\rho$ (AUC) & $\rho$ (TPR@$10^{-3}$) \\
\midrule
Train accuracy               & $+0.64$ & $+0.41$ \\
Generalization gap           & $+0.60$ & $+0.07$ \\
Test accuracy                & $+0.20$ & $+0.47$ \\
Bijection recon.\ error      & $-0.25$ & $+0.10$ \\
$\sigma_{\min}$              & $-0.15$ & $+0.04$ \\
$\sigma_1/\sigma_2$          & $+0.12$ & $+0.02$ \\
Entropy rank                 & $-0.08$ & $+0.12$ \\
Effective rank               & $-0.08$ & $+0.11$ \\
$\sigma_{\max}$             & $+0.00$ & $-0.02$ \\
\bottomrule
\end{tabular}
\end{table}

\begin{figure}[H]
\centering
\begin{minipage}[t]{0.48\linewidth}\centering
\includegraphics[width=\linewidth]{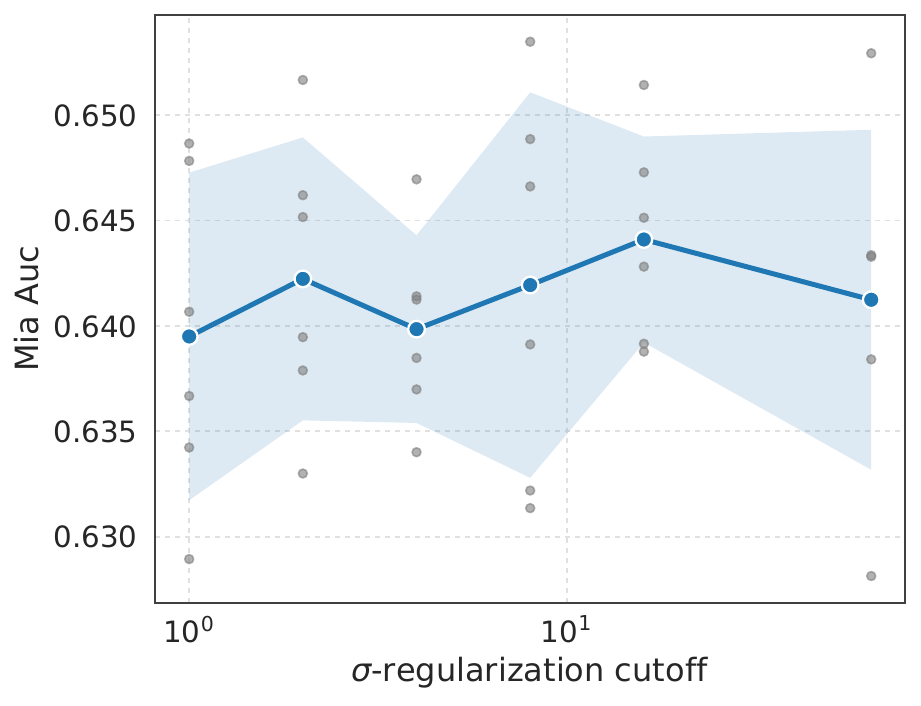}
\end{minipage}
\begin{minipage}[t]{0.48\linewidth}\centering
\includegraphics[width=\linewidth]{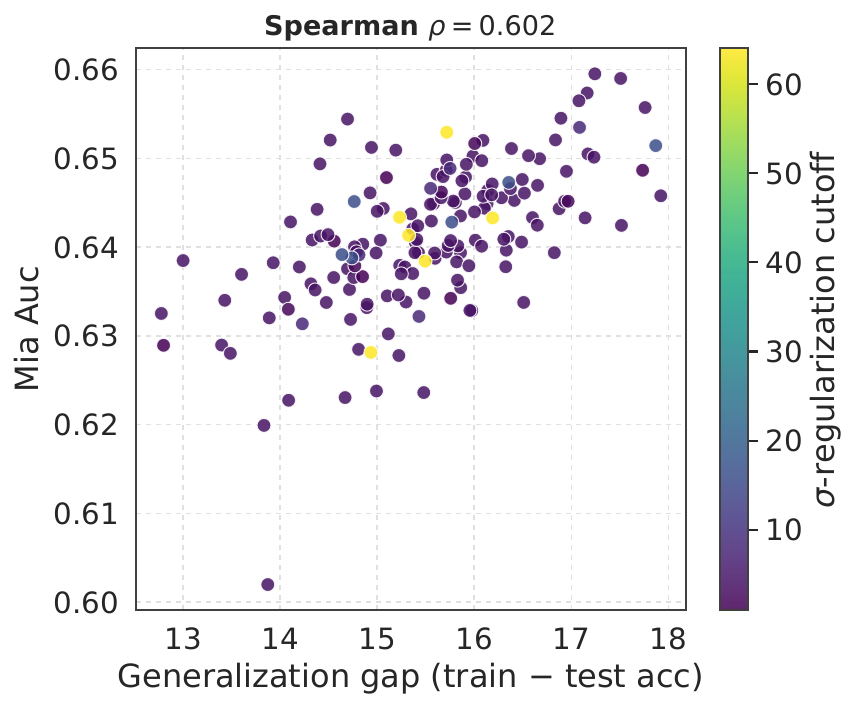}
\end{minipage}
\caption{Leakage is flat across the regularization cutoff (left), but rises with the generalization gap (right)}
\label{fig:mia-predictor}
\end{figure}

\FloatBarrier

\section{Membership Inference Attacks}
\label{sec:mia-attack-appendix}

LiRA \cite{carlini2022membership} is the strongest membership-inference attack available, but an expensive one: it trains a farm
of $128$ shadow models to estimate, for each point, the confidence a model would place on it had it \emph{not} been trained on it. Because the NLSVDNet exposes structure a conventional network
hides (the null space of the head $K$, its singular spectrum, the learned bijection $g^{-L}\!\circ g$) it is natural to ask whether an NLSVDNet can match LiRA's performance \emph{without} that farm by exploiting the structure instead. Using the same setup as in Section \ref{sec:mia} (reusing the same $25$ target models but \emph{zero} shadow models), and calibrating against a set of public non-members (the CIFAR-10 test set), our strongest attack turns out to be a shadow-free \emph{confidence} test. For a
candidate example $(x,y)$ we compute its logit-scaled confidence
$\phi(x)=\ell_y-\log\!\sum_{k\neq y}e^{\ell_k}=\log\tfrac{p_y}{1-p_y}$, which corresponds to the model's log-odds on the
true class $y$, and turn it into a per-example membership score by
$z$-scoring it against the reference non-members of the same class, so the threshold adapts to how confident the model is on that class in general. This exploratory experiment reaches AUC $0.596$, just short of LiRA's $0.635$ score in our setup. Clearly above a random guessing baseline of $0.5$, this shows it is possible to achieve comparable leakage detection at a tiny fraction of the cost. The details of this experiment follow below.

\paragraph{Setup.}
We reuse the $25$ CIFAR-10 NLSVDNet target models of Section \ref{sec:mia-appendix} and their membership masks. The candidate pool is the full $50{,}000$-point training pool (labels = each target's mask). A disjoint set of guaranteed non-members (the CIFAR-10 test set, standing in for public same-distribution data) provides per-example calibration as explained below. Each attack emits one score for each \emph{(model, point)} pair; we sweep it into a per-target ROC and aggregate over the $25$ targets exactly as in Section \ref{sec:mia-appendix} (median with a $5$--$95$th-percentile band), so every curve is directly comparable to the LiRA ROC. Each raw statistic is turned into a membership score two ways: a per-example/per-class $z$-score against the reference non-members (``calibrated''), and the bare statistic under a single global threshold (``raw'').

\paragraph{The attacks.} We consider a wide variety of attacks and approaches that leverage the NLSVDNet structure.
\emph{(i) Null-space pullback (white-box).} Following Section \ref{sec:svdnet-examples}, for each class $j$ we decode a cloud of manifold samples $g^{-L}(K^{\dagger}y_j+z_{\mathcal N})$, $z_{\mathcal N}\in\mathcal N(K)$, and score a candidate by its $k$-NN distance (norm-invariant pixel space, as $g$ is norm-preserving) to the cloud of its class. Because $g^{-L}$ is only a \emph{left} inverse, $g\circ g^{-L}$ need not be the identity, so the decoded clouds are not automatically class-consistent (they re-classify as class $j$ only $0$--$64\%$ of the time). We therefore source the null-space content from real reference encodings and \emph{reject-filter} the clouds, keeping only decoded samples the model re-classifies as $j$, so the manifold is class-faithful by
construction where the decoder permits, and empty for classes it cannot round-trip. We report both the calibrated and raw/global-threshold scores.
\emph{(ii) NLSVD $\gamma$ residual (black-box).} From query access alone we run the Algorithm-1 reconstruction (Section \ref{sec:construction}) on the reference set and score a point by the residual $\gamma(x)=1-\sum_j f_{u_j}(x)^2/(\sigma_j^2\lVert x\rVert^2)$, corresponding to the energy \emph{not} captured by
the estimated top-$p$ gain directions.
\emph{(iii) Bijection reconstruction (white-box).} We score the relative reconstruction error $\lVert g^{-L}(g(x))-x\rVert/\lVert x\rVert$ of the NLSVDNet's built-in autoencoder. Since the
bijection loss is optimized on the training members, memorization would make members reconstruct better. This corresponds to the standard reconstruction attack against autoencoding and generative models.
\emph{(iv) Combined.} The mean of the calibrated confidence and reconstruction scores, the two per-example signals the NLSVDNet actually trains on.
We compare against two baselines: a \emph{shadow-free} confidence attack (the calibrated logit-scaled confidence $\phi$ of Section \ref{sec:mia-appendix}, and LiRA itself, with $128$ shadows.

\begin{table}[t]
\caption{Constructive attacks on the NLSVDNet's NLSVD structure (CIFAR-10, aggregated over $25$ target models). All structure-native attacks use \emph{zero} shadow models. Higher AUC / TPR@$10^{-3}$ means a stronger attack; every structure-native attack is at or near chance and none beats even the
shadow-free confidence baseline, or LiRA.}
\label{tab:mia-attack}
\centering
\footnotesize
\setlength{\tabcolsep}{5pt}
\resizebox{\columnwidth}{!}{%
\begin{tabular}{@{}lcc@{}}
\toprule
Attack & AUC & TPR@$10^{-3}$ \\
\midrule
\multicolumn{3}{@{}l}{\emph{Structure-native, no shadow models}}\\
\quad Null-space pullback, calibrated (white-box) & $.501$ & $.0010$ \\
\quad Null-space pullback, raw/global (white-box) & $.500$ & $.0010$ \\
\quad NLSVD $\gamma$ residual (black-box)          & $.537$ & $.0011$ \\
\quad Bijection reconstruction (white-box)        & $.504$ & $.0011$ \\
\quad Confidence $+$ reconstruction               & $.572$ & $.0016$ \\
\midrule
\multicolumn{3}{@{}l}{\emph{Baselines}}\\
\quad Confidence $\phi$ (shadow-free)             & $.596$ & $.0012$ \\
\quad LiRA ($128$ shadow models)                  & $\mathbf{.635}$ & $\mathbf{.0030}$ \\
\bottomrule
\end{tabular}
}
\end{table}

\begin{figure}[t]
\centering
\includegraphics[width=\linewidth]{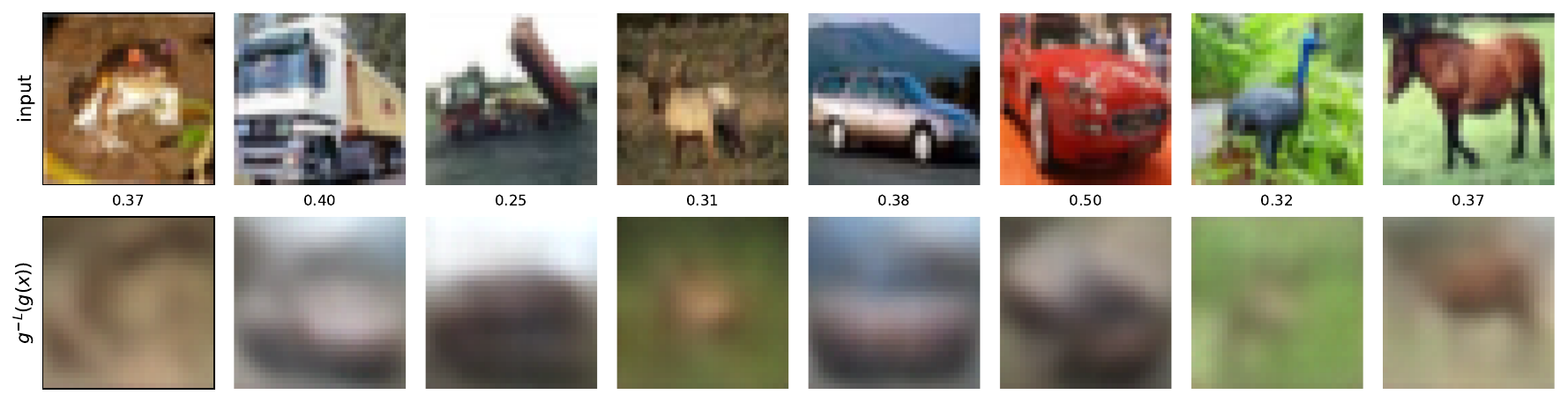}
\caption{The NLSVDNet bijection is a lossy, low-frequency near-inverse. Top: eight CIFAR-10 inputs;
bottom: their reconstructions $g^{-L}(g(x))$ (relative error above each). The decoder recovers the dominant color and coarse layout but not fine detail, and does so identically for members and non-members, which is arguably why the reconstruction attack is at chance.}
\label{fig:mia-recon}
\end{figure}

\paragraph{Findings.}
Table~\ref{tab:mia-attack} is uniformly negative. The null-space pullback distance is at chance (AUC $0.50$) whether calibrated or thresholded globally, and even on the classes whose clouds are
faithful. Arguably, this is because a distance to a class-faithful manifold measures class \emph{typicality}, which
members and same-class non-members share, not the per-example detail that separates a memorized point from an unseen one of the same class. The reconstructed NLSVD geometry does slightly better (AUC $0.54$); and the bijection reconstruction does
not discriminate at all (AUC $0.50$; the mean member/non-member relative-error gap is ${\approx}0.1\%$). None reaches the shadow-free confidence attack (AUC $0.60$). Combining confidence with reconstruction \emph{hurts} ($0.57$), showing that the latter signal is noisy. The reconstruction result is the most informative: the decoder reconstructs members and non-members equally, and only coarsely (Figure~\ref{fig:mia-recon}) at that. The bijection objective \emph{generalizes} rather than \emph{memorizes} individual training points.
This mirrors the leakage-predictor analysis of Section \ref{sec:mia-appendix}: membership tracks per-example fit (confidence / generalization gap), which the training objective suppresses.

\paragraph{Can the shadow-free baseline be strengthened?}
The one attack that carries real signal without shadow models is the calibrated confidence $\phi$ (AUC $0.596$), and it still trails LiRA ($0.635$). Since LiRA's advantage comes entirely from its \emph{per-example} out-distribution (the confidence of $128$ shadow models that never trained on $x$) we asked whether the NLSVDNet's exposed structure could supply a shadow-free substitute or an orthogonal signal that closes the gap. Table~\ref{tab:mia-shadowfree} sweeps four families, all reusing the same $25$ targets with zero shadow models.
\emph{(1) Finer calibration.} Replacing the per-class $z$-score with a per-class empirical rank ($0.592$), or with a genuinely per-example calibration (the $z$-score of $\phi(x)$ against the
$k{=}200$ reference non-members nearest to $x$ in encoder space $g(x)$, the shadow-free analogue of LiRA's per-example distribution) reduces the score to ($0.554$).
\emph{(2) Augmentation.} Aggregating $\phi$ over query augmentations (horizontal flip, pixel shifts) hurts identically whether we average the raw $\phi$ then calibrate or calibrate each
augmentation then average ($0.554$ both ways), and the per-example spread across augmentations is at chance ($0.497$). These augmentations, yet unseen in training, arguably makes a member's augmented confidence look non-member-like, diluting the clean single query.
\emph{(3) Spectral fingerprints.} White-box per-example statistics of the trained model: local $2{\to}2$ gain, the participation ratio of the row-space NLSVD coordinates $V^{\top}g(x)$, and the null-space energy fraction. These are individually weak ($|\mathrm{AUC}-0.5|\le 0.05$) and, combined with $\phi$, all fall \emph{below} $\phi$ alone. This is expected: the logits
are $U\Sigma V^{\top}g(x)$, so the spectral decomposition is a re-parameterization of the same forward pass, not the missing counterfactual as in LiRA.
\emph{(4) Richer reconstruction.} Beyond the scalar bijection error of attack~(iii), we resolved the reconstruction into high- and low-frequency bands (the detail band being where memorization of exact pixels would most likely surface even if the total error is member-independent) and iterated the round-trip
$T=g^{-L}\!\circ g$ to probe closeness to a fixed point. Every variant still remains at chance ($0.50$--$0.51$).
\emph{(5) MC-dropout.} Reading each dropout mask as a model variant makes the spread over $T{=}10$ stochastic forward passes a per-example distribution over variants. Superficially, this is the
ingredient shadow models supply. But these variants all share the training set, so the spread is \emph{epistemic uncertainty of one model} (a memorization proxy), not the counterfactual of a model
that never saw $x$. Empirically the uncertainty (predictive variance, entropy, argmax disagreement) is weak and, crucially, redundant with $\phi$: swept over the dropout rate $p\in\{0.02,\dots,0.35\}$ its AUC is monotone, and best at the gentlest $p{=}0.02$ ($\approx0.55$). It decays to chance as larger dropout degenerates the model (trained without dropout). Every combination with $\phi$ stays below $\phi$ alone, and MC-averaging $\phi$ itself does not improve it. The $p$-sweep rules out a poorly chosen dropout rate.
\paragraph{Summary.} Our shadow-free attacks top off at a $0.596$ ceiling, slightly below LiRA ($0.635$). We conclude that the residual gap to LiRA is attributable specifically to the per-example, over-models counterfactual that only shadow models provide.

\begin{table}[t]
\caption{Attempts to strengthen the shadow-free baseline (CIFAR-10, $25$ targets, \emph{zero}
shadow models). ``$+\phi$'' is the attack combined with the calibrated confidence. No attack or
combination beats the plain per-class-calibrated confidence ($0.596$), and none approaches LiRA.}
\label{tab:mia-shadowfree}
\centering
\small
\setlength{\tabcolsep}{5pt}
\begin{tabular}{@{}lcc@{}}
\toprule
Attack & AUC & $+\phi$ (AUC) \\
\midrule
Confidence $\phi$, per-class $z$ (baseline) & $\mathbf{.596}$ & --- \\
\midrule
\multicolumn{3}{@{}l}{\emph{Finer calibration}}\\
\quad per-class empirical rank              & $.592$ & --- \\
\quad per-example neighbor ($g(x)$ $k$-NN)  & $.554$ & --- \\
\multicolumn{3}{@{}l}{\emph{Augmentation}}\\
\quad flip/shift, calibrated                & $.554$ & --- \\
\quad cross-augmentation consistency        & $.497$ & --- \\
\multicolumn{3}{@{}l}{\emph{Spectral fingerprints (white-box)}}\\
\quad local gain (sharpness)                & $.546$ & $.585$ \\
\quad signal participation ratio            & $.513$ & $.575$ \\
\quad null-space energy fraction            & $.539$ & $.583$ \\
\multicolumn{3}{@{}l}{\emph{Richer reconstruction}}\\
\quad high-frequency band error             & $.510$ & $.584$ \\
\quad round-trip drift / contraction        & $.500$ & $.550$ \\
\multicolumn{3}{@{}l}{\emph{MC-dropout uncertainty}}\\
\quad best rate ($p{=}0.02$)                & $.554$ & $.588$ \\
\midrule
LiRA ($128$ shadow models)                  & $.635$ & --- \\
\bottomrule
\end{tabular}
\end{table}

%
  \ifrevisionmarked
    \begingroup
    \color{RevisionRemoveRed}%
    \section{Appendix: Experiment details and extra results}\label{sec:perturbation-attack-details}
The NLSVD timing values reported in the main text were obtained on the rerun environment used to generate this manuscript.
Timing is not directly comparable for the C\&W attack because our experimental framework used PyTorch while the C\&W code used TensorFlow; hence, in each iteration, tensors were converted, sent to the GPU, then transferred back and converted again.

\subsection{Larger-Budget C\&W $L^2$ Results}\label{sec:cw-maxiter-full}
For the \citet{carlini2017towards} attack, we also use
\texttt{MAX\_ITER = \ManuscriptCWBigMaxIterations}. Compared with
Table~\ref{tab:perturbation-attack-results}, the results are similar but require
about $\ManuscriptCWBigQueryRatioMnist\times$ as many queries on MNIST,
$\ManuscriptCWBigQueryRatioFashion\times$ on Fashion-MNIST, and
$\ManuscriptCWBigQueryRatioCifar\times$ on CIFAR-10.
\begin{table}[t]
\centering
\footnotesize
\begin{threeparttable}
\begin{tabular}{lccc}
\toprule
\textbf{Metric} & \textbf{MNIST} & \textbf{Fashion} & \textbf{CIFAR-10} \\
\midrule
Success percent (\%) & 90.4 & 91.6 & 87.0 \\
Avg. perturb. norm & 13.47 & 11.53 & 13.70 \\
Avg. queries/sample & 9010.0 & 9010.0 & 9010.0 \\
\bottomrule
\end{tabular}
\end{threeparttable}
\end{table}

\FloatBarrier

\FloatBarrier%
    \endgroup
    \IssueRef{M01}%
  \else
    \ifrevisionapplychanges
    \else
    \fi
  \fi

\section{Appendix: Reproducibility, Assets, and Compute}
\label{sec:reproducibility-assets-compute}

\paragraph{Code and execution.}
\RevisionReplace{M01}{The submitted supplement contains the experiment code, configuration files, and
setup scripts.  The project targets Python 3.12 with dependencies pinned in
\texttt{requirements.txt}.  From the repository root, the full rerun workflow is
\texttt{make rerun-full} or equivalently
\texttt{./scripts/run\_all\_experiments.sh}; both regenerate the attack
experiments, NLSVD-construction experiments, NLSVDNet sweeps, and derived
manuscript tables.  The final synchronization step is
\texttt{python scripts/sync\_manuscript\_results.py}.  Small CIFAR smoke tests
are available through \texttt{make smoke-cifar} and
\texttt{make smoke-cifar100}.}{The submitted supplement contains the experiment code, configuration files, and setup scripts. The project targets Python 3.12 with dependencies pinned in \texttt{requirements.txt}. The bias and membership appendices record the model grids and evaluation protocols used for those experiments.}

\paragraph{Data, splits, and configurations.}
\RevisionReplace{M01}{All datasets are standard public image-classification benchmarks loaded through
\texttt{torchvision}: MNIST, Fashion-MNIST, CIFAR-10, and CIFAR-100.  For each
attack experiment, the training split is used to train the classifier/NLSVDNet
model, a held-out test subset of \ManuscriptPerturbationNLSVDPoints samples is
used to estimate the NLSVD, and a disjoint held-out test subset is used for attack
evaluation.  The evaluation sizes are
\ManuscriptPerturbationEvalPointsMnistFashion samples for MNIST and
Fashion-MNIST and \ManuscriptPerturbationEvalPointsCifar samples for CIFAR-10
and CIFAR-100.  The checked-in configurations
\texttt{mnist.json}, \texttt{fashion\_mnist.json}, \texttt{cifar10.json}, and
\texttt{cifar100.json} specify the seeds, training epochs, attack budgets,
C\&W iteration counts, and NLSVD slack parameter.}{The experiments use standard public image-classification benchmarks loaded through \texttt{torchvision}: MNIST, Fashion-MNIST, and CIFAR-10. Sections~\ref{sec:svdnet-examples-details}, \ref{sec:bias-appendix}, and \ref{sec:mia-appendix} give the corresponding training and evaluation details.}

\paragraph{Hyperparameters and reported statistics.}
\RevisionReplace{M01}{NLSVDNet hyperparameters for the visual examples are listed in
Appendix~\ref{sec:svdnet-examples-details}.  Algorithmic details for the NLSVD
directional search attack are listed in Appendix~\ref{sec:algorithms}.  Table
\ref{tab:nlsvd-construction-results} reports means over
\ManuscriptNLSVDNumTrials NLSVD-construction runs.  Table
\ref{tab:perturbation-attack-results} reports empirical success rates, average
perturbation norms, and query counts on the fixed evaluation subsets above; we
do not claim formal statistical significance beyond these empirical summaries.}{NLSVDNet hyperparameters for the visual examples are listed in Appendix~\ref{sec:svdnet-examples-details}. Table~\ref{tab:nlsvd-construction-results} reports means over \ManuscriptNLSVDNumTrials NLSVD-construction runs.}

\paragraph{Compute resources.}
\RevisionReplace{M01}{The reruns used to produce the reported tables were run on a workstation with an
NVIDIA GeForce RTX 3080 GPU.  The logs record wall-clock runtime and query counts
for the reported attack and NLSVD-construction runs; peak memory was not
separately logged.  The main tables report per-sample wall-clock or construction
time where relevant, and Appendix~\ref{sec:perturbation-attack-details} records
the timing caveat for comparing PyTorch NLSVD runs with the TensorFlow C\&W
baseline.}{The reruns used to produce the reported tables were run on a workstation with an NVIDIA GeForce RTX 3080 GPU. The logs record wall-clock runtime for the NLSVD-construction runs. Peak memory was not separately logged.}

\paragraph{Existing assets and licenses.}
The work uses public benchmark datasets and software assets rather than scraped
or newly collected human-subject data.  MNIST is used from the public
LeCun--Cortes--Burges release; Fashion-MNIST is MIT licensed; CIFAR-10 and
CIFAR-100 are used from the public Krizhevsky release and standard
\texttt{torchvision} loaders.  \texttt{torchvision} is BSD-3-Clause licensed,
\RevisionRemove{M01}{and the original Carlini--Wagner attack implementation used for comparison is
BSD licensed.  }The submitted code package does not redistribute private data.

\end{document}